\documentclass[10pt,journal,compsoc]{IEEEtran}

\ifCLASSOPTIONcompsoc
  \usepackage[nocompress]{cite}
\else
  \usepackage{cite}
\fi

\usepackage{booktabs}
\usepackage[noend]{algpseudocode}
\usepackage{algorithm}
\usepackage{xcolor}
\usepackage{multirow}
\usepackage{subfigure}
\usepackage[pdftex]{graphicx}
\usepackage{amsmath}
\usepackage{amsthm}
\usepackage{amsfonts}
\usepackage{amssymb}
\usepackage[hidelinks]{hyperref}
\usepackage[numbers]{natbib}
\usepackage{comment}
\usepackage{enumitem}
\setlist{leftmargin=3mm}

\graphicspath{{figs/}}

\DeclareMathOperator*{\argmin}{arg\,min}
\newcommand\numberthis{\addtocounter{equation}{1}\tag{\theequation}}
\newcommand{\R}{\mathbb{R}}
\newcommand{\E}{\mathbb{E}}
\renewcommand{\P}{\mathrm{Pr}}
\newcommand{\T}{^\top}
\newcommand{\tr}{\mathrm{tr}}

\newcommand{\hide}[1]{}

\newtheorem{definition}{Definition}
\newtheorem{assumption}{Assumption}
\newtheorem{lemma}{Lemma}
\newtheorem{theorem}{Theorem}

\begin{document}

\bstctlcite{IEEEexample:BSTcontrol}

\title{Fast and Secure Distributed\\Nonnegative Matrix Factorization
	\author{Yuqiu Qian,
	        Conghui Tan,
	        Danhao Ding,
	        Hui Li,
	        and Nikos Mamoulis
		\IEEEcompsocitemizethanks{
			\IEEEcompsocthanksitem Yuqiu Qian is with Tencent, Shenzhen, China. E-mail: yuqiuqian@tencent.com.
			\IEEEcompsocthanksitem Conghui Tan is with WeBank, Shenzhen, China. E-mail: tanconghui@gmail.com.
			\IEEEcompsocthanksitem Danhao Ding is with Department of Computer Science,
			University of Hong Kong, Hong Kong SAR, China. E-mail: dhding2@cs.hku.hk.
			\IEEEcompsocthanksitem Hui Li is with School of Informatics, Xiamen University, Xiamen, Fujian, China. E-mail: hui@xmu.edu.cn. He is the corresponding author.
			\IEEEcompsocthanksitem Nikos Mamoulis is with Department of Computer Science and Engineering, University of Ioannina, Ioannina, Epirus, Greece. E-mail: nikos@cs.uoi.gr.
		}
	}
}


\IEEEtitleabstractindextext{%
\begin{abstract} 
Nonnegative matrix factorization (NMF) has been successfully
applied in several data mining tasks. Recently, there is an increasing
interest in the acceleration of NMF, due to its high cost on large matrices.
On the other hand, the privacy issue of NMF over federated data is worthy of
attention, since NMF is prevalently applied in image and text analysis which
may involve leveraging privacy data (e.g, medical image and record) across
several parties (e.g., hospitals). In this paper, we study the
\emph{acceleration} and \emph{security} problems of distributed NMF. Firstly,
we propose a {\em distributed sketched alternating nonnegative least squares}
(DSANLS) framework for NMF, which utilizes a matrix sketching technique to
reduce the size of nonnegative least squares subproblems with a convergence
guarantee. For the second problem, we show that DSANLS with modification can
be adapted to the security setting, but only for \emph{one or limited
iterations}. Consequently, we propose four efficient distributed NMF methods
in both synchronous and asynchronous settings with a security guarantee. We
conduct extensive experiments on several real datasets to show the superiority
of our proposed methods. The implementation of our methods is available at
\url{https://github.com/qianyuqiu79/DSANLS}.
\end{abstract}

}

\maketitle

\IEEEdisplaynontitleabstractindextext
\IEEEpeerreviewmaketitle

\ifCLASSOPTIONcompsoc
\IEEEraisesectionheading{\section{Introduction}\label{sec:introduction}}
\else
\section{Introduction}
\label{sec:introduction}
\fi
\IEEEPARstart{N}{onnegative} matrix factorization (NMF) is a technique for
discovering nonnegative latent factors and/or performing dimensionality
reduction. Unlike general
matrix factorization (MF), NMF restricts the two output matrix factors to be
nonnegative. Specifically, the goal of NMF is to decompose a huge matrix
$M\in \mathbb{R}_+^{m\times n}$ into the product of two matrices
$U\in\R_+^{m\times k}$ and $V\in\R_+^{n\times k}$ such that $M\approx UV\T$.
$\R_+^{m\times n}$ denotes the set of $m \times n$ matrices with nonnegative
real values, and $k$ is a user-specified dimensionality, where typically $k\ll
m,n$. Nonnegativity is inherent in the feature space of many real-world
applications, where the resulting factors of NMF can have a natural
interpretation. Therefore, NMF has been widely used in a branch of fields including text mining
\cite{pauca2004text}, image/video processing \cite{kotsia2007novel},
recommendation \cite{gu2010collaborative}, and analysis of social networks
\cite{ZhangY12}.

Modern data analysis tasks apply on big matrix data with increasing scale and
dimensionality. Examples~\cite{kannan2016high} include community detection in
a billion-node social network, background separation on a 4K video in which
every frame has approximately 27 million rows, and text mining on a
bag-of-words matrix with millions of words. The volume of data is anticipated
to increase in the `big data' era, making it impossible to store the whole
matrix in the main memory throughout NMF. Therefore, there is a need for
high-performance and scalable distributed NMF algorithms. On the other hand,
there is a surge of works on privacy-preserving data mining over
federated data~\cite{KimSYJ17, feng2018privacy} in recent years. In contrast
to traditional research about privacy which emphasizes protecting individual
information from single institution, federated data mining deals with multiple
parties. Each party possesses its own confidential dataset(s) and the union of
data from all parties is utilized for achieving better performance in the
target task.  
Due to the prevalent use of NMF in image and text analysis which may involve
leveraging privacy data (e.g, medical image and record) across several
parties (e.g., hospitals), the privacy issue of NMF over federated data is
worthy of attention.  To address aforementioned challenges of NMF (i.e.,
high performance and privacy), we study the \emph{acceleration} and \emph{security}
problems of distributed NMF in this paper.

First of all, we propose the {\em distributed sketched alternating nonnegative
least squares} (DSANLS) for accelerating NMF. The state-of-the-art distributed NMF
is MPI-FAUN~\cite{kannan2016mpi}, a general framework that iteratively solves
the nonnegative least squares (NLS) subproblems for $U$ and $V$. The main idea
behind MPI-FAUN is to exploit the independence of local updates for rows of
$U$ and $V$, in order to minimize the communication requirements of matrix
multiplication operations within the NMF algorithms. Unlike MPI-FAUN, our idea
is to speed up distributed NMF in a new, orthogonal direction: by reducing the
problem size of each NLS subproblem within NMF, which in turn decreases the
overall computation cost. In a nutshell, we reduce the size of each NLS
subproblem, by employing a {\em matrix sketching} technique: the involved
matrices in the subproblem are multiplied by a specially designed random
matrix at each iteration, which greatly reduces their dimensionality. As a
result, the computational cost of each subproblem significantly drops.

However, applying matrix sketching comes with several issues. First, although
the size of each subproblem is significantly reduced, sketching involves
matrix multiplication which brings computational overhead. Second, unlike in a
single machine setting, data is distributed to different nodes in distributed
environment. Nodes may have to communicate extensively in a poorly designed
solution. In particular, each node only retains part of both the input matrix
and the generated approximate matrices, causing difficulties due to data
dependencies in the computation process. Besides, the generated random
matrices should be the same for all nodes in every iteration, while
broadcasting the random matrix to all nodes brings severe communication
overhead and can become the bottleneck of distributed NMF. Furthermore, after
reducing each original subproblem to a sketched random new subproblem, it is
not clear whether the algorithm still converges and whether it converges to
stationary points of the original NMF problem.

Our DSANLS overcomes these problems. Firstly, the extra computation cost due
to sketching is reduced with a proper choice of the random matrices. Then,
the same random matrices used for sketching are generated independently at
each node, thus there is no need for transferring them among nodes during
distributed NMF. Having the complete random matrix at each node, an NMF
iteration can be done locally with the help of a matrix multiplication rule
with proper data partitioning. Therefore, our matrix sketching approach
reduces not only the computational overhead, but also the communication cost. Moreover,
due to the fact that sketching also {\em shifts} the optimal solution of each
original NMF subproblem, we propose subproblem solvers paired with theoretical
guarantees of their convergence to a stationary point of the original
subproblems.

To provide solutions to the problem of secure distributed NMF over federated
data, we first show that DSANLS with modification can be adapted to this
security setting, but only for \emph{one or limited iterations}. Therefore, we
design new methods called \emph{Syn-SD} and \emph{Syn-SSD} in synchronous
setting. They are later extended to \emph{Asyn-SD} and \emph{Asyn-SSD} in
asynchronous setting (i.e., client/server), respectively. Syn-SSD
improves the convergence rate of Syn-SD, without incurring much extra
communication cost. It also reduces computational overhead by \emph{sketching}. All
proposed algorithms are secure with a guarantee. Secure distributed NMF problem
is hard in nature. All parties involved should not be able to infer the
confidential information during the process. To the best of our knowledge, we
are the first to study NMF over federated data.

In summary, our contributions are as follows:
\begin{itemize}
	\item DSANLS is the first distributed NMF algorithm that leverages matrix sketching to reduce the problem size of each NLS subproblem and can be applied to both dense and sparse input matrices with a convergence guarantee. 
	\item We propose a novel and specially designed subproblem solver ({\em proximal coordinate descent}), which helps DSANLS converge faster.
	We also discuss the use of {\em projected gradient descent} as subproblem solver,
	showing that it is equivalent to
	stochastic gradient descent (SGD) on the original (non-sketched)
	NLS subproblem.
	\item For the problem of secure distributed NMF, we propose efficient methods, Syn-SD and Syn-SSD, in synchronous setting and later extend them to asynchronous setting. Through sketching, their computation cost is significantly reduced. They are the first secure distributed NMF methods for federated data.
	\item We conduct extensive experiments using several (dense and sparse) real datasets, which demonstrates the efficiency and scalability of our proposals.
\end{itemize}

The remainder of the paper is organized as follows. Sec.~\ref{sec:bg}
provides the background and discusses the related work. Our DSANLS algorithm
with detailed theoretical analysis is presented in Sec.~\ref{sec:DSANLS}. Our
proposed algorithms for secure distributed NMF problem in both synchronous and
asynchronous settings are presented in Sec.~\ref{sec:secure}.
Sec.~\ref{sec:experiment} evaluates all algorithms. Finally,
Sec.~\ref{sec:discuss} concludes the paper.

\vspace{5pt}
\noindent\textbf{Notations.} For a matrix $A$, we use $A_{i:j}$ to denote the
entry at the $i$-th row and $j$-th column of $A$. Besides, either $i$ or $j$
can be omitted to denote a column or a row, i.e., $A_{i:}$ is the $i$-th row
of $A$, and $A_{:j}$ is its $j$-th column. Furthermore, $i$ or $j$ can be
replaced by a subset of indices. For example, if $I\subset\{1,2,\dots,m\}$,
$A_{I:}$ denotes the sub-matrix of $A$ formed by all rows in $I$, whereas
$A_{:J}$ is the sub-matrix of $A$ formed by all columns in a subset
$J\subset\{1,2,\dots,n\}$.

\hide{
For a matrix $A$, sub-matrix $A_{i}$ is used to denote $i$-th part of $A$. Besides, $A_{(i)}$ is used to specify local full copy of $A$ stored in node $i$. Note that $A_{(i)} \neq A_{(j)}$ sometimes stands for $i \neq j$. When $A$ is same across all nodes, we can also call local copy of $A$ as $A_{\bar{i}}$ of each node $i$. Therefore, $A_{\bar{i}} = A_{\bar{j}}$ always stands for $i \neq j$.
}
\section{Background and Related Work}
\label{sec:bg}
In Sec.~\ref{sec:pre}, we first illustrate NMF and its security problem in a distributed environment. Then we elaborate on previous works which are related to this paper in Sec.~\ref{sec:related_work}.

\subsection{Preliminary}
\label{sec:pre}

\subsubsection{NMF Algorithms}
\label{sec:nmf:algo}
Generally, NMF can be defined as an optimization
problem~\cite{lee2001algorithms} as follows: 
\begin{equation} 
\label{eq:problem}
\min_{\substack{U\in\R_+^{m\times k}, V\in\R_+^{n\times k}}} \left\|M-UV\T\right\|_F,
\end{equation}
where $||X||_F=\left(\sum_{ij}x_{ij}^2 \right)^{1/2}$ is the Frobenius norm of $X$.
Problem \eqref{eq:problem} is hard to solve directly because it is non-convex. Therefore,
almost all NMF algorithms leverage two-block coordinate descent schemes (shown in Alg.~\ref{nls}): they
optimize over one of the two factors, $U$ or $V$, while keeping the other
fixed~\cite{gillis2014and}. By fixing $V$, we can optimize $U$ by solving a
nonnegative least squares (NLS) subproblem: 
\begin{equation} 
\label{eq:subproblem}
\min_{U\in\R_+^{m\times k}} \left\|M-UV\T\right\|_F.
\end{equation}
Similarly, if we fix $U$, the problem becomes:
\begin{equation} 
\label{eq:subproblem_v}
\min_{V\in\R_+^{n\times k}} \left\|M\T-VU\T\right\|_F.
\end{equation}

\begin{algorithm}[!t]
	\caption{Two-Block Coordinate Descent: Framework of Most NMF Algorithms}
	\label{nls}
	\small
	\textbf{Input}: $M$\\
	\textbf{Parameter}: Iteration number $T$
	\begin{algorithmic}[1] 
		\State initialize $U^0\geq0$, $V^0\geq0$
		\For {$t=0$ \textbf{to} $T-1$}
		\State $U^{t+1}\leftarrow $ update($M$, $U^{t}$, $V^{t}$)
		\State $V^{t+1}\leftarrow$ update($M$, $U^{t+1}$, $V^{t}$)
		\EndFor
		\State \textbf{return} $U^T$ and $V^T$
	\end{algorithmic}
\end{algorithm}

Within two-block coordinate descent schemes (exact or inexact), different
subproblem solvers are proposed. The first widely used update rule is
Multiplicative Updates (MU)~\cite{daube1986iterative,lee2001algorithms}. MU is
based on the majorization-minimization framework and its application
guarantees that the objective function monotonically decreases
\cite{daube1986iterative,lee2001algorithms}. Another extensively studied
method is alternating nonnegative least squares (ANLS), which represents a
class of methods where the subproblems for $U$ and $V$ are solved exactly
following the framework described in Alg.~\ref{nls}. ANLS is guaranteed to
converge to a stationary point \cite{grippo2000convergence} and has been shown
to perform very well in practice with active set \cite{kim2008nonnegative,
kim2011fast}, projected gradient \cite{lin2007projected}, quasi-Newton
\cite{zdunek2006non}, or accelerated gradient \cite{guan2012nenmf} methods as
the subproblem solver. Therefore, we focus on ANLS in this paper. 


\subsubsection{Secure Distributed NMF}
Secure distributed NMF problem is meaningful with practical applications. Suppose two hospitals $A$ and $B$ have different clinical records, $M_1$ and $M_2$ (i.e., matrices), for same set of phenotypes. For legal or commercial concerns, it is required that none of the hospitals can reveal personal records to another directly. For the purpose of phenotype classification, NMF task can be applied independently (i.e., $M_1\approx U_1V_1\T$ and $M_2\approx U_2V_2\T$). However, since $M_1$ and $M_2$ have the same schema for phenotypes, the concatenated matrix $M=[M_1, M_2]$ can be taken as input for NMF and results in better user (i.e., patients) latent representations $V_1$ and $V_2$ by sharing the same item (i.e., phenotypes) latent representation $U$:
\begin{equation}
M=\begin{bmatrix}M_1  & M_2\\\end{bmatrix}
\approx \begin{bmatrix}U V_1\T  & U V_2\T \\\end{bmatrix} 
= U \cdot \begin{bmatrix}V_1\T  & V_2\T \\\end{bmatrix}.
\end{equation} 

Throughout the factorization process, a \emph{secure} distributed NMF should
guarantee that party $A$ only has access to $M_1$, $U$ and $V_1$ and party $B$
only has access to $M_2$, $U$ and $V_2$. It is worth noting that the above problem of distributed NMF with two parties can be straightforwardly extended to $N$ parties. The requirement of all parties over federated data in secure distributed NMF is actual the so-called \emph{$t$-private protocol} (shown in Definition~\ref{def:honest} with $t=N-1$) which derives from secure function evaluation~\cite{NaorN01}. In this paper, we will use it to assess whether a distributed NMF is \emph{secure}.
\begin{definition}{($t$-private protocol).}
\label{def:honest}
All $N$ parties follow the protocol honestly, but they are also curious about inferring
other party's private information based on their own data (i.e., honest-but-curious). A protocol is $t$-private if any $t$ parties who collude at the end of the protocol learn nothing beyond their own outputs.
\end{definition}

Note that a single matrix transpose
operation transforms a column-concatenated matrix to a row-concatenated
matrix. Without loss of generality, we only consider the scenario that
matrices are concatenated along rows in this paper.


Secure distributed NMF problem is hard in nature. Firstly, party $A$ needs to solve the NMF problem to get $U$ and $V_1$ together with party $B$. At the same time, party $A$ should not be able to infer $V_2$ or $M_2$ during the whole process. Such secure requirement makes it totally different from traditional distributed NMF problem, whose data partition already incurs secure violation.

\subsection{Related Work}
\label{sec:related_work}
In the sequel, we briefly review three research areas which are related to this paper.

\subsubsection{Accelerating NMF}
\label{sec:related_work:acc}
Parallel NMF algorithms are well studied in the literature~\cite{kanjani2007parallel,robila2006parallel}. However, different from a parallel and single machine setting, data sharing and communication have considerable cost in a distributed setting. Therefore, we need specialized NMF algorithms for massive scale data handling in a distributed environment.
The first method in this direction~\cite{liu2010distributed} is based on the MU algorithm.
It mainly focuses on sparse matrices and applies a careful partitioning of the data in order to maximize data locality and parallelism.
Later, CloudNMF~\cite{liao2014cloudnmf}, a MapReduce-based NMF algorithm similar to~\cite{liu2010distributed}, was implemented and tested on large-scale biological datasets.
Another distributed NMF algorithm~\cite{yin2014scalable} leverages block-wise updates for local aggregation and parallelism. It also performs frequent updates using whenever possible the most recently updated data, which is more efficient
than traditional concurrent counterparts. Apart from MapReduce implementations, Spark is also attracting attention for its advantage in iterative algorithms, e.g., using MLlib~\cite{meng2016mllib}. Finally, there are implementations using X10~\cite{grove2014supporting} and on GPU~\cite{mejia2015nmf}.

The most recent and related work in this direction is MPI-FAUN~\cite{kannan2016high,kannan2016mpi}, which is the first implementation of NMF using MPI for interprocessor communication. MPI-FAUN is flexible and can be utilized for a broad class of NMF algorithms that iteratively solve NLS subproblems including
MU, HALS, and ANLS/BPP.
MPI-FAUN exploits the independence of local update computation for rows of $U$ and $V$
to apply communication-optimal matrix multiplication. In a nutshell, the full matrix $M$ is split across a two-dimensional grid of processors and multiple copies of both $U$ and $V$ are kept at different nodes, in order to reduce the communication between nodes during the iterations of NMF algorithms.

\subsubsection{Matrix Sketching}\label{sec:sketch}
Matrix sketching is a technique that has been previously used in numerical linear algebra~\cite{gower2015randomized}, statistics \cite{PilanciW16} and optimization \cite{PilanciW17}. Its basic idea is described as follows. Suppose we need to find a solution $x$ to the equation: $Ax=b,\quad (A\in\R^{m\times n},\ b\in\R^{m})$.
Instead of solving this equation directly, in each iteration of matrix sketching, a random matrix $S\in\R^{d\times m}$ $(d\ll m)$ is generated, and we instead solve the following problem: $(SA)x=Sb$.
Obviously, the solution to the first equation is also a solution to the second equation, but not vice versa. However, the problem size has now decreased from $m\times n$ to $d\times n$. With a properly generated random matrix $S$ and an appropriate method to solve subproblem in the second equation, it can be guaranteed that we will progressively approach the solution to the first equation by iteratively applying this sketching technique.


To the best of our knowledge, there is only one piece of previous work \cite{wang2010efficient} which incorporates dual random projection into the NMF problem, in a centralized environment,
sharing similar ideas as SANLS,
the centralized version of our DSANLS algorithm.
However, Wang et al. \cite{wang2010efficient} did not provide an efficient subproblem solver, and their method was less effective than non-sketched methods in practical experiments.
Besides, data sparsity was not taken into consideration in their work.
Furthermore, no theoretical guarantee was provided for NMF with dual random projection. 
In short, SANLS is not same as \cite{wang2010efficient} and DSANLS is much more than a distributed version of \cite{wang2010efficient}. The methods that we propose in this paper are efficient in practice and have strong theoretical guarantees.

\subsubsection{Secure Matrix Computation on Federated Data}
In federated data mining, parties collaborate to perform data processing task
on the union of their unencrypted data, without leaking their private data to
other participants~\cite{lindell2000privacy}.  A surge of work in the
literature studies federated matrix computation algorithms, such as
privacy-preserving gradient descent~\cite{WanNHL07,HanNWL10}, eigenvector
computation~\cite{PathakR10a}, singular value
decomposition~\cite{han2009privacy,ChenLZ17}, \emph{k}-means
clustering~\cite{SakumaK10}, and spectral clustering~\cite{LinJ11} over
partitioned data on different parties.  Secure multi-party computation (MPC)
are applied to preserve the privacy of the parties involved (e.g. secure
addition, secure multiplication and secure dot
product)~\cite{DuanC08,SakumaK10}.  
These Secure MPC protocols compute arbitrary function among $n$ parties and tolerate up to $t<(1/2)n$ corrupted parties, at a cost $\Omega(n)$ per bit~\cite{beerliova2008perfectly, damgaard2007scalable}. These protocols are too generic when it comes to a specific task like secure NMF. Our proposed protocol does not incorporate costly MPC multiplication protocols while tolerates up to $n$-1 corrupted (static, honest but curious) parties.
Recently, \citet{KimSYJ17}
proposed a federated method to learn phenotypes across multiple hospitals with
alternating direction method of multipliers (ADMM) tensor factorization; and
\citet{feng2018privacy} developed a privacy-preserving tensor decomposition
framework for processing encrypted data in a federated cloud setting.  

\section{DSANLS: Distributed Sketched ANLS}\label{sec:DSANLS}
In this section, we illustrate our DSANLS method for accelerating NMF in general distributed environment.

\subsection{Data Partitioning}\label{sec:DSANLS:part}
Assume there are $N$ computing nodes in the cluster. We partition the row indices $\{1,2,\dots,m\}$ of the input matrix $M$ into $N$ disjoint sets $I_1,I_2,\dots,I_N$, where $I_r\subset\{1,2,\dots,m\}$
is the subset of rows assigned to node $r$, as in~\cite{liu2010distributed}.
Similarly,
we partition the column indices $\{1,2,\dots,n\}$ into disjoint sets $J_1,J_2,\dots,J_N$ and assign column set $J_r$ to node $r$.
The number of rows and columns in each node are near the same in order to achieve load balancing,
i.e.,
$\left|I_r\right|\approx m/N$ and $\left|J_r\right|\approx n/N$ for each node $r$.
The factor matrices $U$ and $V$ are also assigned to nodes accordingly,
i.e., node $r$
stores and updates $U_{I_r:}$ and $V_{J_r:}$ as shown in Fig.~\ref{fig:storage}.


\begin{figure}[!t]
  \centering
  \subfigure[DSANLS]{
    \label{fig:storage}
    \includegraphics[width=0.53\linewidth]{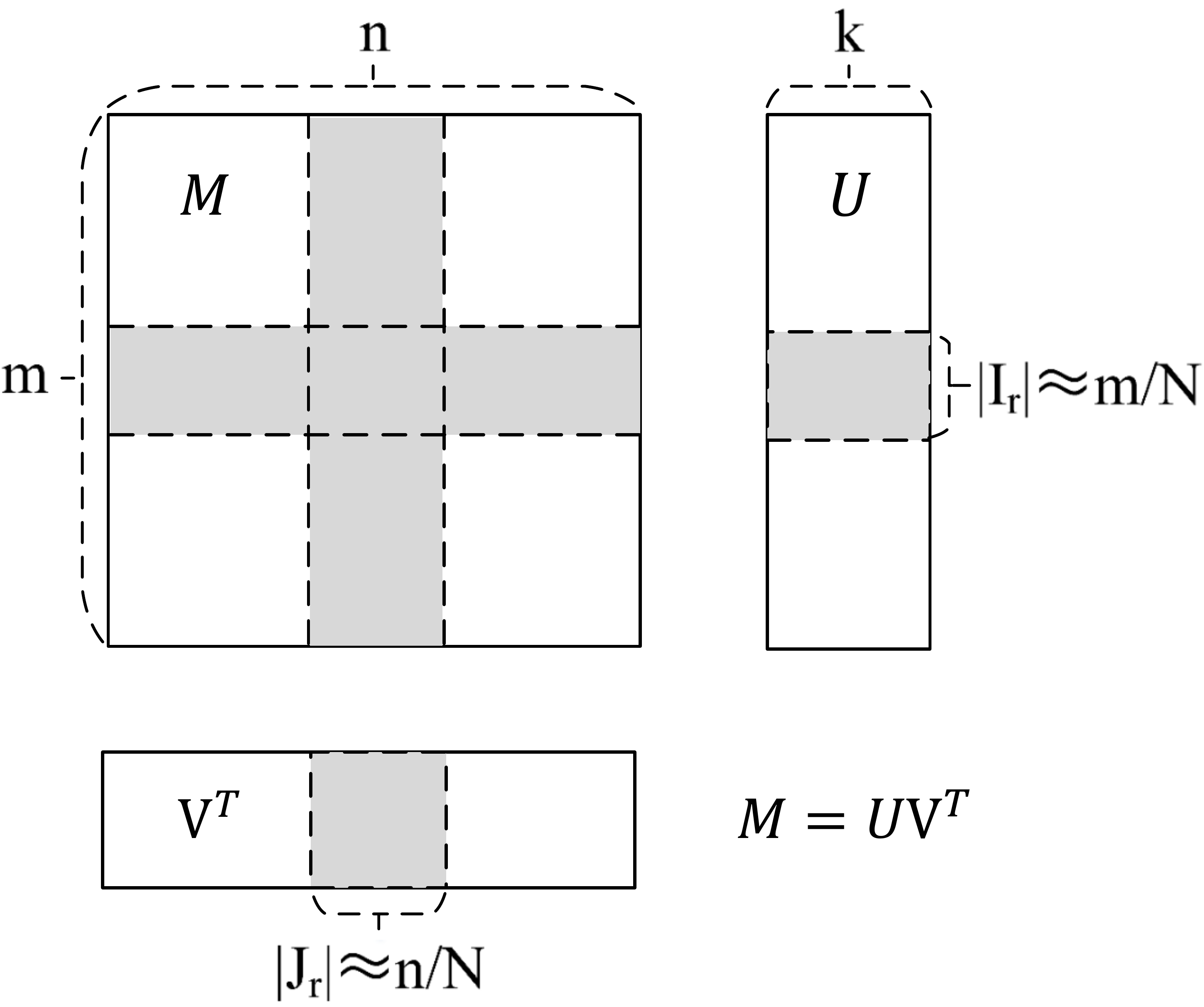}}
  \hspace{0.1in}
  \subfigure[Secure NMF]{
    \label{fig:secure_storage}
    \includegraphics[width=0.4\linewidth]{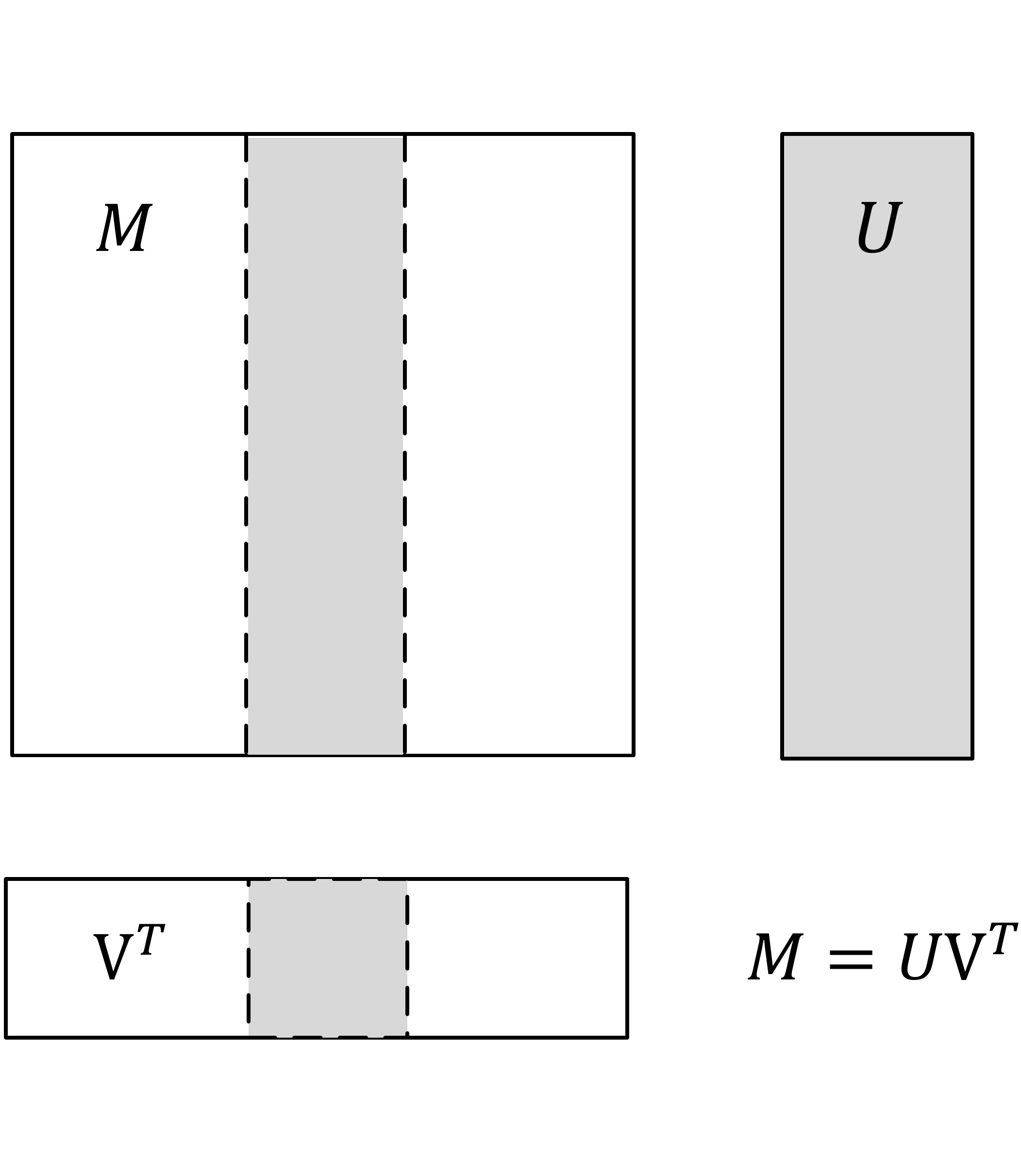}}
  \caption{Partitioning data to $N$ nodes with node $r$'s data shaded.}
\end{figure}

Data partitioning in distributed NMF differs from that in parallel NMF. Previous works on parallel NMF \cite{kanjani2007parallel,robila2006parallel} choose to partition $U$ and $V$ along the long dimension, but we adopt the row-partitioning of $U$ and $V$ as in~\cite{liu2010distributed}. To see why,
take the $U$-subproblem \eqref{eq:subproblem} as an example
and observe that it is
row-independent in nature, i.e., the $r$-th row block of its solution $U_{I_r:}$ is given by:
\begin{equation} 
  \label{eq:u_block}
  U_{I_r:}=  \argmin_{U_{I_r:}\in\mathbb{R}_+^{|I_r|\times k}}  \left\|M_{I_r:} - U_{I_r:} V^\top\right\|^2_F,
\end{equation}
and thus can be solved independently without referring to any other row blocks of $U$.
The same holds for the $V$-subproblem.
In addition, no communication is needed concerning $M$ when solving \eqref{eq:u_block} because
$M_{I_r:}$ is already present in node $r$.

On the other hand,
solving \eqref{eq:u_block} requires the entire $V$ of size $n\times k$, meaning that every node needs to gather $V$ from all other nodes. This process can easily be
the bottleneck of a naive distributed ANLS implementation.
As we will explain shortly, our DSALNS algorithm alleviates this problem, since we use a sketched matrix of reduced size instead of the original complete matrix $V$.

\subsection{SANLS: Sketched ANLS}
To better understand DSANLS, we first introduce the Sketched ANLS (SANLS), i.e., a centralized version of our algorithm.
Recall that in Sec.~\ref{sec:nmf:algo}, at each step of ANLS, either $U$ or $V$ is fixed and we solve a nonnegative least square problem \eqref{eq:subproblem} or \eqref{eq:subproblem_v} over the other variable.
Intuitively, it is unnecessary to solve this subproblem with high accuracy, because we may not have reached the optimal solution for the fixed variable so far.
Hence, when the fixed variable changes in the next step, this accurate solution from the previous step will not be optimal anymore and will have to be re-computed. Our idea is to
apply matrix sketching for each subproblem, in order to obtain an approximate solution for it at a much lower computational and communication cost.

Specifically, suppose we are at the $t$-th iteration of ANLS, and our current estimations for $U$ and $V$ are $U^t$ and $V^t$ respectively. We must solve subproblem \eqref{eq:subproblem} in order to update $U^t$ to a new matrix $U^{t+1}$.
We apply matrix sketching to the residual term of subproblem \eqref{eq:subproblem}. The subproblem now becomes:
\begin{equation}\label{eq:sketched_u}
\min_{U\in\R_+^{m\times k}} \left\|M S^t - U \left(V^{t\top} S^t\right) \right\|_F^2,
\end{equation}
where $S^t\in\R^{n\times d}$ is a randomly-generated matrix. Hence, the problem size decreases from $n\times k$ to $d\times k$.
$d$ is chosen to be much smaller than $n$, in order to sufficiently reduce the computational cost%
\footnote{However, we should not choose an extremely small $d$, otherwise the
the size of sketched subproblem would become so small that it can hardly represent the original subproblem, preventing NMF from converging to a good result.
In practice, we can set $d=0.1n$ for medium-sized matrices and $d=0.01n$ for large matrices if $m\approx n$. When $m$ and $n$ differ a lot, e.g., $m\ll n$ without loss of generality, we should not apply sketching technique to the $V$ subproblem (since solving the $U$ subproblem is much more expensive) and simply choose $d=m\ll n$.
}.
Similarly, we transform the $V$-subproblem into:
\begin{equation}\label{eq:sketched_v}
\min_{V\in\R_+^{n\times k}} \left\|M\T S'^t - V \left(U^{t\top} S'^t\right) \right\|_F^2,
\end{equation}
where $S'^t\in \mathbb{R}^{m\times d'}$ is also a random matrix with $d'\ll m$.

\subsection{DSANLS: Distributed SANLS}\label{sec:distributed SANLS}
Now, we come to our proposal: the distributed version of SANLS called DSANLS. Since the $U$-subproblem \eqref{eq:sketched_u} is the same as the $V$-subproblem \eqref{eq:sketched_v} in nature, here we restrict our attention to the $U$-subproblem.
The first observation about subproblem \eqref{eq:sketched_u} is that it is still row-independent, thus node $r$ only needs to solve:
\begin{equation}
  \label{eq:solve}
  \min_{U_{I_r:}\in\mathbb{R}^{|I_r|\times k}_+} \left\|\left(M S^t\right)_{I_r:} - U_{I_r:} \left(V^{t\top} S^t\right) \right\|_F^2.
\end{equation}
For simplicity, we denote:
\begin{equation}\label{eq:def_A_B}
  A^t_{r}\triangleq \left(MS^t\right)_{I_r:}\quad\text{and}\quad \ B^t\triangleq V^{t\top} S^t,
\end{equation}
and the subproblem \eqref{eq:solve} can be written as:
\begin{equation}\label{eq:sketched_sub}
  \min_{U_{I_r:}\in\mathbb{R}^{|I_r|\times k}_+} \left\|A^t_r - U_{I_r:} B^t \right\|_F^2.
\end{equation}
Thus, node $r$ needs to know matrices $A^t_r$ and $B^t$ in order to solve the subproblem.

For $A^t_r$, by applying matrix multiplication rules, we get
$A^t_r=\left(MS^t\right)_{I_r:} = M_{I_r:}S^t$. Therefore, if $S^t$ is stored
at node $r$, $A^t_r$ can be computed without any communication.
On the other hand, computing $B^t=\left(V^{t\top} S^t\right)$ requires communication across the whole cluster, since the rows of $V^t$ are distributed across different nodes. Fortunately, if we assume that $S^t$ is stored at all nodes again, we can compute $B^t$ in a much cheaper way. Following block matrix multiplication rules, we can rewrite $B^t$ as:
\begin{scriptsize}
\begin{equation}
  B^t=V^{t\top} S^t = \left[\left(V^t_{J_1:}\right)\T\ \cdots\ \left(V^t_{J_N:}\right)\T\right]
  \left[
  \begin{array}{c}
  S^t_{J_1:} \\
  \vdots \\
  S^t_{J_N:}
  \end{array}
  \right]
  =\sum_{r=1}^N \left(V^t_{J_r:}\right)\T S^t_{J_r:}.
\end{equation}
\end{scriptsize}
Note that the summand $\bar{B}^t_r\triangleq\left(V^t_{J_r:}\right)\T S^t_{J_r:}$ is a matrix of size $k\times d$ and can be computed locally. As a result, communication is only needed for summing up the matrices $\bar{B}^t_r$ of size $k\times d$ by using MPI all-reduce operation, which is much cheaper than transmitting the whole $V_t$ of size $n\times k$.

Now, the only remaining problem is the transmission of $S^t$. Since $S^t$ can be dense, even larger than $V^t$, broadcasting it across the whole cluster can be quite expensive. However, it turns out that we can avoid this.
Recall that $S^t$ is a randomly-generated matrix; each node can generate exactly the same matrix, if we use the same pseudo-random generator and the same seed.
Therefore, we only need to broadcast the random seed, which is just an integer, at the beginning of the whole program. This ensures that each node generates exactly the same random number sequence and hence the same random matrices $S^t$ at each iteration.

In short, the communication cost of each node is reduced from $\mathcal{O}(n k)$ to $\mathcal{O}(d k)$ by adopting our sketching technique for the $U$-subproblem. Likewise, the communication cost of each $V$-subproblem is decreased from $\mathcal{O}\left(m k\right)$ to $\mathcal{O}\left(d' k\right)$. The general framework of our DSANLS algorithm is listed in Alg.~\ref{alg:disSANLS}.

\begin{algorithm}[!t]
  \caption{Distributed SANLS on Node $r$}
  \label{alg:disSANLS}
  \small
  \textbf{Input}: $M_{I_r:}$ and $M_{:J_r}$\\
  \textbf{Parameter}: Iteration number $T$
  \begin{algorithmic}[1] 
    \State {Initialize $U^0_{I_r:}\geq0$, $V^0_{J_r:}\geq0$}
    \State {Broadcast the random seed}
    \For {$t=0$ \textbf{to} $T-1$}
       \State {Generate random matrix $S^t\in\R^{n\times d}$ }
       \State {Compute $A^t_r\leftarrow M_{I_r:}S^t$ }
       \State {Compute $\bar{B}^t_r\leftarrow \left(V^t_{J_r:}\right)\T S^t_{J_r:} $ }
       \State {All-Reduce: $B^t\leftarrow\sum_{i=1}^N \bar{B}^t_i$ }
       \State {Update $U^{t+1}_{I_r:}$ by solving $\min_{U_{I_r:}}\|A^t_r - U_{I_r:} B^t\|$ }
       \State
       \State {Generate random matrix $S'^t\in\R^{m\times d'}$ }
       \State {Compute $A'^t_r\leftarrow\left(M_{:J_r}\right)\T S'^t$ }
       \State {Compute $\bar{B}'^t_r\leftarrow \left(U^t_{I_r:}\right)\T S'^t_{I_r:}$ }
       \State {All-Reduce: $B'^t\leftarrow\sum_{i=1}^N \bar{B}'^t_i$ }
       \State {Update $V^{t+1}_{J_r:}$  by solving $\min_{V_{J_r:}}\|A'^t_r - V_{J_r:} B'^t\|$ }
    \EndFor
    \State {\textbf{return} $U^T_{I_r:}$ and $V^T_{J_r:}$}
  \end{algorithmic}
\end{algorithm}

\subsection{Generation of Random Matrices}\label{sec:generate matrices}
A key problem in Alg.~\ref{alg:disSANLS} is how to generate random matrices $S^t\in\mathbb{R}^{n\times d}$ and $S'^t\in\mathbb{R}^{m\times d'}$. Here we focus on generating a random $S^t\in\R^{d\times n}$ satisfying Assumption \ref{assumption:1}. The reason for choosing such a random matrix is that the corresponding sketched problem would be equivalent to the original problem on expectation; we will prove this in Sec.~\ref{sec:solving subproblems}.
\begin{assumption}\label{assumption:1}
Assume the random matrices are normalized and have bounded variance, i.e.,
 there exists a constant $\sigma^2$ such that $\mathbb{E}\left[S^t S^{t\top}\right]=I\quad\text{and}\quad\mathbb{V}\left[S^t S^{t\top}\right] \leq \sigma^2$
for all $t$, where $I$ is the identity matrix.
\end{assumption}

Different options exist for such matrices, which
have different computation costs in forming sketched matrices $A^t_r=M_{I_r:} S^t$ and $\bar{B}^t_r=\left(V^t_{J_r:}\right)\T S^t_{J_r:}$.
Since $M_{I_r:}$ is much larger than $V^t_{J_r:}$ and thus computing $A^t_r$ is more expensive, we only consider the cost of constructing $A^t_r$ here.

The most classical choice for a random matrix is one with i.i.d. Gaussian entries having mean 0 and variance $1/d$.
It is easy to show that $\mathbb{E}\left[S^t S^{t\top}\right]=I$. Besides, Gaussian random matrix has bounded variance because Gaussian distribution has finite fourth-order moment. However, since each entry of such a matrix is totally random and thus no special structure exists in $S^t$, matrix multiplication will be expensive. That is, when given $M_{I_r:}$ of size $|I_r|\times n$, computing its sketched matrix $A^t_r=M_{I_r:}S^t$ requires $\mathcal{O}(|I_r|n d)$ basic operations.

A seemingly better choice for $S^t$ would be a {\em subsampling} random matrix. Each column of such random matrix is uniformly sampled from $\{e_1,e_2,\dots,e_n\}$ without replacement,
where $e_i\in\R^n$ is the $i$-th canonical basis vector (i.e., a vector having its $i$-th element 1 and all others 0). We can easily show that such an $S^t$ also satisfies $\mathbb{E}\left[S^t S^{t\top}\right]=I$ and the variance $\mathbb{V}\left[S^t S^{t\top}\right]$ is bounded, but this time
constructing the sketched matrix $A^t_r=M_{I_r:}S^t$ only requires $\mathcal{O}\left(|I_r|d\right)$. Besides, subsampling random matrix can preserve the sparsity of original matrix.
Hence, a subsampling random matrix would be favored over a Gaussian random matrix by most applications, especially for very large-scale or sparse problems.
On the other hand, we observed in our experiments that a Gaussian random matrix can result in a faster per-iteration convergence rate, because each column of the sketched matrix $A^t_r$ contains entries from multiple columns of the original matrix and thus is more informative. Hence, it would be better to use a Gaussian matrix when the sketch size $d$ is small and thus a $\mathcal{O}(|I_r|n d)$ complexity is acceptable, or
when the network speed of the cluster is poor, hence we should trade more local computation cost for less communication cost.

Although we only test two representative types of random matrices
(i.e., Gaussian and subsampling random matrices), our framework is
readily applicable for other choices, such as subsampled
randomized Hadamard transform (SRHT) \cite{ailon2006approximate,lu2013faster} and count sketch \cite{clarkson2013low,pham2013fast}. 
The choice of random matrices is not the focus of this paper and
left for future investigation.

\subsection{Solving Subproblems}\label{sec:solving subproblems}
Before describing how to solve subproblem \eqref{eq:sketched_sub},
let us make an important observation. As discussed in Sec.~\ref{sec:sketch}, the sketching technique has been applied in solving linear systems before.
However, the situation is different in matrix factorization. Note that for the distributed matrix factorization problem we usually have:
\begin{equation}
\min_{U_{I_r:}\in\mathbb{R}_+^{|I_r|\times k}}\left\|M_{I_r:}-U_{I_r:}V^{t\top}\right\|_F^2 \neq 0.
\end{equation}
So, for the sketched subproblem \eqref{eq:sketched_sub}, which can be equivalently written as:
\begin{equation}
\min_{U_{I_r:}\in\mathbb{R}_+^{|I_r|\times k}}\left\|\left(M_{I_r:}-U_{I_r:}V^{t\top}\right)S^t\right\|_F^2,
\end{equation}
where the non-zero entries of the residual matrix
$\left(M_{I_r:}-U_{I_r:}V^{t\top}\right)$ will be scaled by the matrix
$S^t$ at different levels. As a consequence, the optimal solution will
be shifted because of sketching. This fact alerts us that for SANLS,
we need to update $U^{t+1}$ by exploiting the sketched subproblem
\eqref{eq:sketched_sub} to step towards the true optimal solution and
avoid convergence to the solution of the sketched subproblem.

\subsubsection{Projected Gradient Descent}\label{sec:gradient descent}
A natural method is to use \emph{one step}\footnote{Note that we only apply one step of projected gradient descent here to avoid solution shifted. } of projected gradient descent for the sketched subproblem:
\begin{equation}
\label{eq:gradient_descent}
\begin{aligned}
U^{t+1}_{I_r:} &= \max\left\{U^t_{I_r:} - \eta_t \left.\nabla_{U_{I_r:}} \left\|A^t_r - U_{I_r:} B^t \right\|_F^2\right|_{U_{I_r:}=U^t_{I_r:}},\ 0\right\} \\
&= \max\left\{U^{t}_{I_r:} - 2\eta_t\left[ U^t_{I_r:} B^t B^{t\top} -  A^t_r B^{t\top} \right], \ 0\right\}, 
\end{aligned}
\end{equation}
where $\eta_t>0$ is the step size and $\max\{\cdot,\cdot\}$ denotes the entry-wise maximum operation. In the gradient descent step \eqref{eq:gradient_descent}, the computational cost mainly comes from two matrix multiplications: $B^t B^{t\top}$ and $A_{t,r} B^{t\top}$. Note that $A^t_r$ and $B^t$ are of sizes $|I_r|\times d$ and $k\times d$ respectively, thus the gradient descent step takes $\mathcal{O}\left(kd(|I_r|+k)\right)$ in total.

To exploit the nature of this algorithm, we further expand the gradient:
\begin{equation}
\begin{aligned}
  &\nabla_{U_{I_r:}} \left\|A^t_r - U_{I_r:} B^t \right\|_F^2
  =2\left[U_{I_r:} B^t B^{t\top} -  A^t_r B^{t\top}\right] \\
  \stackrel{\eqref{eq:def_A_B}}{=}& 2\left[ U_{I_r:} \left(V^{t\top} S^t\right)\left(V^{t\top} S^t\right)\T - \left(M_{I_r:}S^t\right)\left(V^{t\top} S^t\right)\T \right] \\
  =& 2\left[ U_{I_r:} V^{t\top}\left( S^t S^{t\top}\right)V^t - M_{I_r:}\left( S^t S^{t\top}\right)V^t \right].
\end{aligned}
\end{equation}
By taking the expectation of the above equation, and using the fact $\mathbb{E}\left[S^tS^{t\top}\right]=I$, we have:
\begin{equation}
\begin{aligned}
&\mathbb{E}\left[ \nabla_{U_{I_r:}} \left\|A^t_r - U_{I_r:} B^t \right\|_F^2 \right]
= 2\left[ U_{I_r:}V^{t\top} V^t - M_{I_r:}V^t \right] \\
=& \nabla_{U_{I_r:}} \left\|M_{I_r:} - U_{I_r:} V^{t\top} \right\|_F^2
\end{aligned}
\end{equation}
which means that the gradient of the sketched subproblem is equivalent to the gradient of the original problem on expectation.
Therefore, such a step of gradient descent can be interpreted as a (generalized) \emph{stochastic gradient descent} (SGD) \cite{nemirovski2009robust} method on the original subproblem. Thus, according to the theory of SGD, we naturally require the step sizes $\left\{\eta_t\right\}$ to be diminishing, i.e., $\eta_t\rightarrow 0$ as $t$ increases.

\subsubsection{Proximal Coordinate Descent}
However, it is well known that the gradient descent method converges
slowly, while the coordinate descent
method, namely the HALS method for NMF, is quite efficient
\cite{gillis2014and}. Still, because of its very fast convergence, HALS should not be applied to the sketched subproblem directly 
because it shifts the solution away from the true optimal solution.
Therefore, we would like to develop a method which resembles HALS but will not converge
towards the solutions of the sketched subproblems.

To achieve this, we add a regularization term to the sketched subproblem \eqref{eq:sketched_sub}. The new subproblem becomes:
\begin{equation} 
\label{eq:reg_sub}
\min_{U_{I_r:}\in\mathbb{R}_+^{|I_r|\times k}} \left\| A^t_r - U_{I_r:}B^t \right\|_F^2 + \mu_t\left\|U_{I_r:} - U^t_{I_r:}\right\|_F^2,
\end{equation}
where $\mu_t>0$ is a parameter.
Such regularization is reminiscent to the proximal point method \cite{rockafellar1976monotone}
and parameter $\mu_t$ controls the step size as $1/\eta_t$ in projected gradient descent. We therefore require $\mu_t\rightarrow +\infty$ to enforce the convergence of the algorithm, e.g., $\mu_t=t$.

At each step of proximal coordinate descent, only one column of $U_{I_r:}$, say $U_{I_r,j}$ where $j\in\{1,2,\dots,k\}$, is updated:
\begin{equation}
  \min_{U_{I_r:j}\in\mathbb{R}_+^{|I_r|}} \bigg\|A^t_r - U_{I_r:j} B^t_{j:} - \sum_{l\neq j} U_{I_r:l} B^t_{l:} \bigg\|_F^2 + \mu_t \left\|U_{I_r:j} - U^t_{I_r:j} \right\|_2^2.
\end{equation}
It is not hard to see that the above problem is still row-independent, which means that each entry of the row vector $U_{I_r:j}$ can be solved independently at each node. For example, for any $i\in I_r$, the solution of $U^{t+1}_{i:j}$ is given by:
\begin{equation}
\label{eq:entry_sol}
\begin{aligned}
U^{t+1}_{i:j}=&\argmin_{U_{i:j}\geq 0} \bigg\| \left(A^t_r\right)_{i:} - U_{i:j}B^t_{j:} - \sum_{l\neq j} U_{i:l}B^t_{l:} \bigg\|_2^2 \\
&+ \mu_t \left\|U_{i:j} - U^t_{i:j}\right\|^2_2 \\
=&\max\left\{\frac{ \mu_t U_{i:j}^t + \left(A^t_r\right)_{i:} B_{j:}^{t\top} - \sum_{l\neq j} U_{i:l}B^{t}_{l:}B^{t\top}_{j:} }{ B^t_{j:} B^{t\top}_{j:} + \mu_t }, 0 \right\}. 
\end{aligned}
\end{equation}

At each step of coordinate descent, we choose the column $j$ from $\{1,2,\dots,k\}$ successively. When updating column $j$ at iteration $t$, the columns $l<j$ have already been updated and thus $U_{I_r:l}=U^{t+1}_{I_r:l}$, while the columns $l>j$ are old so $U_{I_r:l}=U^{t}_{I_r:l}$.

The complete proximal coordinate descent algorithm for the $U$-subproblem is
summarized in Alg.~\ref{alg:cd}. When updating column $j$,
computing the matrix-vector multiplication $A^t_r B_{j:}^{t\top}$
takes $\mathcal{O}(d|I_r|)$. The whole inner loop takes $\mathcal{O}\left(k\left(d+|I_r|\right)\right)$ because one vector dot product of length $d$ is required for computing each summand and the summation itself needs $\mathcal{O}\left(k|I_r|\right)$. Considering that there are $k$ columns in total, the overall complexity of coordinate descent is $\mathcal{O}\left(k(\left(k+d\right)|I_r|+kd)\right)$. Typically, we choose $d>k$, so the complexity can be simplified to $\mathcal{O}\left(kd\left(|I_r|+k\right)\right)$, which is the same as that of gradient descent.

Since proximal coordinate descent is much more efficient than projected gradient descent, we adopt it as the default subproblem solver within DSANLS.

\begin{algorithm}[!t]
  \caption{Proximal Coordinate Descent for Local Subproblem \eqref{eq:sketched_sub}  on Node $r$ }
  \label{alg:cd}
  \small
  \textbf{Parameter:} $\mu_t>0$ 
  \begin{algorithmic}[1] 
    \For{$j=1$ \textbf{to} $k$}
       \State $T\leftarrow \mu_t U_{I_r:j}^t +  A^t_r B_{j:}^{t\top}$ 
       \For{$l=1$ \textbf{to} $j-1$}
           \State $T\leftarrow T - \left(B^t_{l:} B_{j:}^{t\top} \right) U^{t+1}_{I_r:l}$
       \EndFor
       \For{$l=j+1$ \textbf{to} $k$}
           \State $T\leftarrow T - \left(B^t_{l:} B_{j:}^{t\top} \right) U^{t}_{I_r:l}$
       \EndFor
       \State $U^{t+1}_{I_r:j}\leftarrow \max\left\{T / \left(B_{j:}^t B_{j:}^{t\top} + \mu_t\right),0 \right\}$
    \EndFor
    \State \textbf{return} $U_{I_r:}^{t+1}$
  \end{algorithmic}
\end{algorithm}

\subsection{Theoretical Analysis} \label{sec:theoretical}
\subsubsection{Complexity Analysis}\label{sec:complexity}
We now analyze the computational and communication costs of our DSANLS algorithm, when using subsampling random sketch matrices.
The computational complexity at each node is:
\begin{equation}
\label{eq:complexity}
\begin{aligned}
	&\mathcal{O}\big(\overbrace{d}^{\text{generating }S^t} + \overbrace{|I_r|d}^{\text{constructing $A^t_r$ and $B^t$}} + \overbrace{kd(|I_r|+k)}^{\text{solving subproblem}} \big) \\
	=\ &\mathcal{O}\left( kd(|I_r|+k) \right) \approx \mathcal{O}\left( kd\left(\frac{m}{N}+k\right) \right).
\end{aligned}
\end{equation}
Moreover, as we have shown in Sec.~\ref{sec:distributed SANLS}, the communication cost of DSANLS is $\mathcal{O}\left(kd\right)$.

On the other hand, for a classical implementation of distributed HALS \cite{fairbanks2015behavioral}, the computational cost is:
\begin{equation}
\label{eq:naive complexity}
\begin{aligned}
	\mathcal{O}\left(kn\left(|I_r| + k\right)\right)\approx \mathcal{O}\left(kn\left(\frac{m}{N}+k\right)\right)
\end{aligned}
\end{equation}
and the communication cost is $\mathcal{O}\left(kn\right)$ due to
the all-gathering of $V^t$'s.

Comparing the above quantities, we observe an $n/d\gg 1$ speedup of our DSANLS algorithm over HALS in both computation and communication. However,
we empirically observed
that DSANLS has a slower per-iteration convergence rate (i.e., it needs more iterations to converge).
Still, as we will show in the next section, in practice, DSANLS is superior to alternative distributed NMF algorithms, after taking all factors into account.

\subsubsection{Convergence Analysis}\label{sec:convergence}
Here we provide theoretical convergence guarantees for the proposed
SANLS and DSANLS algorithms. We show that SANLS and DSANLS
converge to a stationary point.

To establish convergence result, Assumption \ref{assumption:2} is needed first.



\begin{assumption}\label{assumption:2}
Assume all the iterates $U^t$ and $V^t$ have uniformly bounded norms,
which means that there exists a constant $R$ such that
$\|U^t\|_F\leq R$ and $\|V^t\|_F\leq R$
for all $t$.
\end{assumption}

We experimentally observed that this assumption holds in practice,
as long as the step sizes used are not too large.
Besides, Assumption \ref{assumption:2} can also be enforced by imposing additional constraints, such as:
\begin{equation}
\label{eq:extra_con}
\begin{aligned}
U_{i:l}\leq \sqrt{2\|M\|_F} \quad\text{and}\quad V_{j:l}\leq  \sqrt{2\|M\|_F} \quad \forall i,j,l,
\end{aligned}
\end{equation}
with which we have $R=\max\{m,n\}k \sqrt{2\|M\|_F}$.
Such constraints can be very easily handled by both of our projected
gradient descent and regularized coordinate descent
solvers.
Lemma \ref{lemma:exisit_optimal} shows that
imposing such extra constraints does not prevent us from finding the
global optimal solution.
\begin{lemma}\label{lemma:exisit_optimal}
	If the optimal solution to the original problem \eqref{eq:problem} exists, there is at least one global optimal solution in the domain \eqref{eq:extra_con}.
\end{lemma}

Based on Assumptions \ref{assumption:1} (see Sec.~\ref{sec:generate matrices}) and Assumption \ref{assumption:2}, we now can formally show our main convergence result:
\begin{theorem}\label{theorem:pgd}
	Under Assumptions \ref{assumption:1} 
   and \ref{assumption:2}, if the step sizes satisfy $
	\sum_{t=1}^{\infty} \eta_t = \infty$ and $\sum_{t=1}^{\infty} \eta_t^2 < \infty$, for projected gradient descent, or $\sum_{t=1}^{\infty} 1/\mu_t = \infty$ and $\sum_{t=1}^{\infty} 1/\mu_t^2 < \infty$, for regularized coordinate descent, then SANLS and DSANLS with either sub-problem solver will converge to a stationary point of problem \eqref{eq:problem} with probability 1.
\end{theorem}

The proofs of Lemma \ref{lemma:exisit_optimal} and Theorem \ref{theorem:pgd} can be found in Appendices~\ref{sec:lemma1} and~\ref{proof:Theorem}.

\section{Secure Distributed NMF}
\label{sec:secure}
In this section, we provide our solutions to the problem of secure distributed NMF over federated data.

\subsection{Extend DSANLS to Secure Setting}
DSANLS and all lines of works discussed in Sec.~\ref{sec:related_work:acc} store copies of $M$ across two-dimensional (shown in Fig.~\ref{fig:storage}), and exploit the independence of local update computation for rows of $U$ and $V$ to apply communication-optimal matrix multiplication. They cannot be applied directly to secure distributed NMF setting. The reason is that, in secure distributed NMF setting (shown in Fig.~\ref{fig:secure_storage}), only one column copy is stored in each node, while the others cannot be disclosed.

Nevertheless, DSANLS can be adapted to this secure setting with modification, but only for \emph{one or limited iterations}. The reason is illustrated in Theorem~\ref{theorem:secure}. In modified DSANLS algorithm, each node still takes charge of updating $U_{I_r:}$ and $V_{J_r:}$ as before, but only one copy $M_{:J_r}$ of $M=[M_1, M_2, ..., M_N]$ will be stored in node $r$. Thus, $V$-subproblem is exactly the same as in DSANLS. Differently, we need to use MPI-AllReduce function to gather $M_{:J_r}S^t$ from all nodes before each iteration of $U$-subproblem, so that each node has access to fully sketched matrix $MS^t$ to solve sketched $U$-subproblem. Note that here random matrix $S^t$ not only helps reduce the communication cost from $\mathcal{O}(m n)$ to $\mathcal{O}(m d)$ with a smaller NLS problem, but also conceals the full matrix $M$ in each iteration. 



\begin{theorem}\label{theorem:secure}
	$M$ cannot be recovered only using information about $MS$ (or $SM$) and $S$.
\end{theorem}

\begin{proof}
    Assume $S$ is a square matrix. Given $MS$ (or $SM$) and $S$, we are able to get $M$ by $M=MSS^{-1}$ (or $M=S^{-1}$SM).
	However, the numbers of row and column are highly imbalanced in $S$ and it is not a square matrix. Therefore $M$ cannot be recovered only using information about $MS$ (or $SM$) and $S$.
\end{proof}

However, NMF is an iterative algorithm (shown in Alg.~\ref{nls}). Secure computation in limited iterations cannot guarantee an acceptable accuracy for practical use due to the following reason:
\begin{theorem}\label{theorem:nosecure}
	$M$ can be recovered after enough iterations.
\end{theorem}

\begin{proof}
	If we view $ M \cdot S = MS$ as a system of linear equations with a variable matrix $M$ and constant matrices $S$ and $MS$. Each row of $M$ can be solved by a standard Gaussian Elimination solver, given a sufficient number of ($S$, $MS$) pairs.
\end{proof}

Theorem \ref{theorem:nosecure} suggests that DSANLS algorithm suffers from the dilemma of choosing between information disclosure and unacceptable accuracy, making it impractical to real applications. Therefore, we need to propose new practical solutions to secure distributed NMF.

\subsection{Synchronous Framework}
A straightforward solution to secure distributed NMF is that each node solves a local NMF problem with a local copy of $U$ (denoted as $U_{(r)}$ for node $r$). Periodically, nodes communicate with each other, and update local copy of $U$ to the aggregation of all local copies $U_{(j)}, j\in\{1,\cdots,N\}$ by All-Reduce operation. We name this method as \emph{Syn-SD} under synchronous setting. The detailed algorithm is shown in Alg.~\ref{Syn_SD}. Within inner iterations, every node maintains its own copy of $U$ (i.e., $U_{(r)}$) by solving the regular NMF problem. Every $T_2$ rounds, different local copies of $U$ will be averaged through nodes by using $\sum_{j=1}^N U_{(j)}/N$. Note that, $U_{(r)}$ is one copy of the whole matrix $U$ stored locally in node $r$, while $V_{J_r:}$ is the corresponding part of the matrix $V=[V_{J_1:}, V_{J_2:}, ..., V_{J_N:}]$ stored in node $r$.

In Syn-SD, the local copy $U_{(r)}$ in node $r$ will be updated to a uniform aggregation of local copies from all nodes periodically. Small number of inner iteration $T_2$ incurs large communication cost caused by All-Reduce. Larger $T_2$ may lead to slow convergence, since each node does not share any information of its local copy $U_{(r)}$ inside the inner iterations. 

To improve the efficiency of data exchange, we incorporate matrix sketching to \emph{Syn-SD}, and propose an improved version called \emph{Syn-SSD}. In Syn-SSD, information of local copies is shared across cluster nodes more frequently, with communication overhead roughly the same as Syn-SD. As shown in Alg.~\ref{Syn_SSD}, the sketched version $S^tU_{(r)}$ of the local copy $U_{(r)}$ is exchanged within each inner iteration. There are two advantages of applying matrix sketching: (1) Since the sketched matrix has a much smaller size, All-Reduce operation causes much less communication cost, making it affordable with higher frequency. (2) Solving a sketched NLS problem can also reduce the computation cost due to a reduced problem size of solving $U_{(r)}$ and $V_{J_r:}$ for each node. It is worth noting that $S_1^t$ is exactly the same for each node by using the same seed and generator. The same for $S_2^t$. But $S_1^t$ and $S_2^t$ are not necessarily equivalent. With such a constraint, the algorithm is equivalent to NMF in single-machine environment and the convergence can be guaranteed. 

\begin{algorithm}[tb]
	\caption{Syn-SD: Secure Distributed NMF on node $r$}
	\label{Syn_SD}
	\small
	\textbf{Input}: $M_{:J_r}$\\
	\textbf{Parameter}: Iteration numbers $T_1, T_2$
	\begin{algorithmic}[1] 
		\State initialize $U_{(r)}^0\geq0$, $V_{J_r:}^0\geq0$
		\For {$t_1=0$ \textbf{to} $T_1-1$}
		\For {$t_2=1$ \textbf{to} $T_2$}
		\State $t\leftarrow t_1 \times T_2 + t_2$
		\State $U_{(r)}^{t}\leftarrow $ update($M_{:J_r}$, $U_{(r)}^{t-1}$, $V_{J_r:}^{t-1}$)
		\State $V_{J_r:}^{t}\leftarrow$ update($M_{:J_r}$, $U_{(i)}^{t}$, $V_{J_r:}^{t-1}$)
		\EndFor
		\State All-Reduce: $U_{(r)}^t\leftarrow \frac{\sum_{j=1}^N U_{(j)}^t}{N}$ 
		\EndFor
		\State \textbf{return} $U_{(r)}^t$ and $V_{J_r:}^t$
	\end{algorithmic}
\end{algorithm}

\begin{algorithm}[!t]
	\caption{Syn-SSD: Secure Sketched Distributed NMF on node $r$}
	\label{Syn_SSD}
	\small
	\textbf{Input}: $M_{:J_r}$\\
	\textbf{Parameter}: Iteration numbers $T_1, T_2$
	\begin{algorithmic}[1] 
		\State initialize $U_{(i)}^0\geq0$, $V_{J_r:}^0\geq0$
		\For {$t_1=0$ \textbf{to} $T_1-1$}
		\For {$t_2=1$ \textbf{to} $T_2$}
		\State $t\leftarrow t_1 \times T_2 + t_2$
		\State Generate random matrix $S_1^t$
		\State $U_{(r)}^{t}\leftarrow $ update($M_{:J_r} S_1^t$, $U_{(r)}^{t-1}$, $V_{J_r:}^{t-1} S_1^t$)
		\State Generate random matrix $S_2^t$
		\State All-Reduce: $\overline{SU}^{t}\leftarrow \frac{\sum_{j=1}^{N}S_2^tU_{(j)}^t}{N}$
		\State $V_{J_r:}^{t}\leftarrow$ update($S_2^tM_{:J_r}$, $\overline{SU}^{t}$, $V_{J_r:}^{t-1}$)
		\EndFor
		\State All-Reduce: $U_{(r)}^t \leftarrow \frac{\sum_{j=1}^{N}U_{(j)}^t}{N}$
		\EndFor
		\State \textbf{return} $U_{(r)}^t$ and $V_{J_r:}^t$
	\end{algorithmic}
\end{algorithm}

It is straightforward to see that Syn-SD and Syn-SSD satisfy Definition~\ref{def:honest} and they are $(N-1)$-private protocols, since $V_{J_r:}$ and $M_{:J_r}$ are only seen by node $r$. 

\subsection{Asynchronous Framework}\label{sec:asynchronous}

In Syn-SD and Syn-SSD, each node must stall until all participating nodes
reach the synchronization barrier before the All-Reduce operation. However,
highly imbalanced data in real scenario of federated data mining may cause
severe workload imbalance problem. The synchronization barrier will force
nodes with low workload to halt, making synchronous algorithms less efficient.
In this section, we study secure distributed NMF in an asynchronous (i.e.,
server/client architecture) setting and propose corresponding asynchronous
algorithms.

First of all, we extend the idea of Syn-SD to asynchronous setting and name
the new method \emph{Asyn-SD}. In Asyn-SD, the server (in
Alg.~\ref{asyn_server}) takes full charge of updating and broadcasting $U^t$.
Once received $U^t_{(r)}$ from the client node $r$, the server would update
$U^t$ locally, and return the latest version of $U^t$ back to the client node
$r$ for further computing.  Note that the server may receive local copies of
$U^t$ from clients in an arbitrary order.  Consequently, we cannot use the
same operation of All-Reduce as Syn-SD any more.  Instead, $U^t$ in server
side is updated by the weighted sum of current $U^t$ and newly received local
copy $U^t_{(r)}$ from client node $r$. Here the relaxation weight $\omega^t$
asymptotically converges to 0. Thus a converged $U^t$ is guaranteed on server
side. Our experiments in Sec.~\ref{sec:experiment} suggest that this
relaxation has no harm to factorization convergence.

\begin{algorithm}[!t]
	\caption{\emph{Asyn-SD}, \emph{Asyn-SSD}: Server part}
	\label{asyn_server}
	\small
	\textbf{Parameter}: Relaxation parameter $\rho$
	\begin{algorithmic}[1] 
		\State initialize $U^0\geq0$
		\State $t\leftarrow0$       \Comment{$t$ is the update counter.}
		\While {not stopping} 
		\State Receive $U^t_{(r)}$ from client node $r$
		\State $\omega^t \leftarrow \frac{\rho}{\rho+t}$ \Comment{$\omega^t$ is the relaxation weight.} 
		\State $U^t \leftarrow(1-\omega^t) U^t +\omega^t U^t_{(r)}$
		\State Send $U^t$ back to client node $r$
		\State $t\leftarrow t+1$
		\EndWhile
		\State \textbf{return} $U^t$
	\end{algorithmic}
\end{algorithm}

\begin{algorithm}[!t]
	\caption{\emph{Asyn-SD}, \emph{Asyn-SSD}: Client part of node $r$}
	\label{asyn_client_SD}
	\small
	\textbf{Input}: $M_{:J_r}$\\
	\textbf{Parameter}: Iteration number $T$
	\begin{algorithmic}[1] 
		\State initialize $V_{J_r:}^0\geq0$
		\While {Server not stopping}
		\State Receive $U$ from server
		\State $U_{(r)}^{0}\leftarrow U$
		\For {$t=1$ \textbf{to} $T$}
		\State $V_{J_r:}^{t}\leftarrow$ update($M_{:J_r}$, $U_{(r)}^{t-1}$, $V_{J_r:}^{t-1}$)
		\State $U_{(r)}^{t}\leftarrow $ update($M_{:J_r}$, $U_{(r)}^{t-1}$, $V_{J_r:}^{t}$) \Comment{For Asyn-SSD, replace it with Lines 5-6 of Alg.~\ref{Syn_SSD}.}
		\EndFor
		\State Send $U_{(r)}^{T}$ to server
		\EndWhile
		\State \textbf{return} $V_{J_r:}^T$
	\end{algorithmic}
\end{algorithm}

On the other hand, client nodes of Asyn-SD (in Alg.~\ref{asyn_client_SD})
behave similarly as nodes in Syn-SD. Clients locally solve the standard NMF
problem for $T$ iterations, and then update local $U^t_{(r)}$ by communicating
only with the server node. Unlike Syn-SD, Asyn-SD does not have a global
synchronization barrier. Client nodes in Asyn-SD independently exchange their
local copy $U^t_{(r)}$ with the server without  an All-Reduce operation.

Similarly, Syn-SSD can be extended to its asynchronous version
\emph{Asyn-SSD}. However, the algorithm for clients is more constrained and
conservative in sketching. Note that the random sketching matrices $S_1$ and
$S_2$ (in Alg.~\ref{Syn_SSD}) should be the same across the nodes in the same
summation in order to have a meaningful summation of sketched matrices.
However, enforcing the same $S_2^t$ for updating sketched $U$ will result in a
synchronous All-Reduce operation. Therefore, $U$ cannot be sketched in
asynchronous algorithms and we only consider sketching $V_{J_r:}$ in Asyn-SSD (Line 7 in Alg.~\ref{asyn_client_SD}). The server part of Asyn-SSD is the same as Asyn-SD in
Alg.~\ref{asyn_server}.

Similar to synchronous versions, Asyn-SD and Asyn-SSD satisfy Definition~\ref{def:honest} and they are $(N-1)$-private protocols, since $V_{J_r:}$ and $M_{:J_r}$ are only seen by node $r$. 

\section{Experimental Evaluation}
\label{sec:experiment}
This section includes an experimental evaluation of our algorithms on both
dense and sparse real data matrices. The implementation of our methods is
available at \url{https://github.com/qianyuqiu79/DSANLS}.

\subsection{Setup}

\begin{table}[!t]
  \centering
  \caption{Statistics of datasets}
  \scalebox{0.95}{
    \begin{tabular}{@{}ccccc@{}}
    \toprule
    Dataset        & \#Rows   & \#Columns & Non-zero values & Sparsity \\ \midrule
    BOATS          & 216,000 & 300      & 64,800,000      & 0\%       \\
    MIT CBCL FACE  & 2,429   & 361      & 876,869         & 0\%       \\
    MNIST          & 70,000  & 784      & 10,505,375      & 80.86\%   \\
    GISETTE        & 13,500  & 5,000    & 8,770,559       & 87.01\%   \\
    Reuters (RCV1) & 804,414 & 47,236   & 60,915,113      & 99.84\%   \\
    DBLP           & 317,080 & 317,080  & 2,416,812       & 99.9976\% \\ \bottomrule
    \end{tabular}
  }
  \label{table: Dataset Details}
\end{table}

We use several (dense and sparse) real datasets as \citet{QianTMC18} for
evaluation. They corresponds to different NMF tasks, including video analysis,
image processing, text mining and community detection. Their statistics are
summarized in Tab.~\ref{table: Dataset Details}.

We conduct our experiments on a Linux cluster with 16 nodes. Each node
contains 8-core Intel\circledR\ Core\textsuperscript{TM} i7-3770 CPU @ 1.60GHz
cores and 16 GB of memory. Our algorithms are implemented in C++ using the
Intel\circledR\ Math Kernel Library (MKL) and Message Passing Interface (MPI). 
By default, we use 10 nodes and set the factorization rank $k$ to 100. We also report the impact of different node number (2-16) and $k$ (20-500). We use $\mu_t=\alpha+\beta t$~\cite{Boyd03}, do the grid search for $\alpha$ and $\beta$ in the range of \{0.1, 1, 10\} for each dataset and report the best results. Because
the use of Gaussian random matrices is too slow on large datasets RCV1 and
DBLP, we only use subsampling random matrices for them. 

For the general acceleration of NMF, we assess DSANLS with subsampling and
Gaussian random matrices, denoted by DSANLS/S and DSANLS/G, respectively,
using proximal coordinate descent as the default subproblem solver. As
mentioned in~\cite{kannan2016high,kannan2016mpi}, it is unfair to compare with
a Hadoop implementation. We only compare DSANLS with
MPI-FAUN\footnote{\url{https://github.com/ramkikannan/nmflibrary}}
(MPI-FAUN-MU, MPI-FAUN-HALS, and MPI-FAUN-ABPP implementations), which is the
first and the state-of-the-art C++/MPI implementation with MKL and Armadillo.
For parameters $pc$ and $pr$ in MPI-FAUN, we use the optimal values for each
dataset, according to the recommendations
in~\cite{kannan2016high,kannan2016mpi}.

For the problem of secure distributed NMF, we evaluate all proposed methods:
Syn-SD, Syn-SSD with sketch on $U$ (denoted as
Syn-SSD-U), Syn-SSD with sketching on $V$ (denoted as
Syn-SSD-V), Syn-SSD with sketching on both $U$ and $V$ (denoted
as Syn-SSD-UV), Asyn-SD, Asyn-SSD with sketching on $V$
(denoted as Asyn-SSD-V), using proximal coordinate descent as the
default subproblem solver. We do not list secure building block methods as
baselines, since communication overhead is heavy in these multi-round
handshake protocols and it is unfair to compare them with MPI based methods. 
For example, a matrix sum described by~\citet{DuanC08} results in 5X
communication overhead compared to a MPI all-reduce operation.  

We use the relative error of the low rank approximation compared to the
original matrix to measure the effectiveness of different NMF approaches. This
error measure has been widely used in previous
work~\cite{kannan2016high,kannan2016mpi,kim2014algorithms} and is formally
defined as $\left\|M-UV\T\right\|_F/\left\|M\right\|_F$.

\subsection{Evaluation on Accelerating General NMF}

\subsubsection{Performance Comparison}

\begin{figure*}[!ht]
  \centering
  \subfigure[BOATS]{
    \label{fig:dsanl:boats}
    \includegraphics[width=0.3\linewidth]{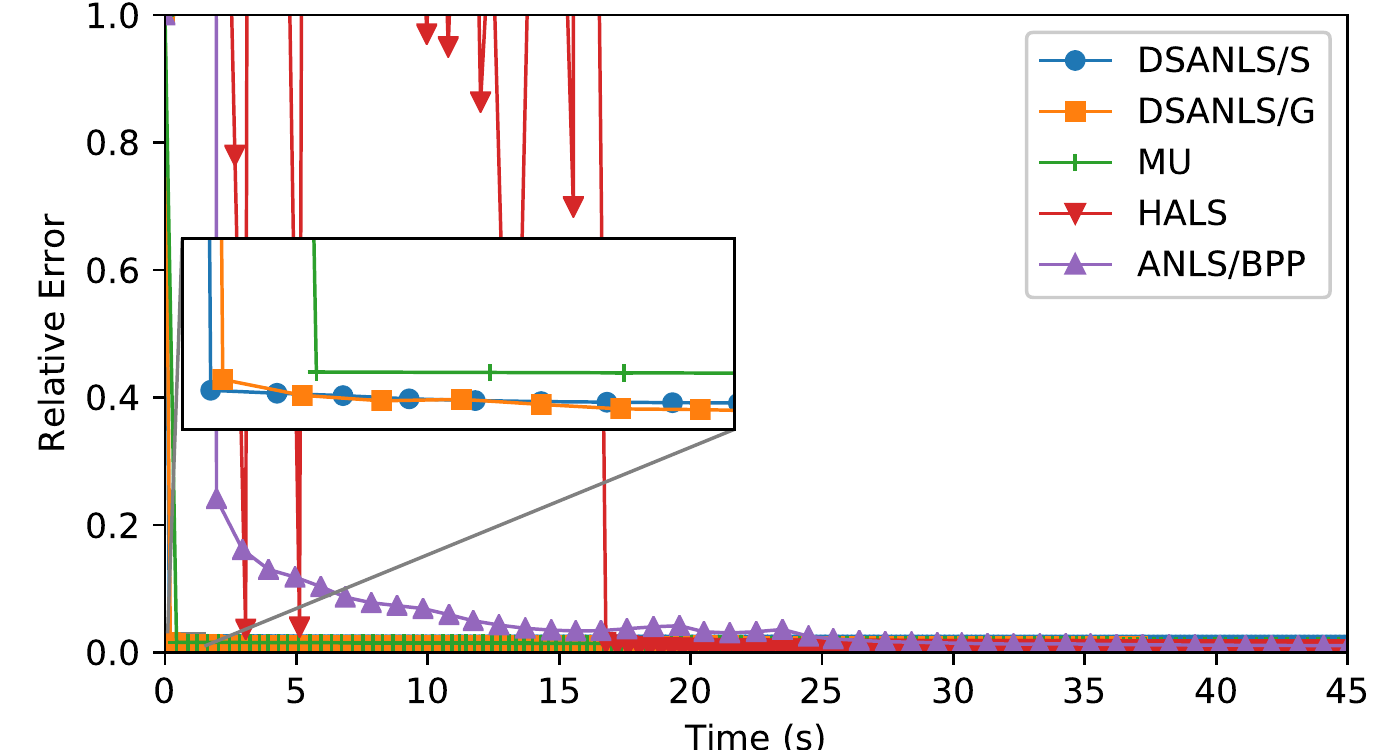}}
  \hspace{0.07in}
  \subfigure[FACE]{
    \label{fig:dsanl:face}
    \includegraphics[width=0.3\linewidth]{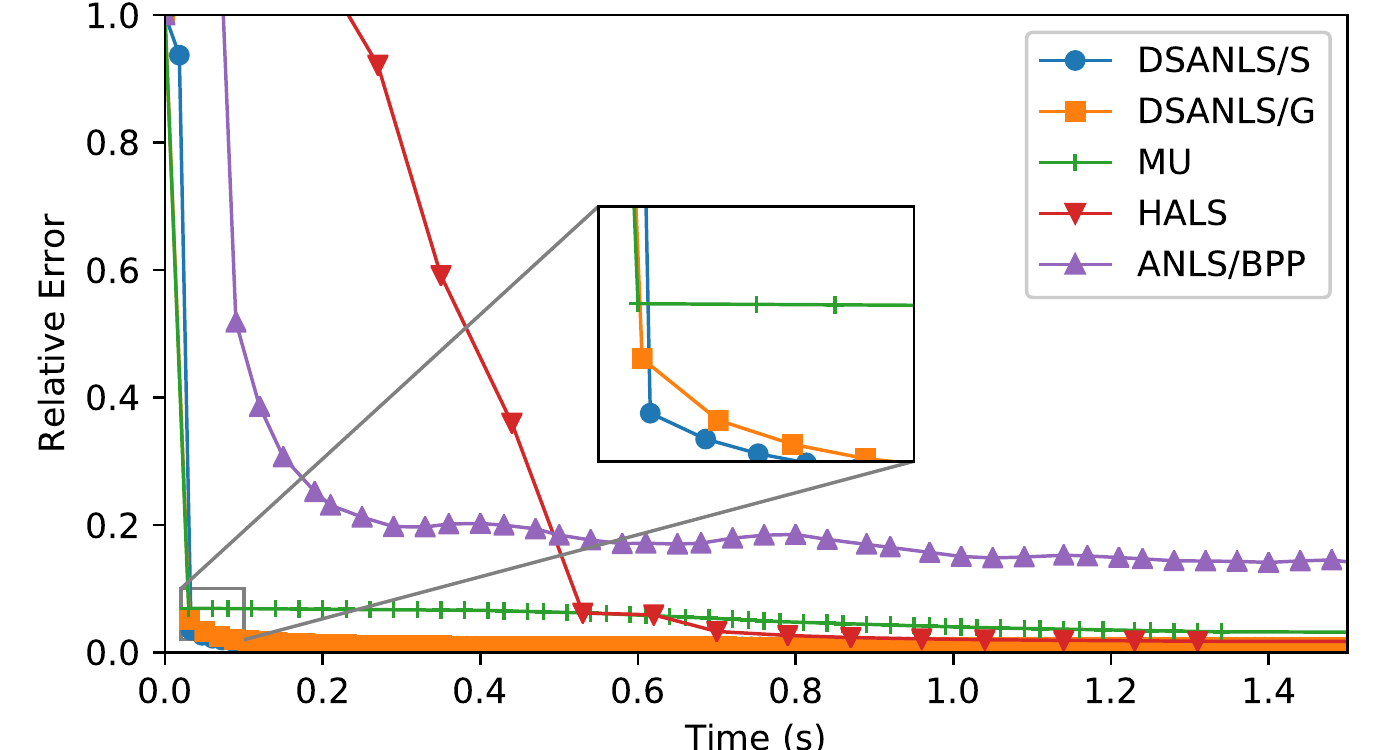}}
  \hspace{0.07in}
   \subfigure[MNIST]{
    \label{fig:dsanl:mnist}
    \includegraphics[width=0.3\linewidth]{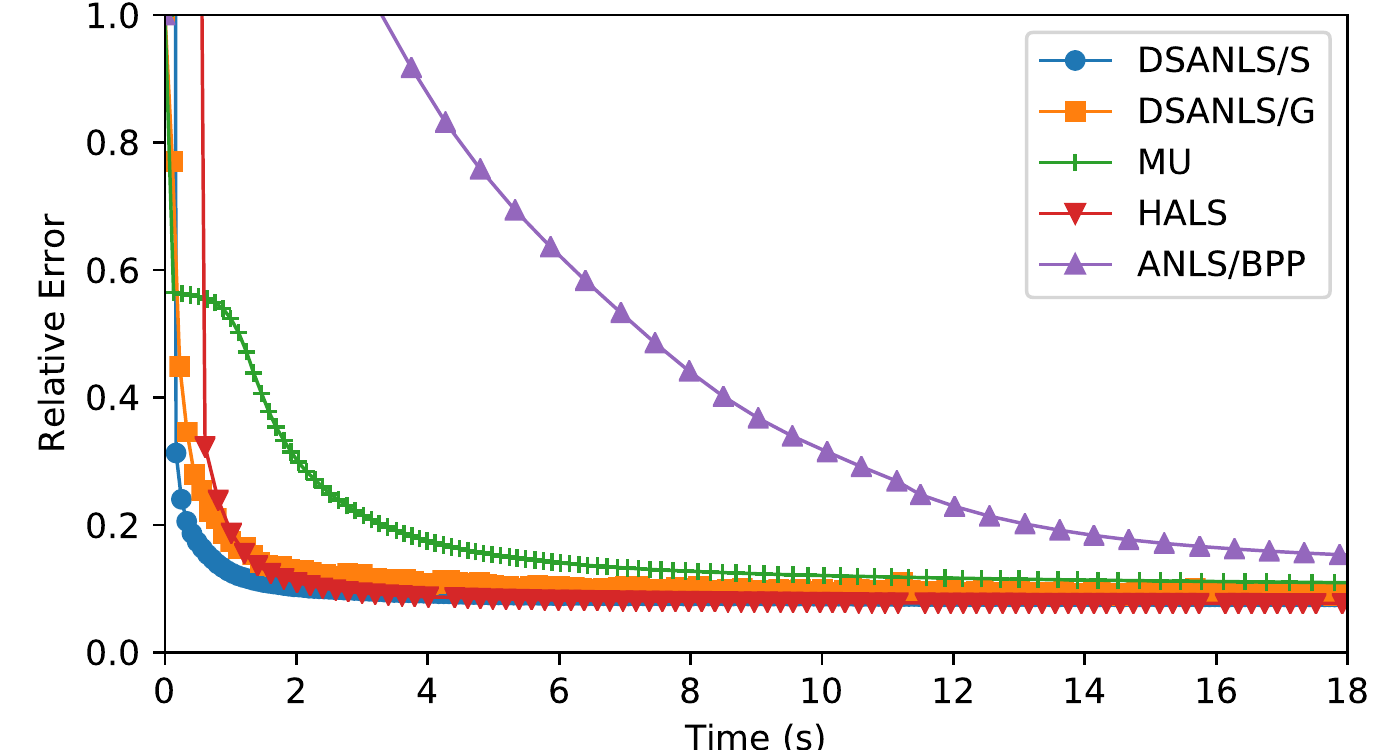}}
    \\
  \subfigure[GISETTE]{
    \label{fig:dsanl:number}
    \includegraphics[width=0.3\linewidth]{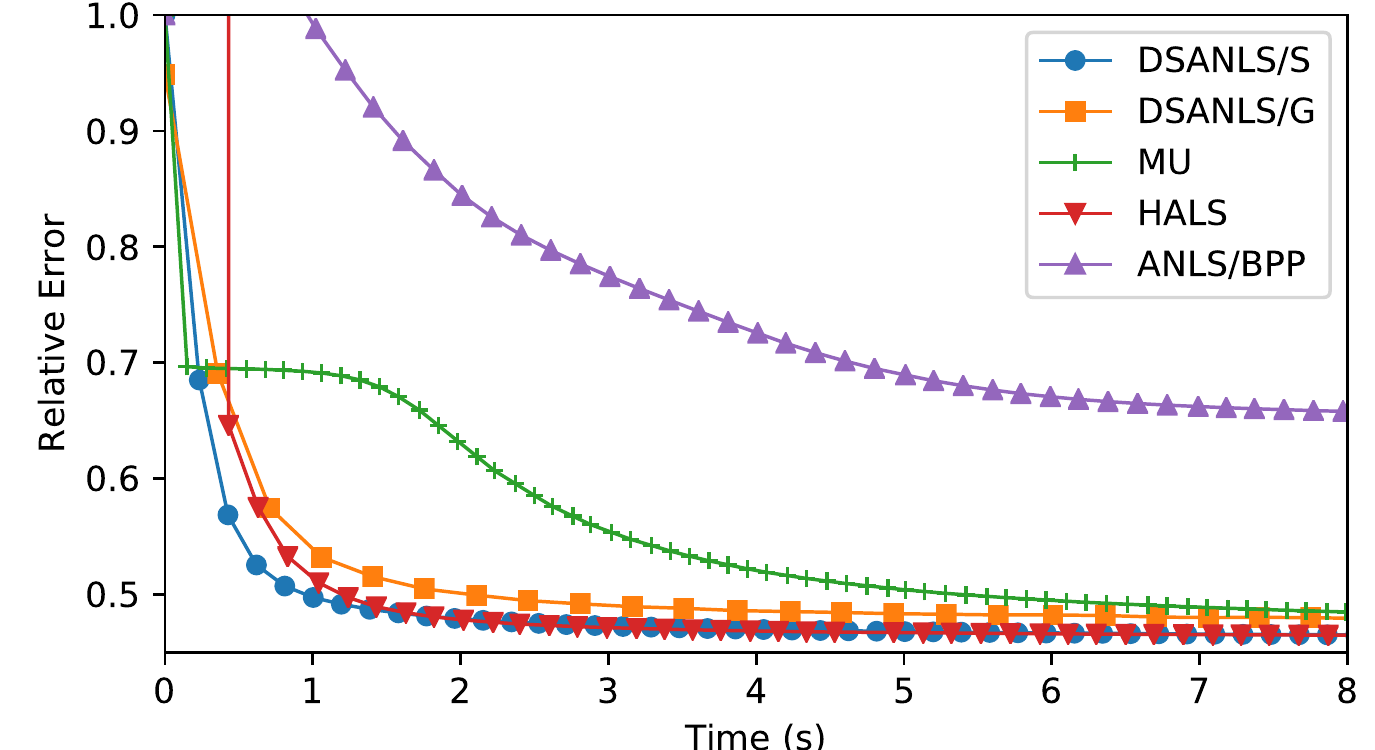}}
   \hspace{0.07in}
   \subfigure[RCV1]{
    \label{fig:dsanl:rcv1}
    \includegraphics[width=0.3\linewidth]{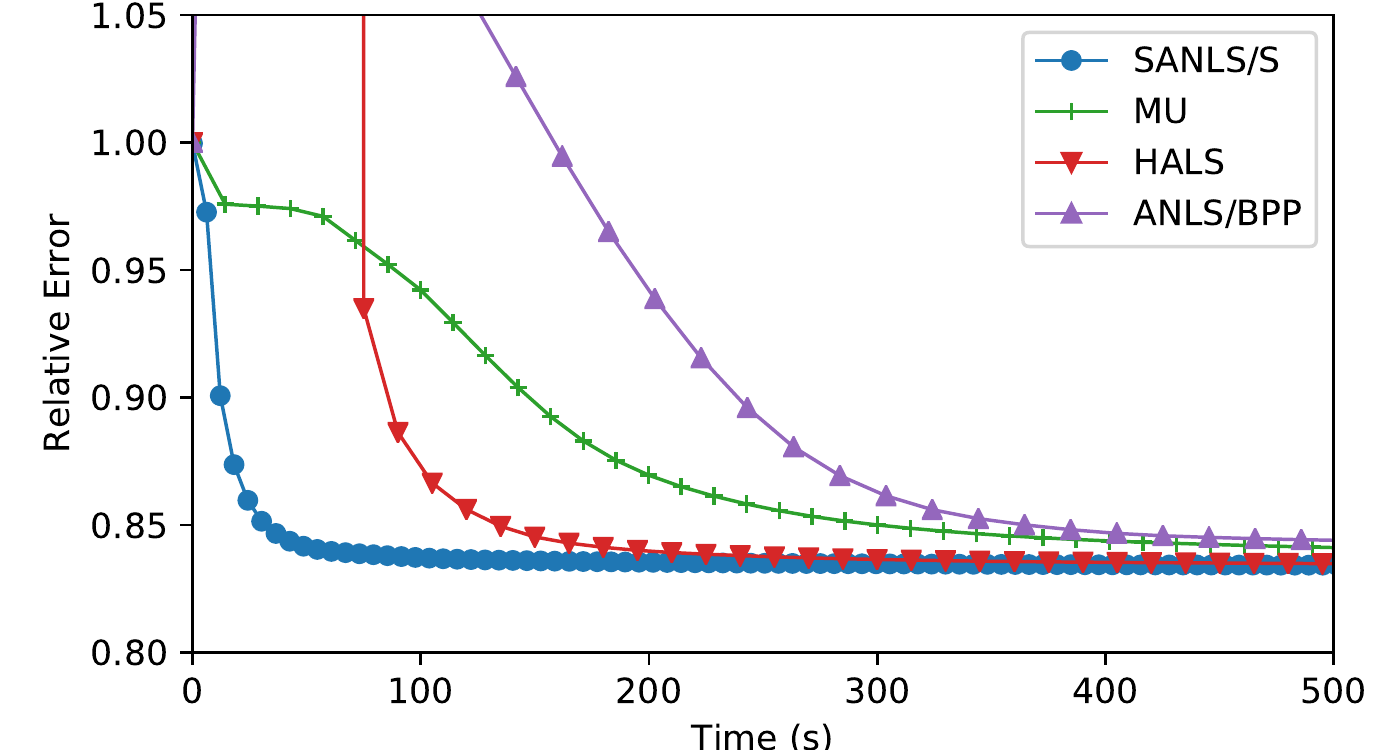}}
  \hspace{0.07in}
  \subfigure[DBLP]{
    \label{fig:dsanl:dblp}
    \includegraphics[width=0.3\linewidth]{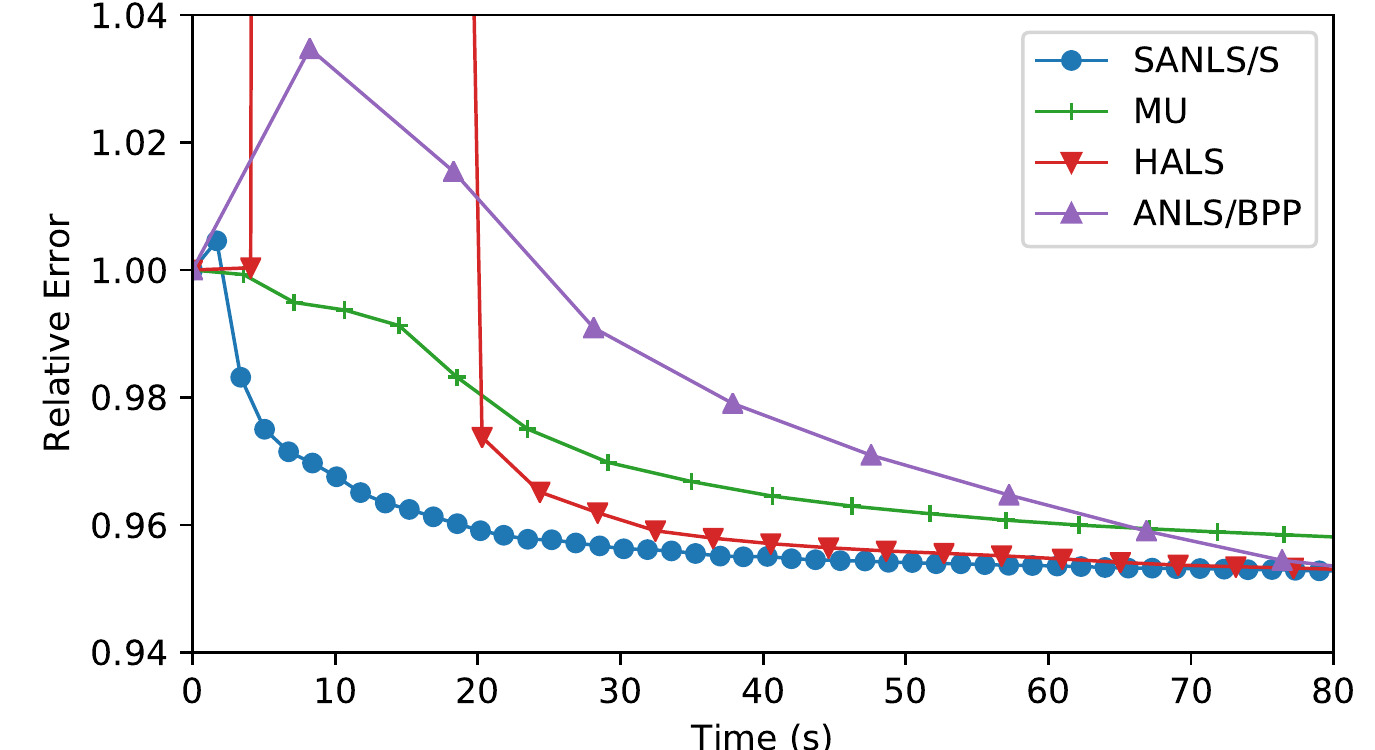}}
  \caption{Relative error over time for general distributed NMF}
  \label{fig:dsanl:error_time}
\end{figure*}

\begin{figure*}[!ht]
  \centering
  \subfigure[FACE]{
    \label{fig:dsanl:scaleface}
    \includegraphics[width=0.22\linewidth]{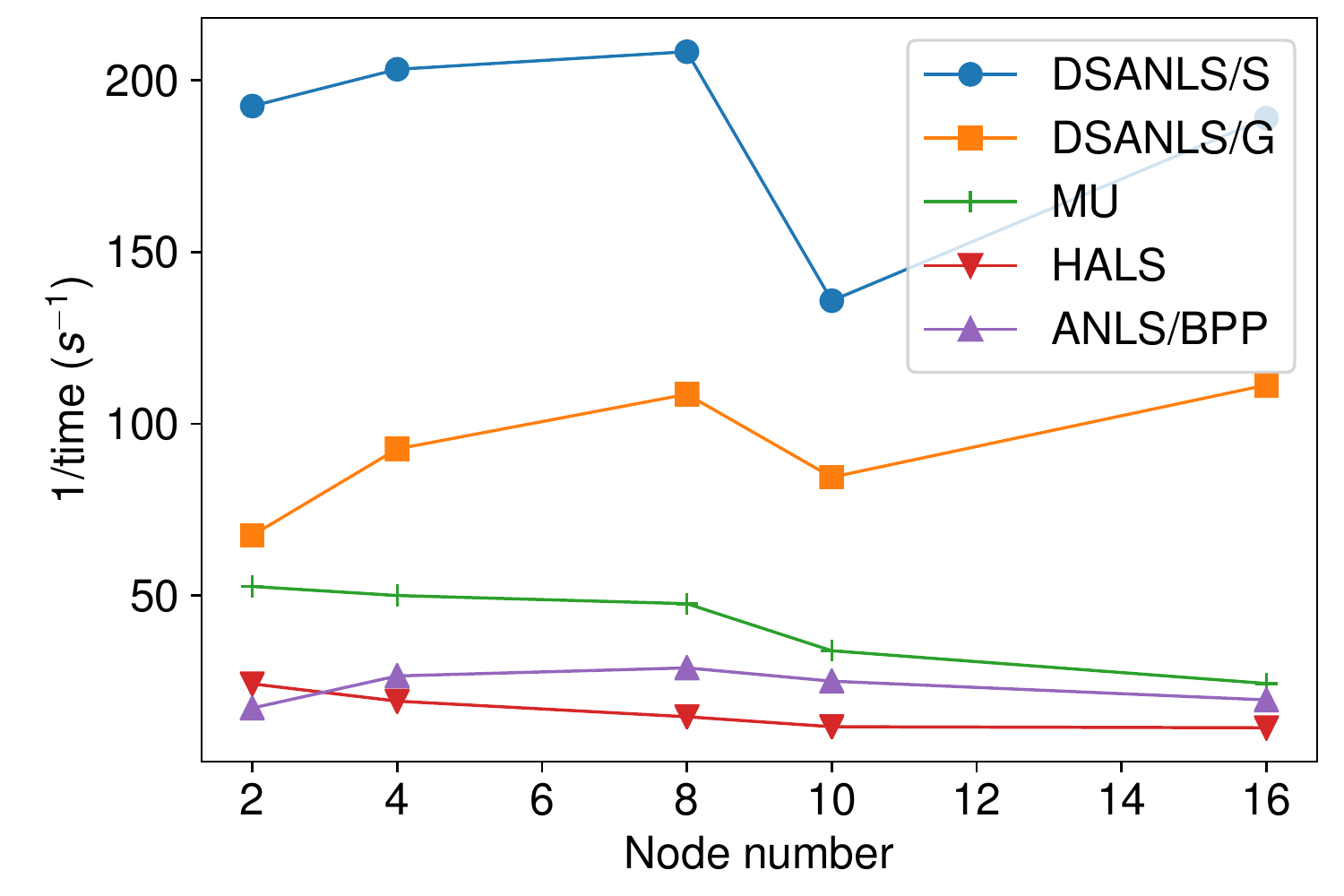}}
    \hspace{0.03in}
   \subfigure[MNIST]{
    \label{fig:dsanl:scalemnist}
    \includegraphics[width=0.22\linewidth]{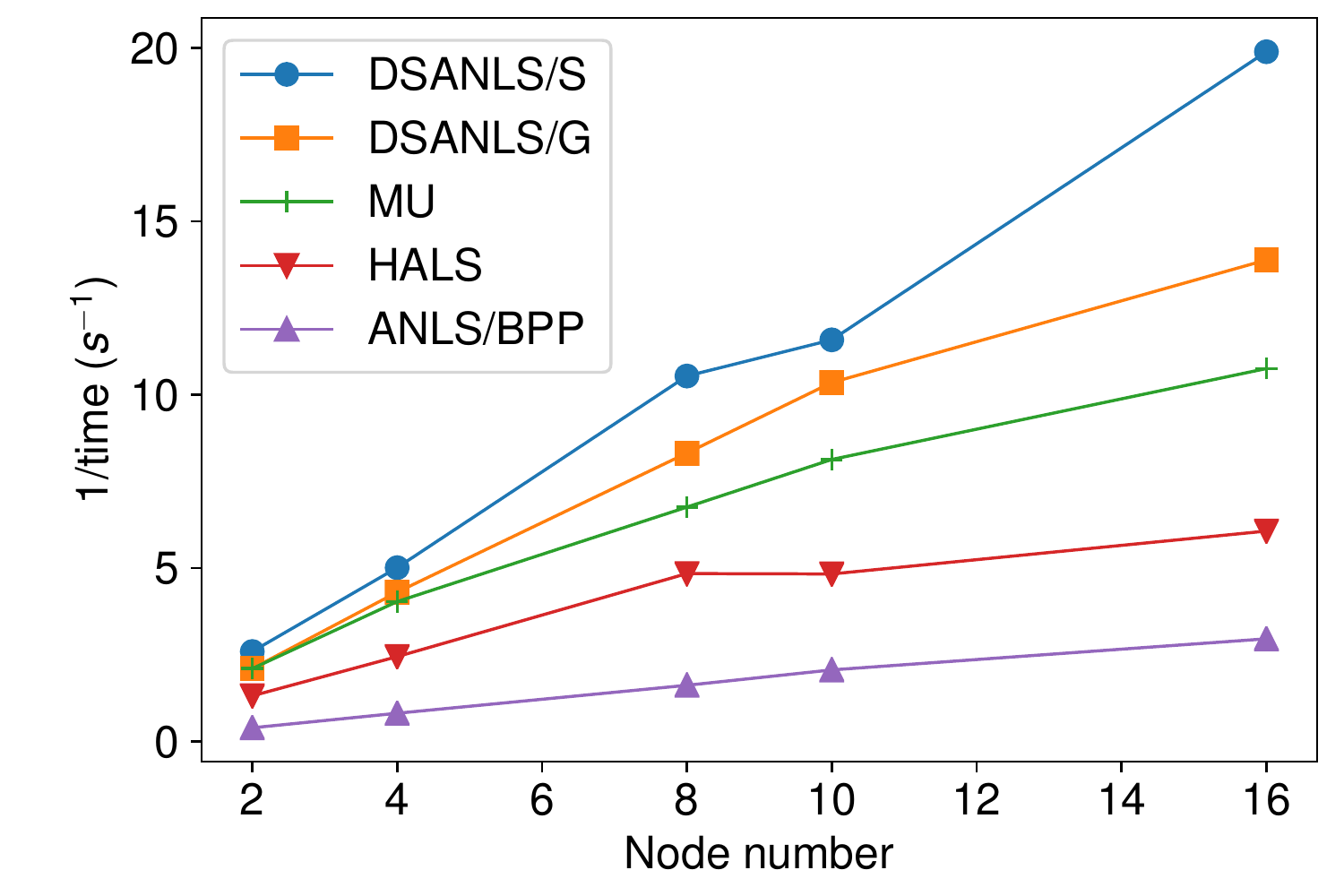}}
    \hspace{0.03in}
   \subfigure[RCV1]{
    \label{fig:dsanl:scalercv1}
    \includegraphics[width=0.22\linewidth]{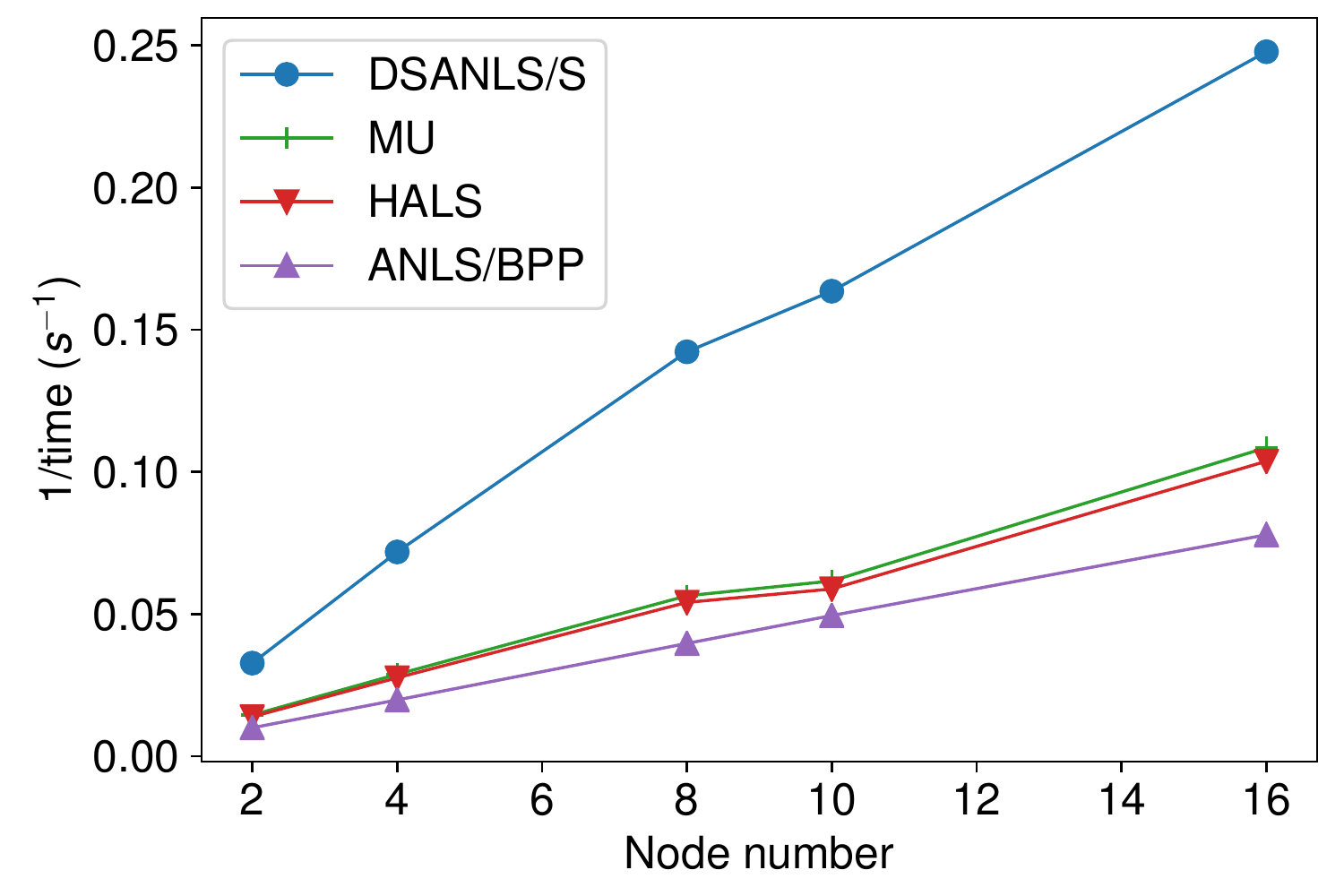}}
    \hspace{0.03in}
  \subfigure[DBLP]{
    \label{fig:dsanl:scaledblp}
    \includegraphics[width=0.22\linewidth]{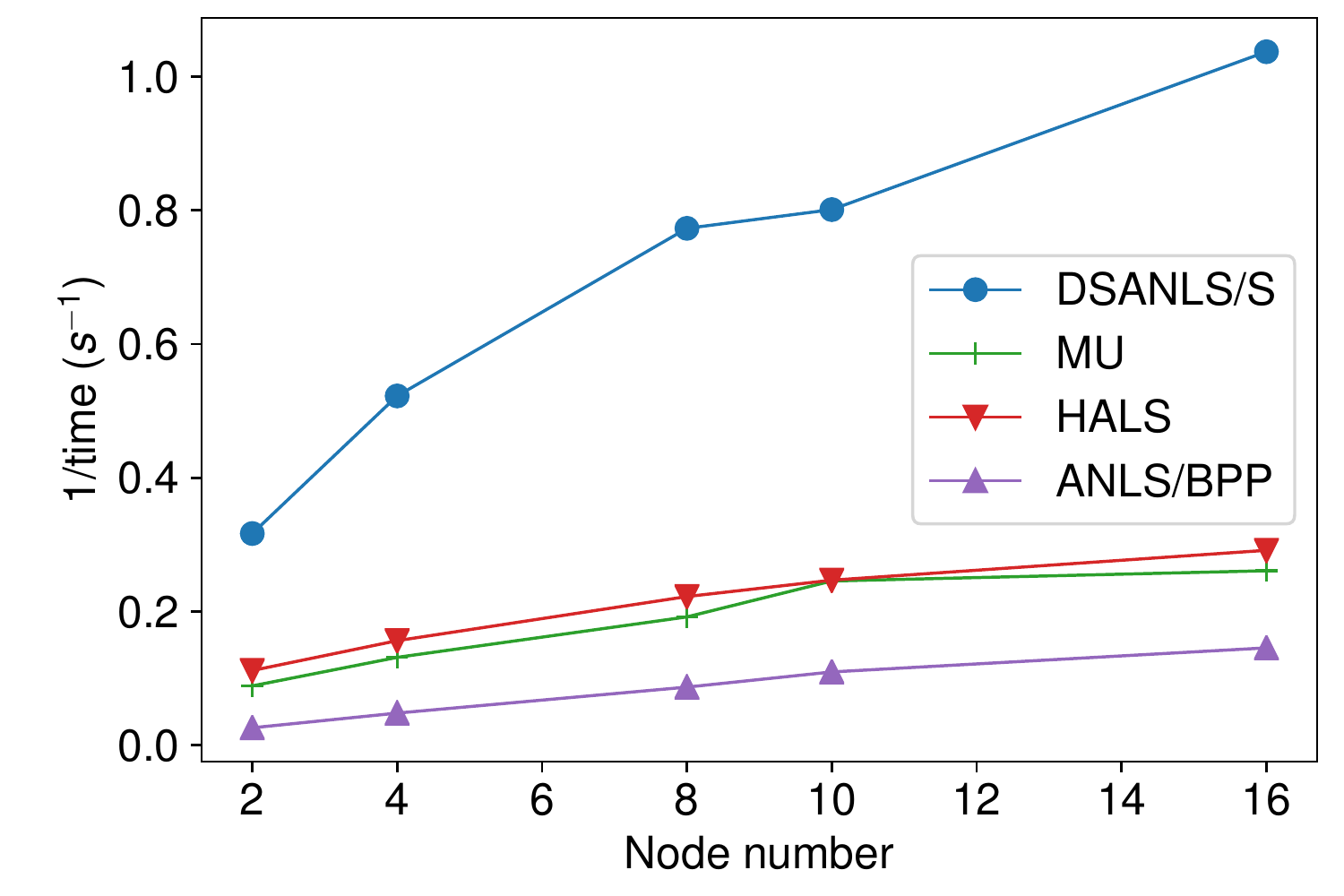}}
  \caption{Reciprocal of per-iteration time as a function of cluster size for general distributed NMF}
  \label{fig:dsanl:scale}
\end{figure*}

\begin{figure*}[!ht]
  \centering
  \subfigure[$k$=20]{
    \label{fig:dsanl:k20}
    \includegraphics[width=0.22\linewidth]{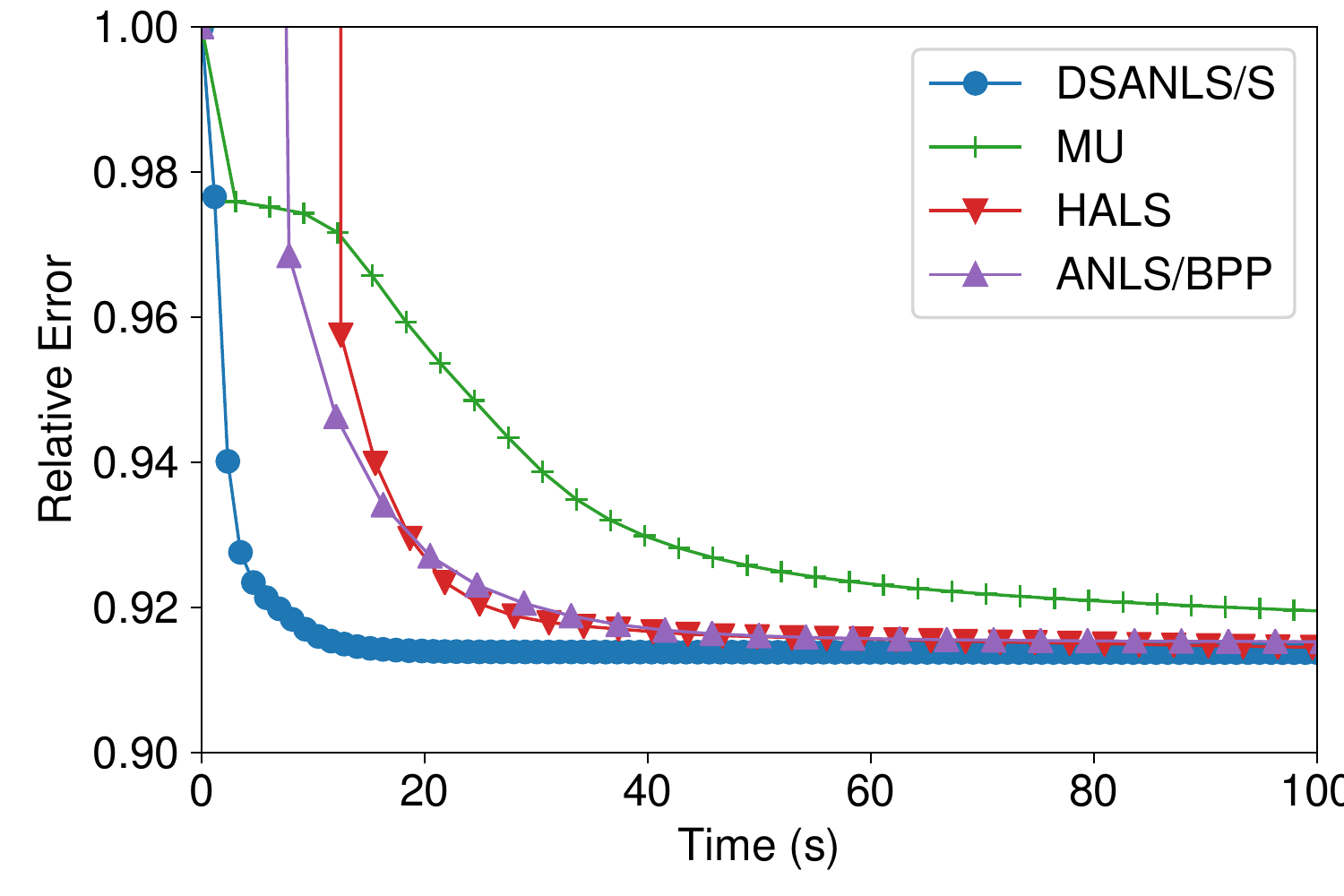}}
  \hspace{0.03in}
  \subfigure[$k$=50]{
    \label{fig:dsanl:k50}
    \includegraphics[width=0.22\linewidth]{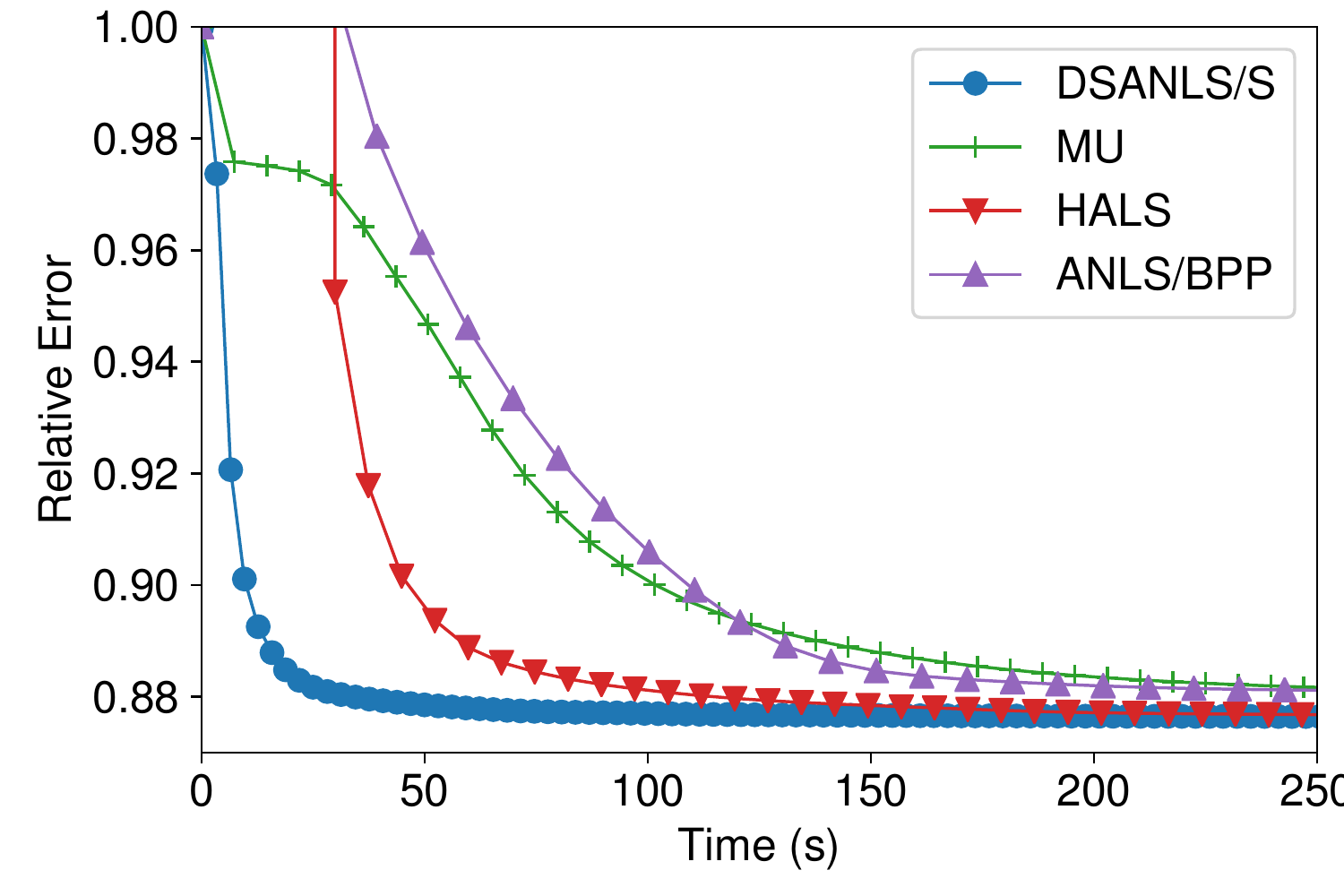}}
  \hspace{0.03in}
   \subfigure[$k$=200]{
    \label{fig:dsanl:k200}
    \includegraphics[width=0.22\linewidth]{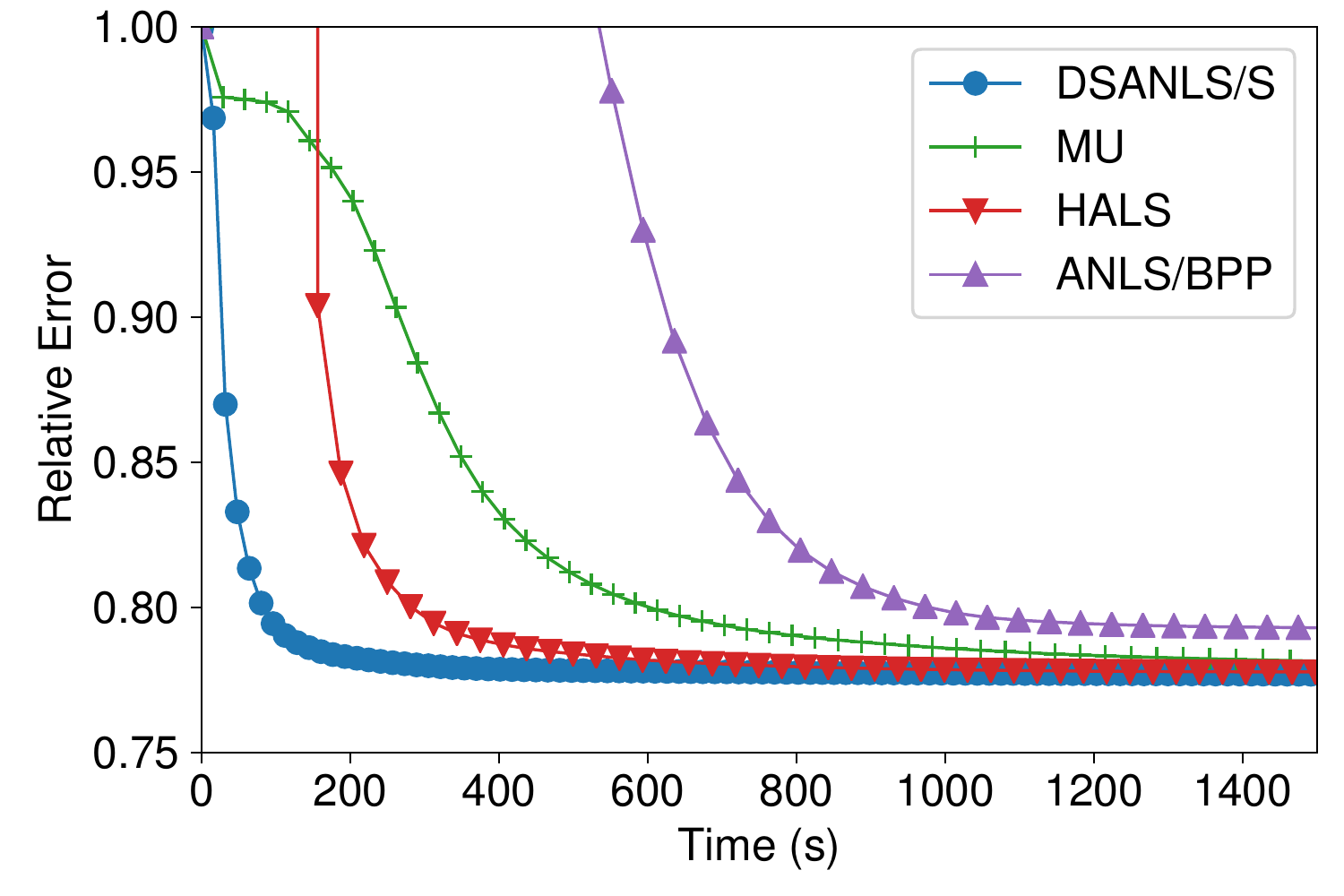}}
  \hspace{0.03in}
  \subfigure[$k$=500]{
    \label{fig:dsanl:k500}
    \includegraphics[width=0.22\linewidth]{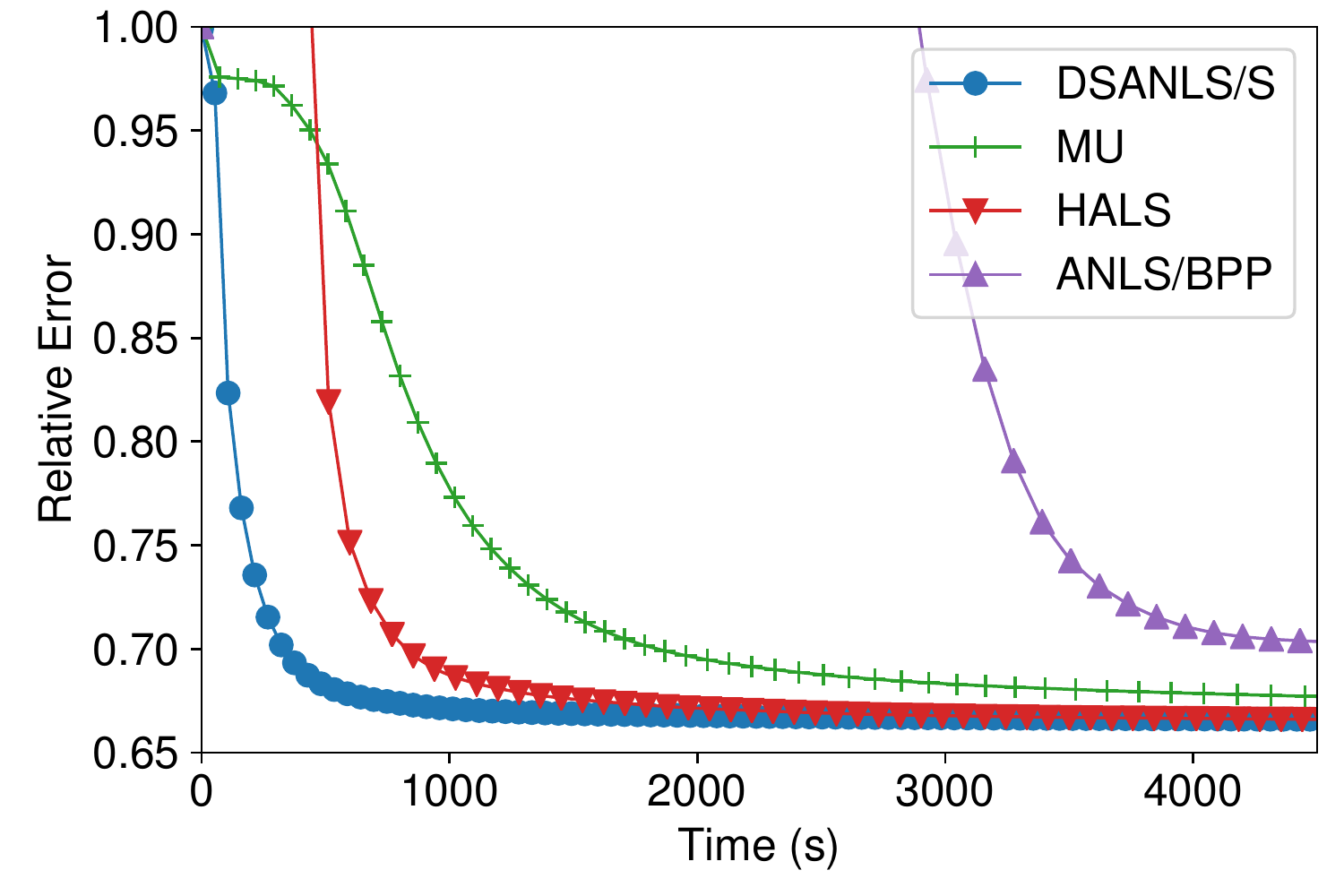}}
  \caption{Relative error over time for general distributed NMF, varying $k$ value}
  \label{fig:dsanl:varyk}
\end{figure*}

\begin{figure*}[!ht]
  \centering
  \subfigure[BOATS]{
    \label{fig:dsanl:gdboats}
    \includegraphics[width=0.22\linewidth]{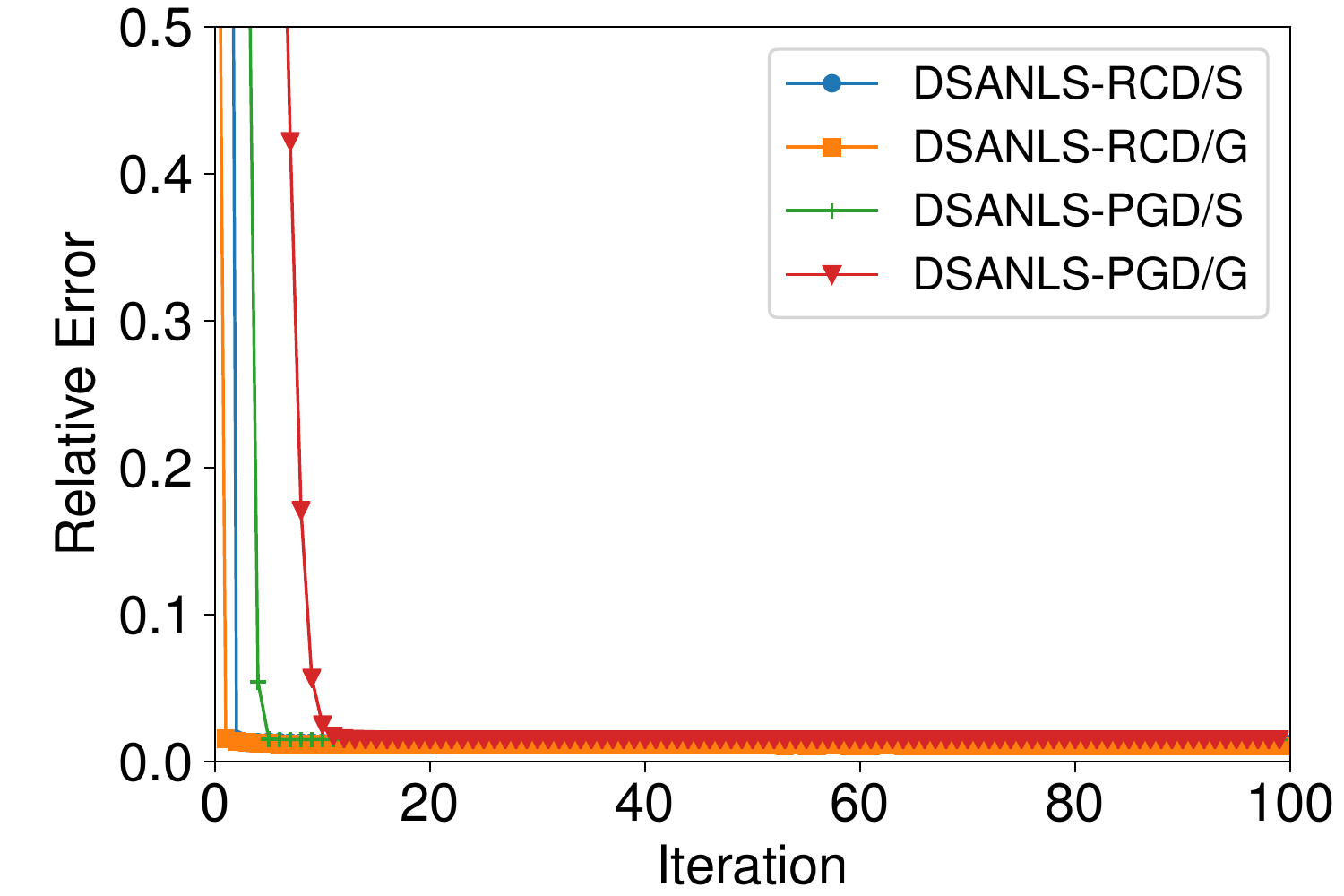}}
  \hspace{0.03in}
  \subfigure[FACE]{
    \label{fig:dsanl:gdface}
    \includegraphics[width=0.22\linewidth]{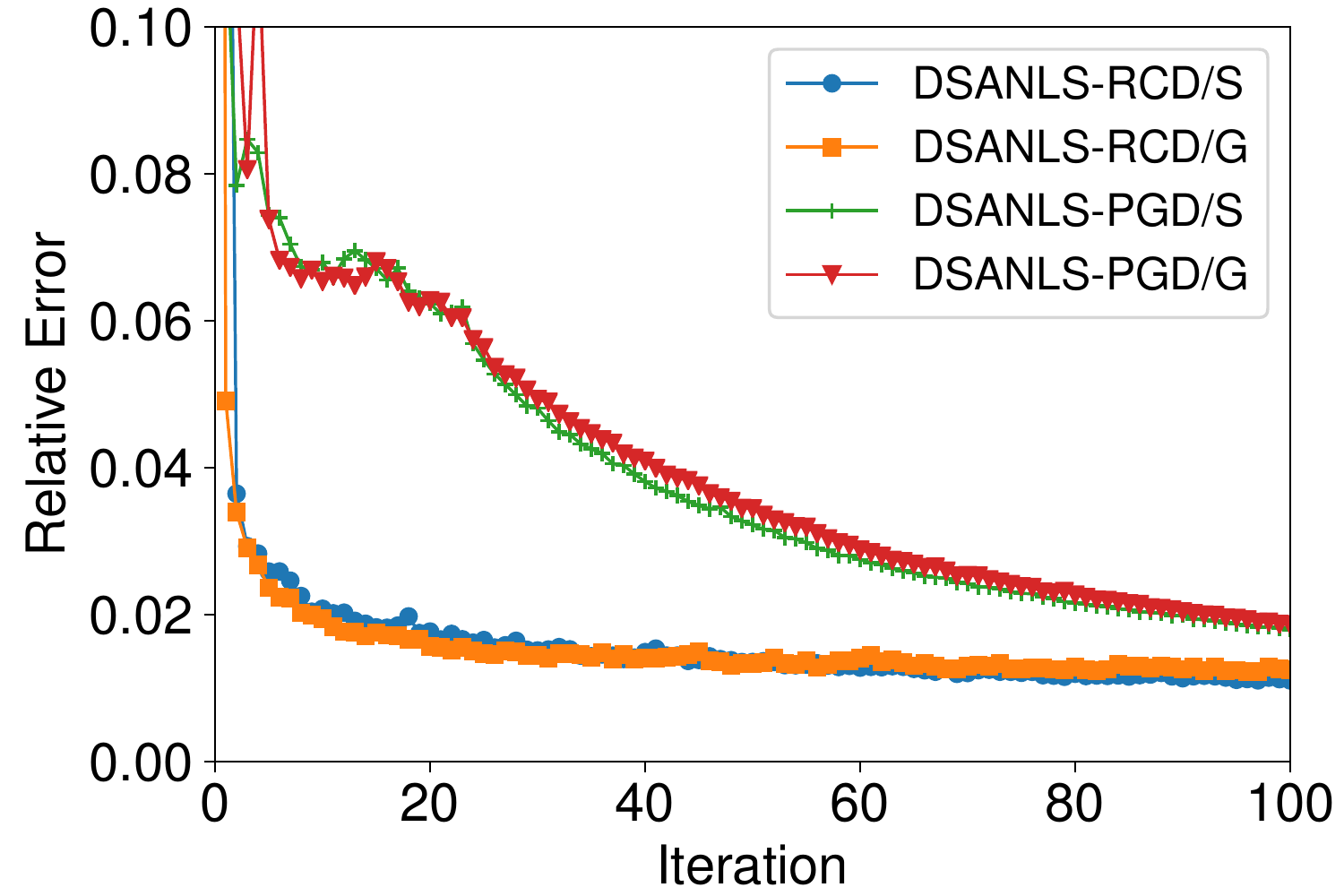}}
  \hspace{0.03in}
  \subfigure[GISETTE]{
    \label{fig:dsanl:gdnumber}
    \includegraphics[width=0.22\linewidth]{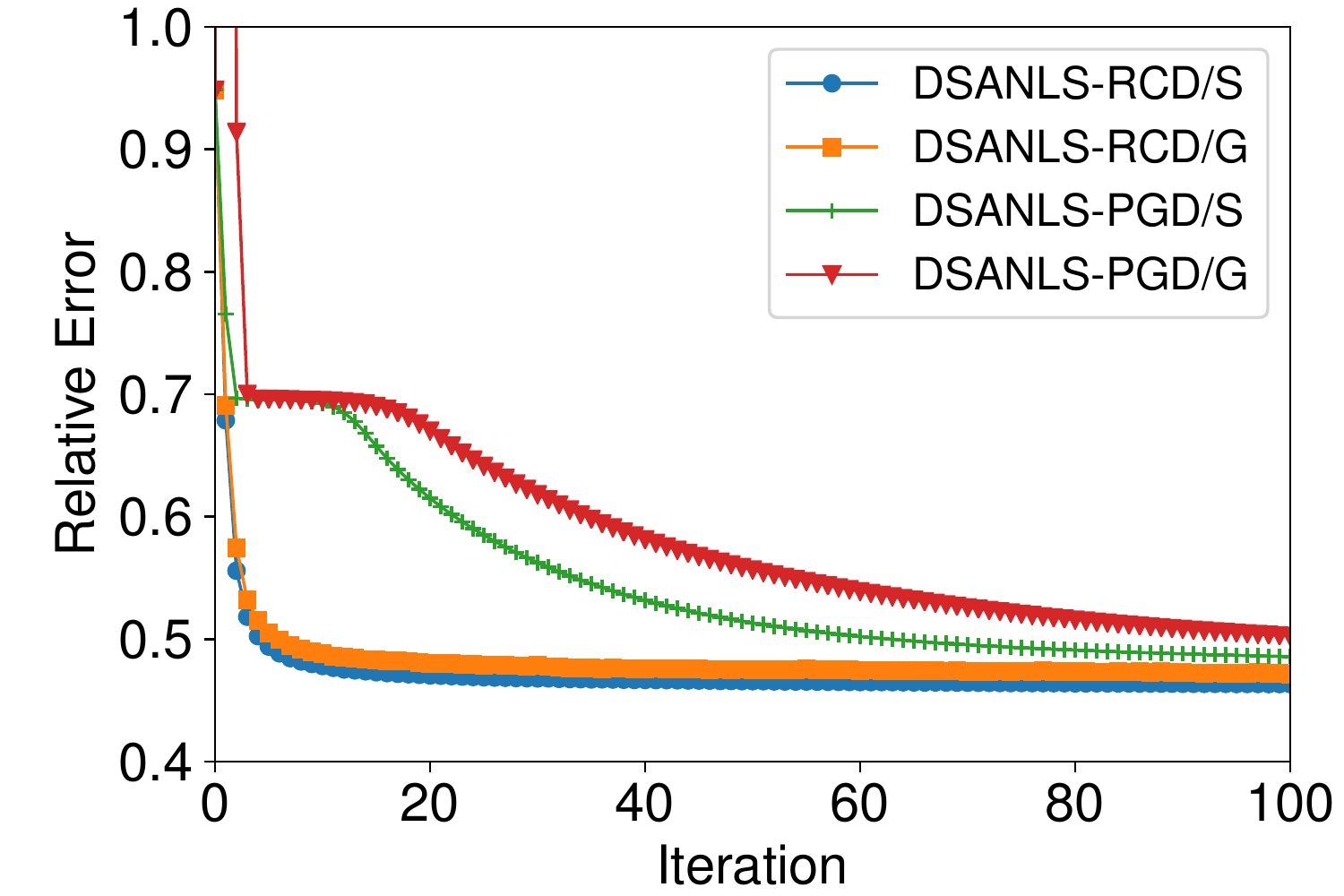}}
   \hspace{0.03in}
   \subfigure[RCV1]{
    \label{fig:dsanl:gdrcv1}
    \includegraphics[width=0.22\linewidth]{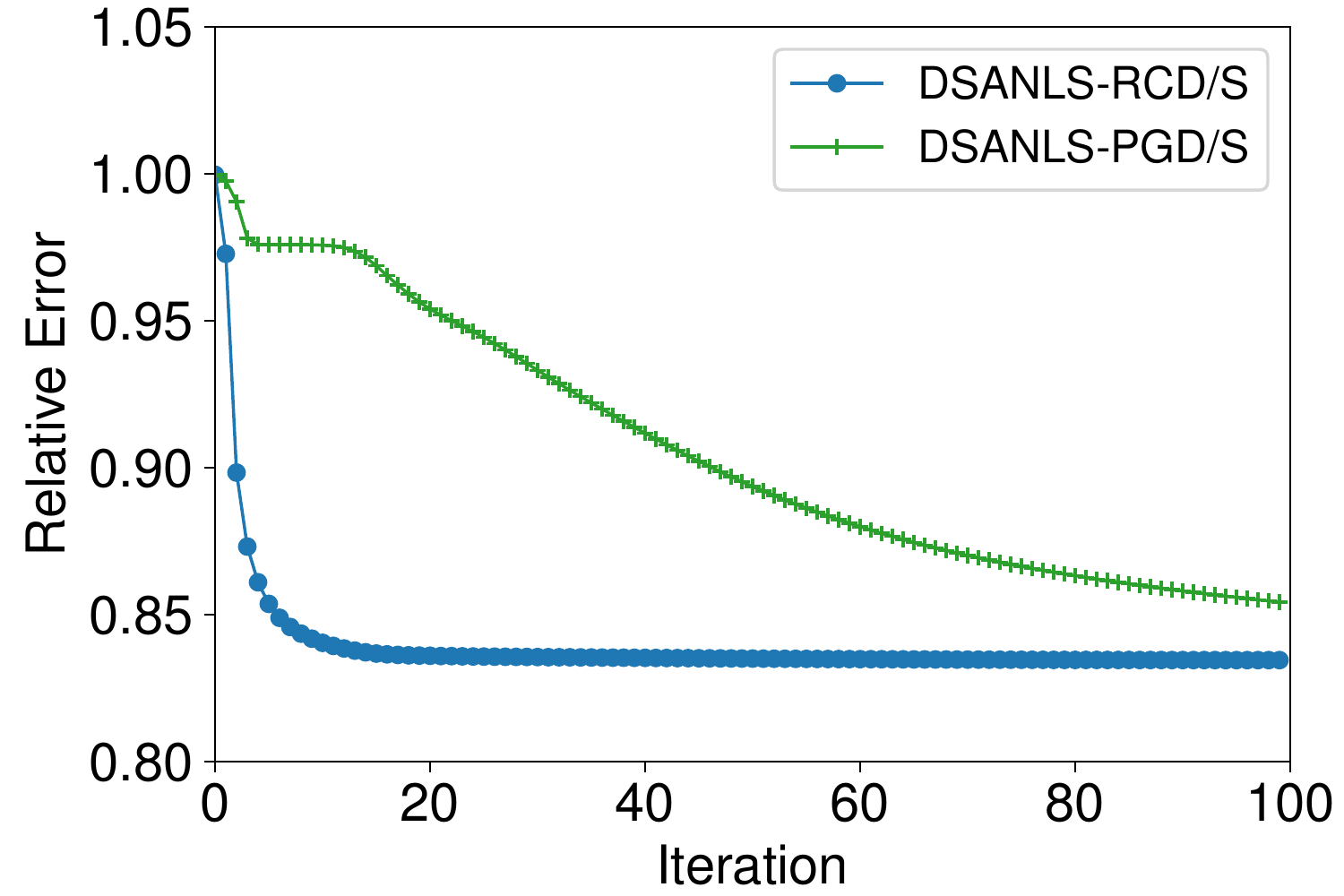}}
  \caption{Relative error per-iteration of different subproblem solvers for general distributed NMF}
  \label{fig:dsanl:sub}
\end{figure*}

\begin{figure*}[!ht]
  \centering
  \subfigure[BOATS]{
    \label{fig:secure:boats}
    \includegraphics[width=0.22\linewidth]{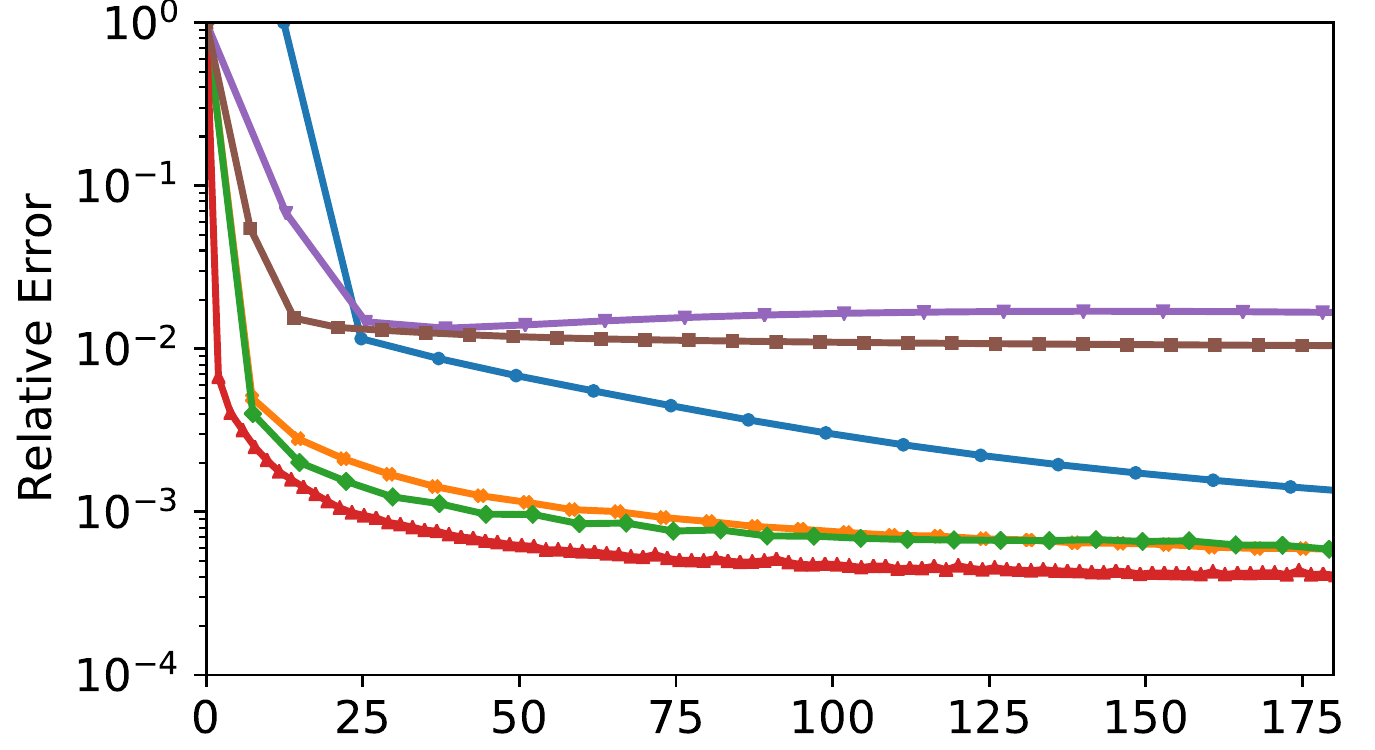}}
  \hspace{0.03in}
  \subfigure[FACE]{
    \label{fig:secure:face}
    \includegraphics[width=0.22\linewidth]{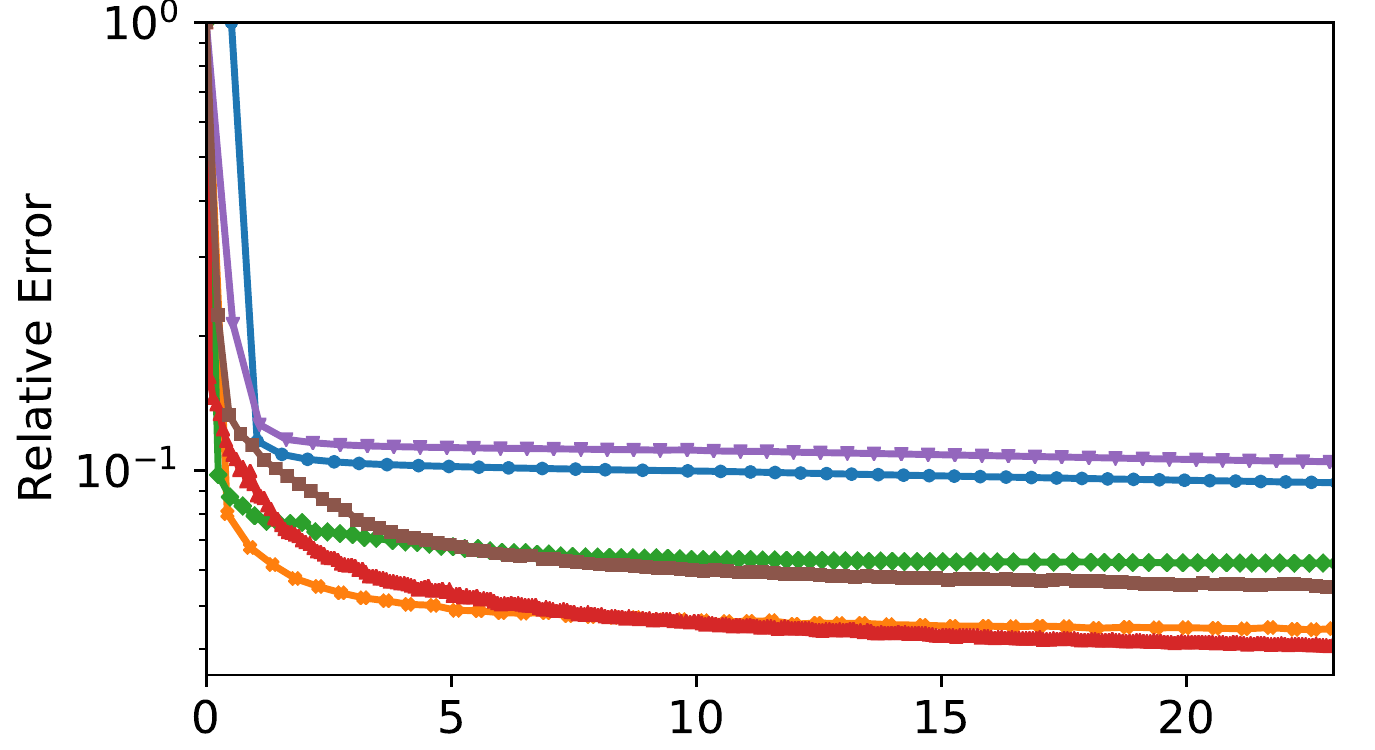}}
  \hspace{0.03in}
  \subfigure[MNIST]{
    \label{fig:secure:mnist}
    \includegraphics[width=0.22\linewidth]{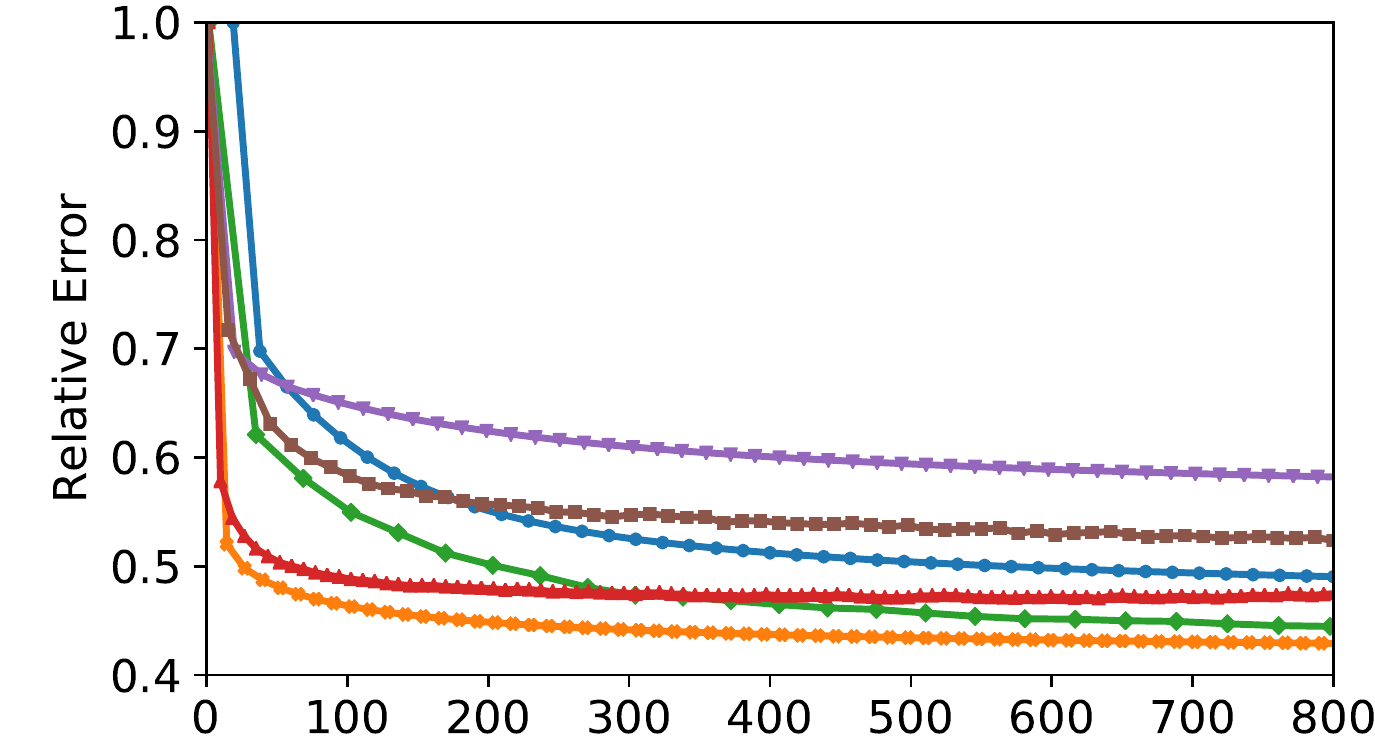}}
  \hspace{0.03in}
  \subfigure[GISETTE]{
    \label{fig:secure:number}
    \includegraphics[width=0.22\linewidth]{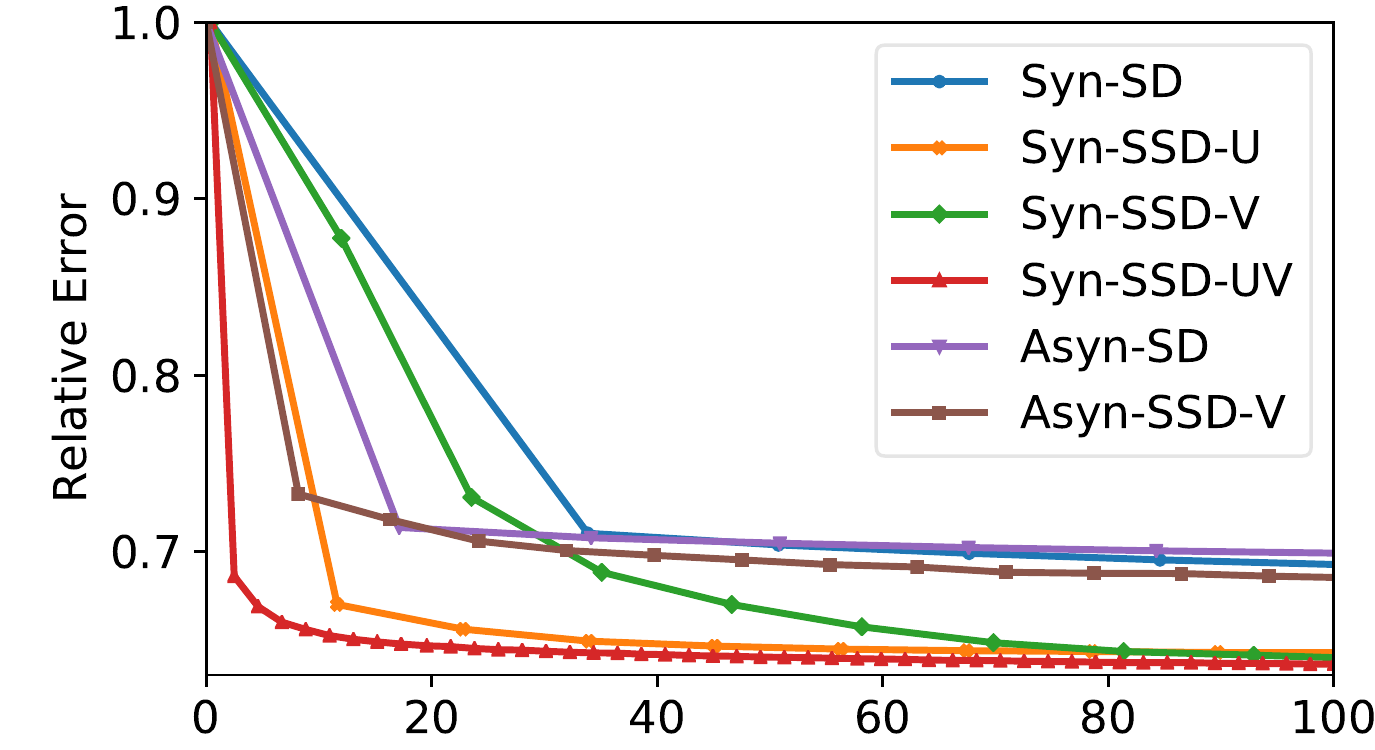}}
  \caption{Relative error over time for uniform workload in secure distributed NMF}
  \label{fig:secure:error_time}
\end{figure*}

\begin{figure*}[!ht]
  \centering
  \subfigure[BOATS]{
    \label{fig:secure:boats_imba}
    \includegraphics[width=0.22\linewidth]{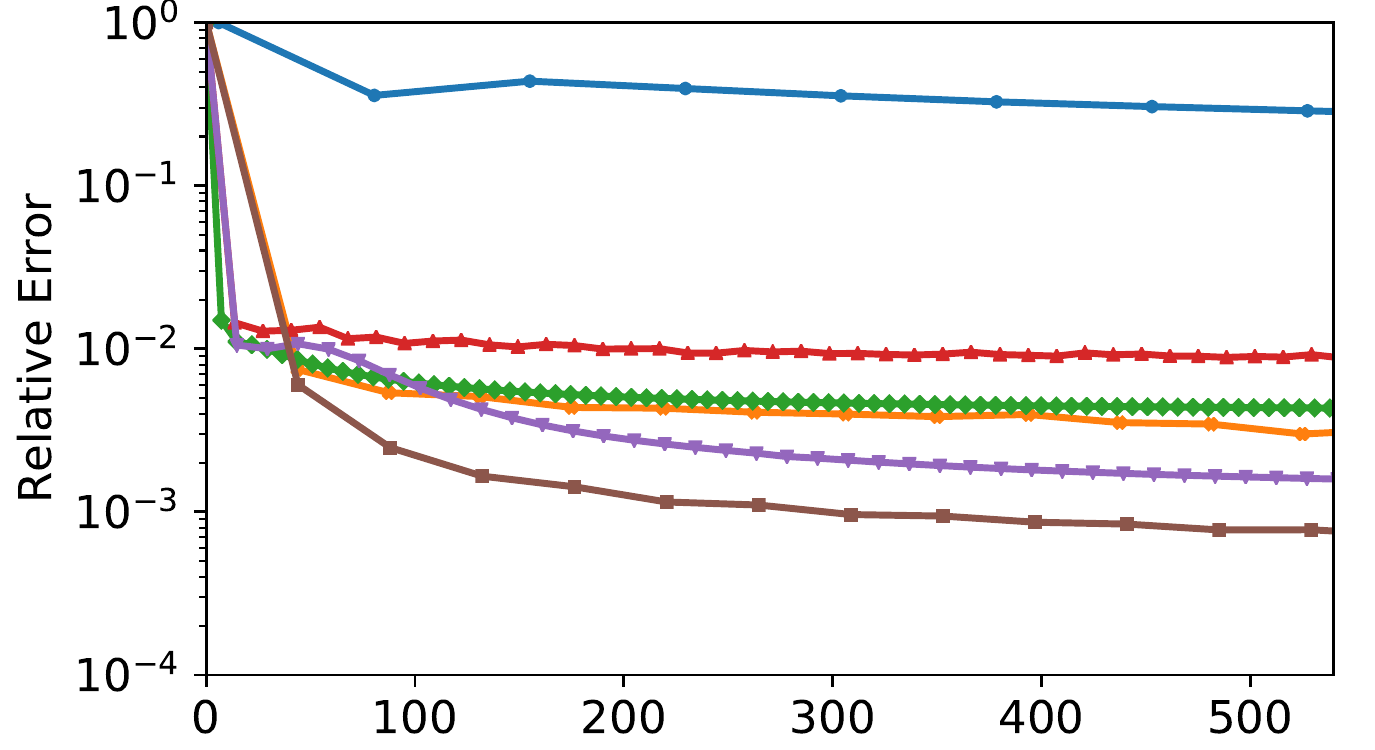}}
  \hspace{0.03in}
  \subfigure[FACE]{
    \label{fig:secure:face_imba}
    \includegraphics[width=0.22\linewidth]{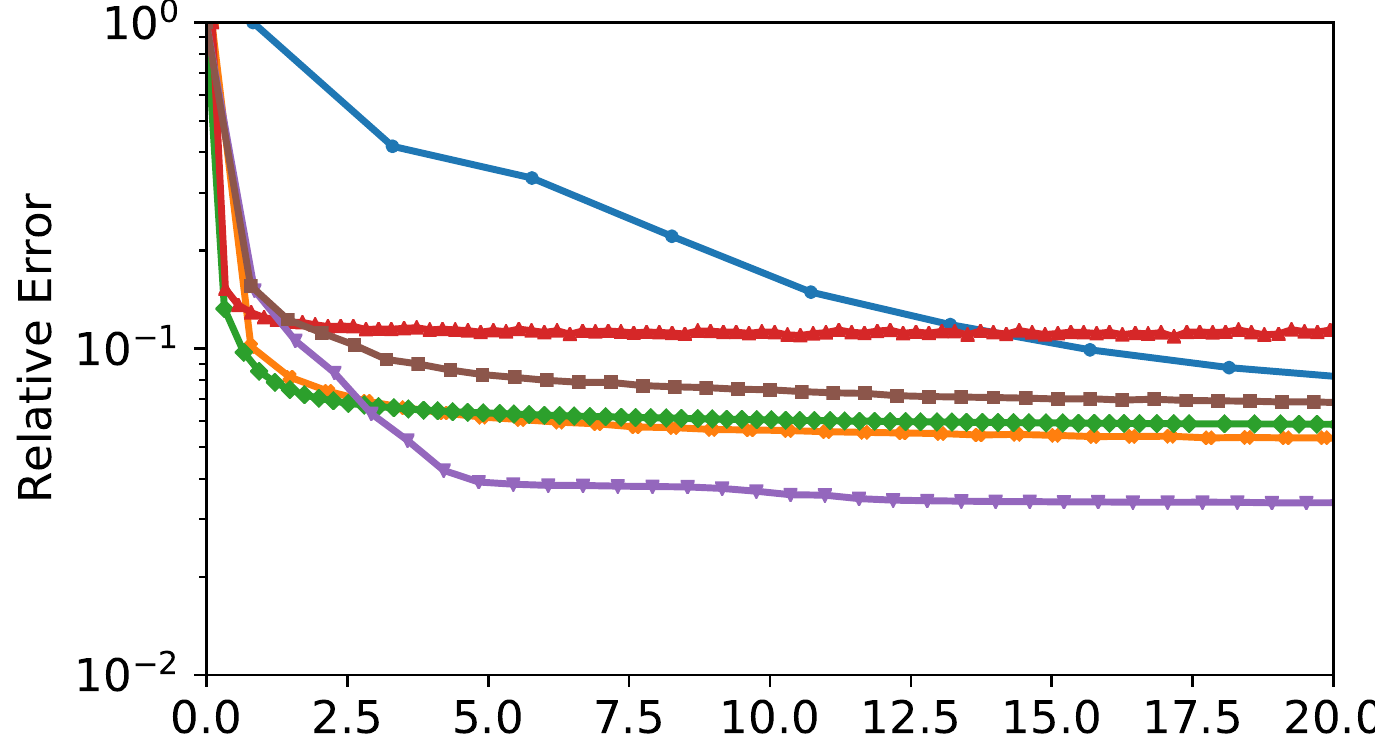}}
  \hspace{0.03in}
  \subfigure[MNIST]{
    \label{fig:secure:mnist_imba}
    \includegraphics[width=0.22\linewidth]{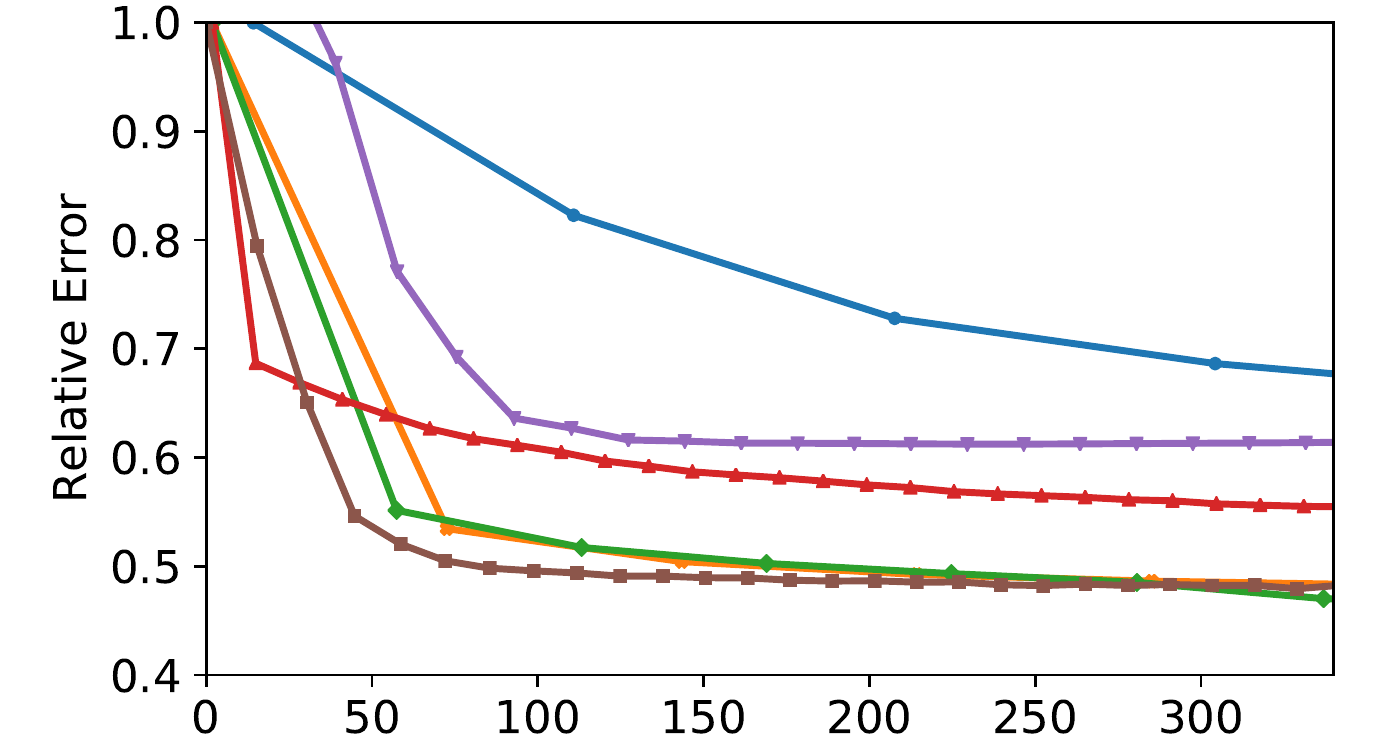}}
  \hspace{0.03in}
  \subfigure[GISETTE]{
    \label{fig:secure:number_imba}
    \includegraphics[width=0.22\linewidth]{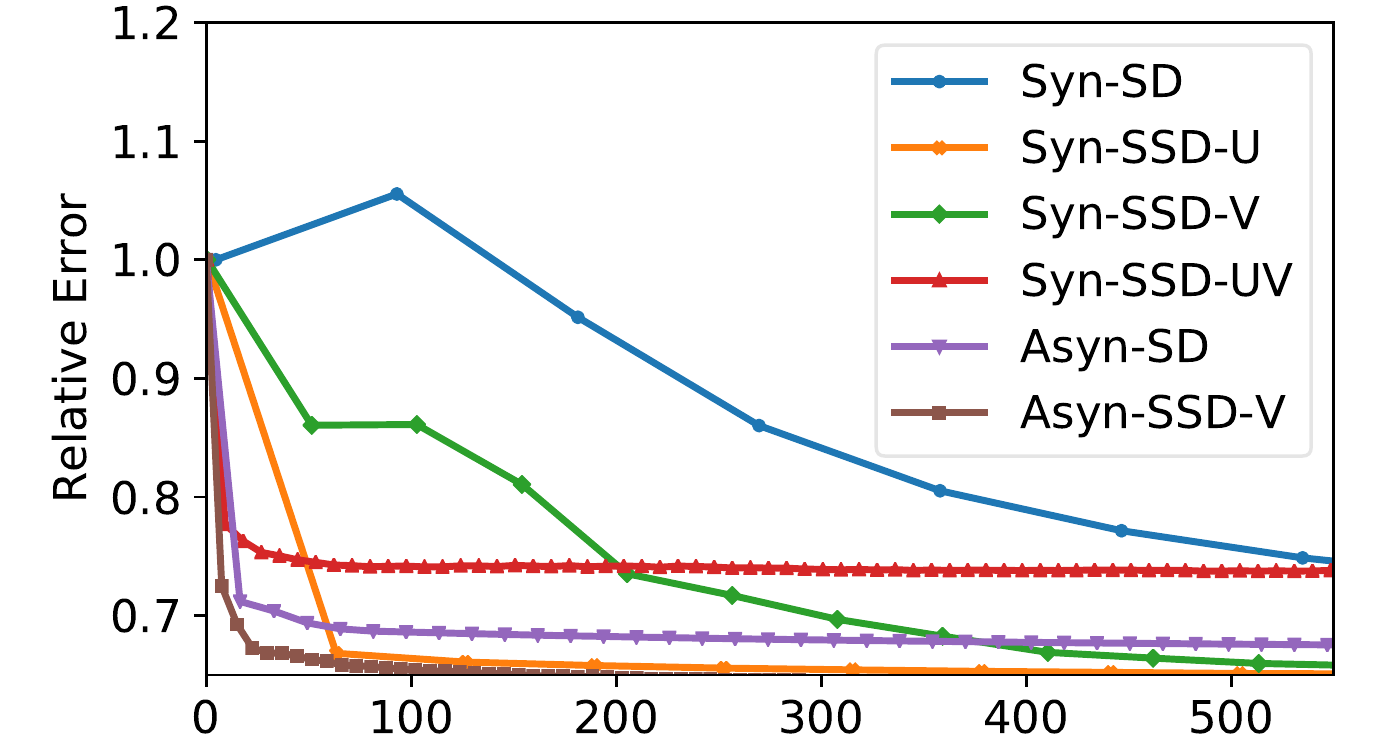}}
  \hspace{0.03in}
  \caption{Relative error over time for imbalanced workload in secure distributed NMF}
  \label{fig:secure:error_time_imba}
\end{figure*}

\begin{figure*}[!ht]
  \centering
  \subfigure[BOATS]{
    \label{fig:secure:scalercv1}
    \includegraphics[width=0.22\linewidth]{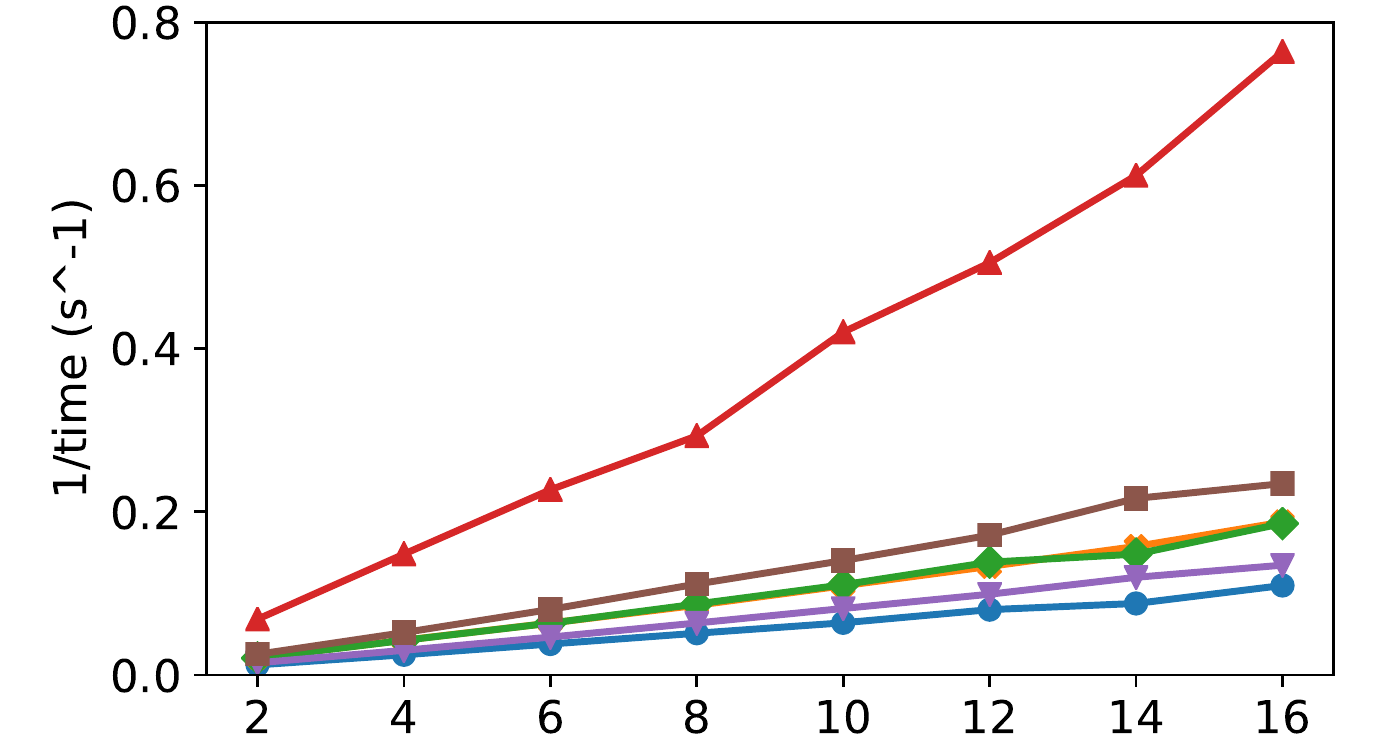}}
    \hspace{0.03in}
  \subfigure[FACE]{
    \label{fig:secure:scaleface}
    \includegraphics[width=0.22\linewidth]{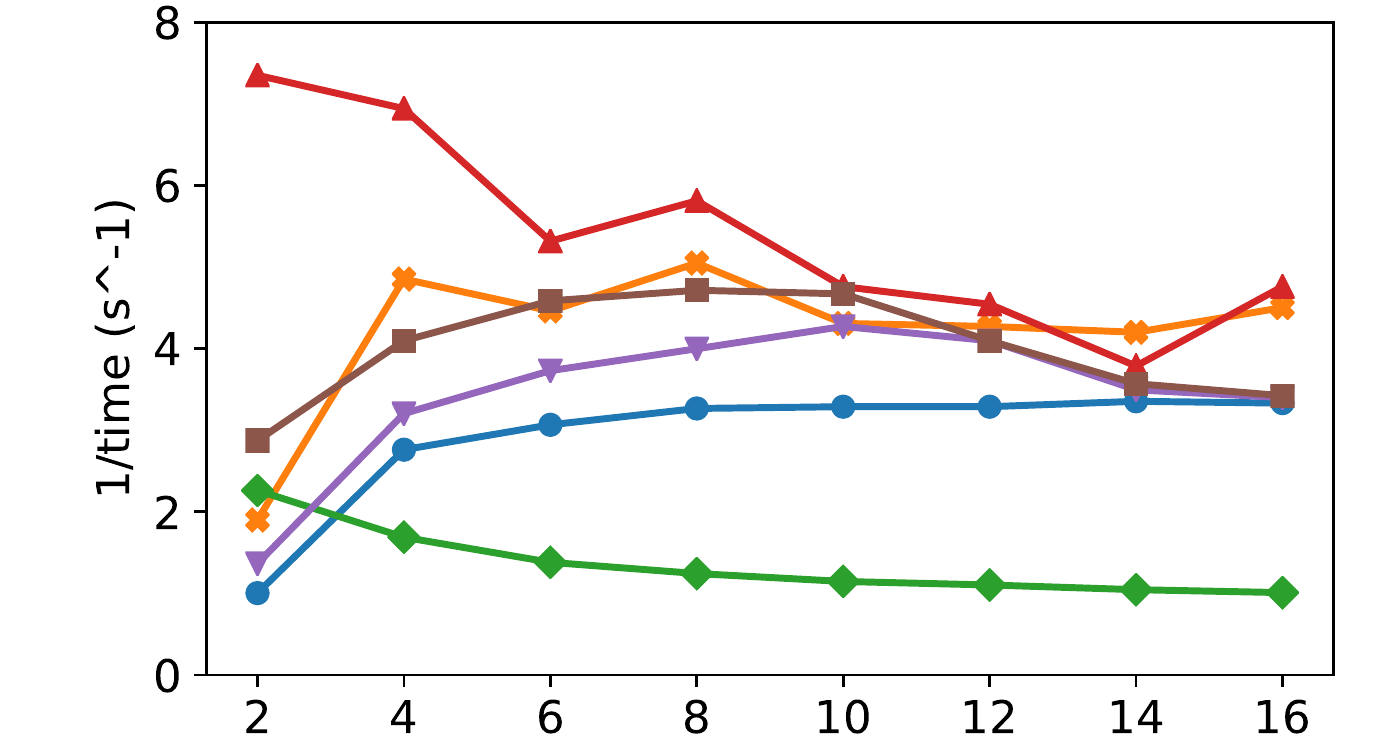}}
    \hspace{0.03in}
  \subfigure[MNIST]{
    \label{fig:secure:scalemnist}
    \includegraphics[width=0.22\linewidth]{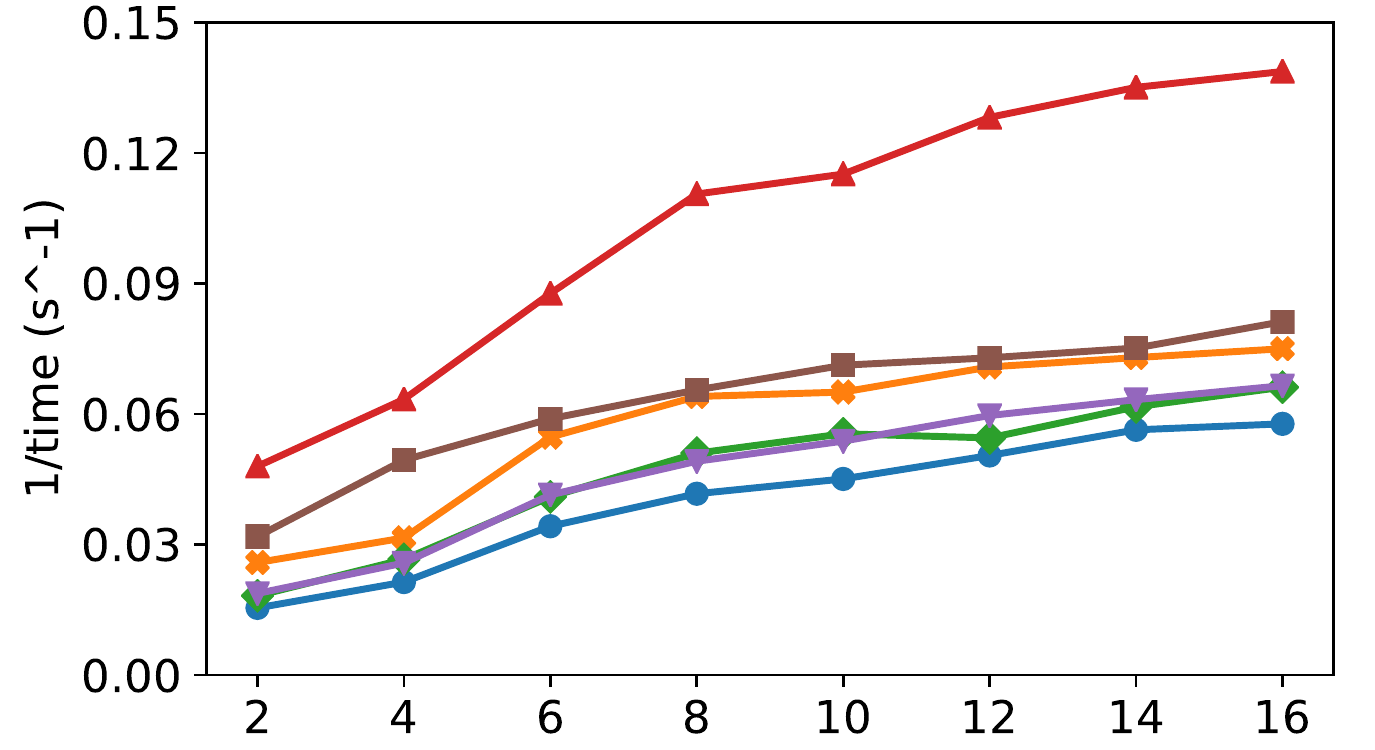}}
    \hspace{0.03in}
  \subfigure[GISETTE]{
    \label{fig:secure:scaledblp}
    \includegraphics[width=0.22\linewidth]{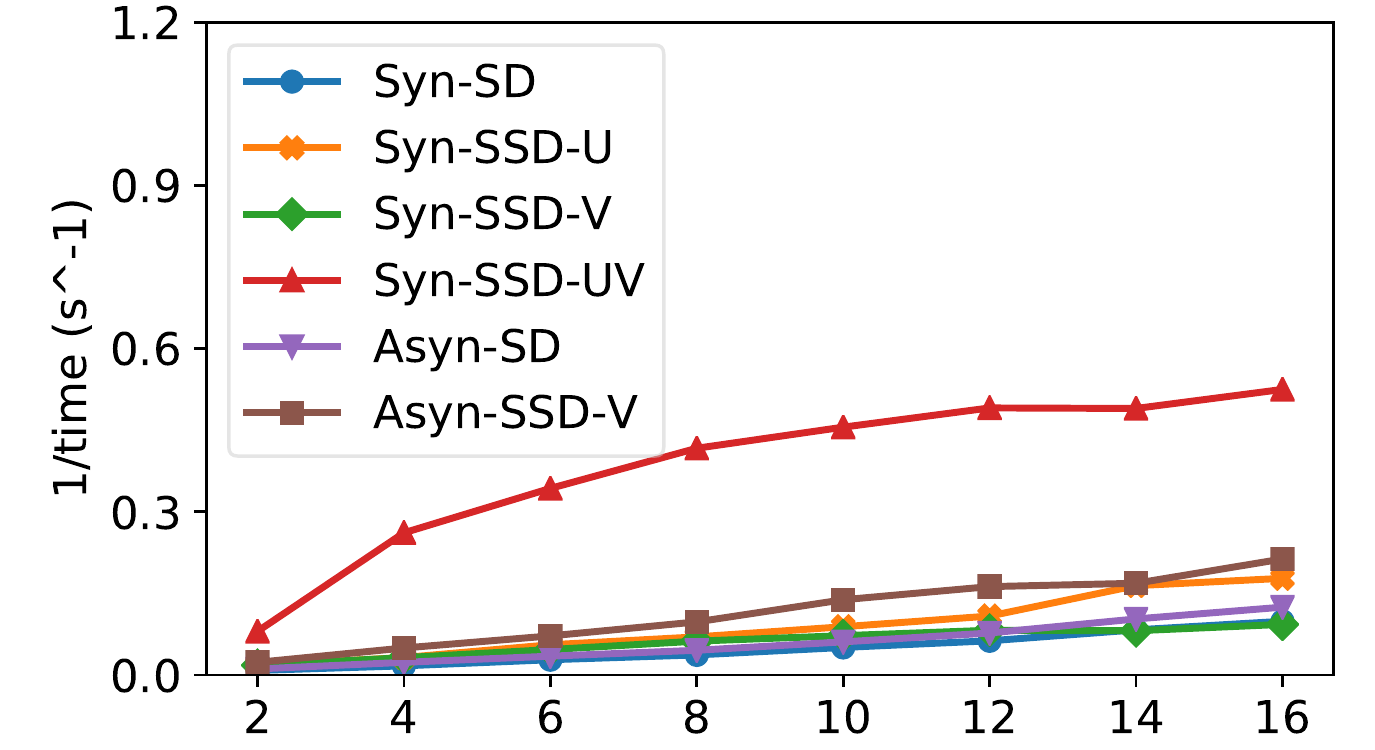}}
  \caption{Reciprocal of per-iteration time for uniform workload in secure distributed NMF}
  \label{fig:secure:scale_uniform}
\end{figure*}

\begin{figure*}[!ht]
  \centering
  \subfigure[BOATS]{
    \label{fig:secure:scalevideo_imba}
    \includegraphics[width=0.22\linewidth]{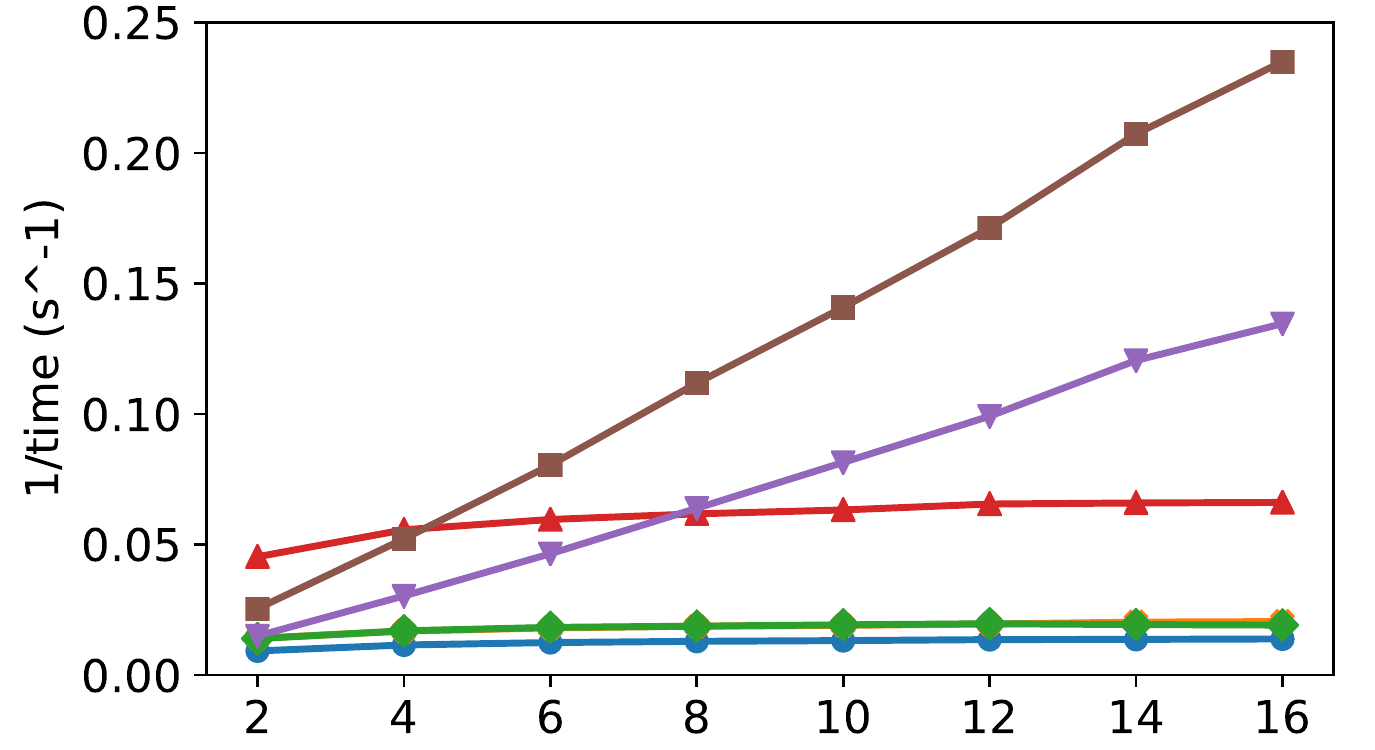}}
    \hspace{0.03in}
  \subfigure[FACE]{
    \label{fig:secure:scaleface_imba}
    \includegraphics[width=0.22\linewidth]{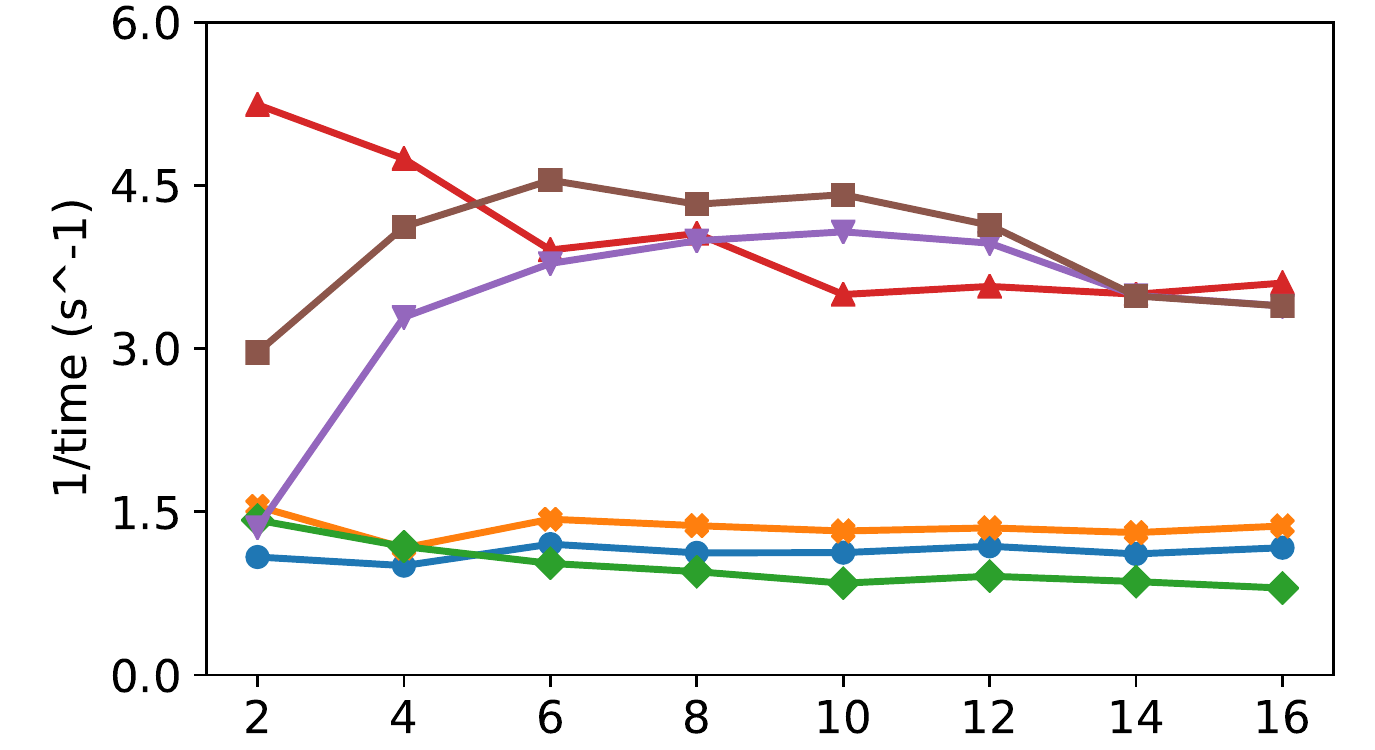}}
    \hspace{0.03in}
  \subfigure[MNIST]{
    \label{fig:secure:scalemnist_imba}
    \includegraphics[width=0.22\linewidth]{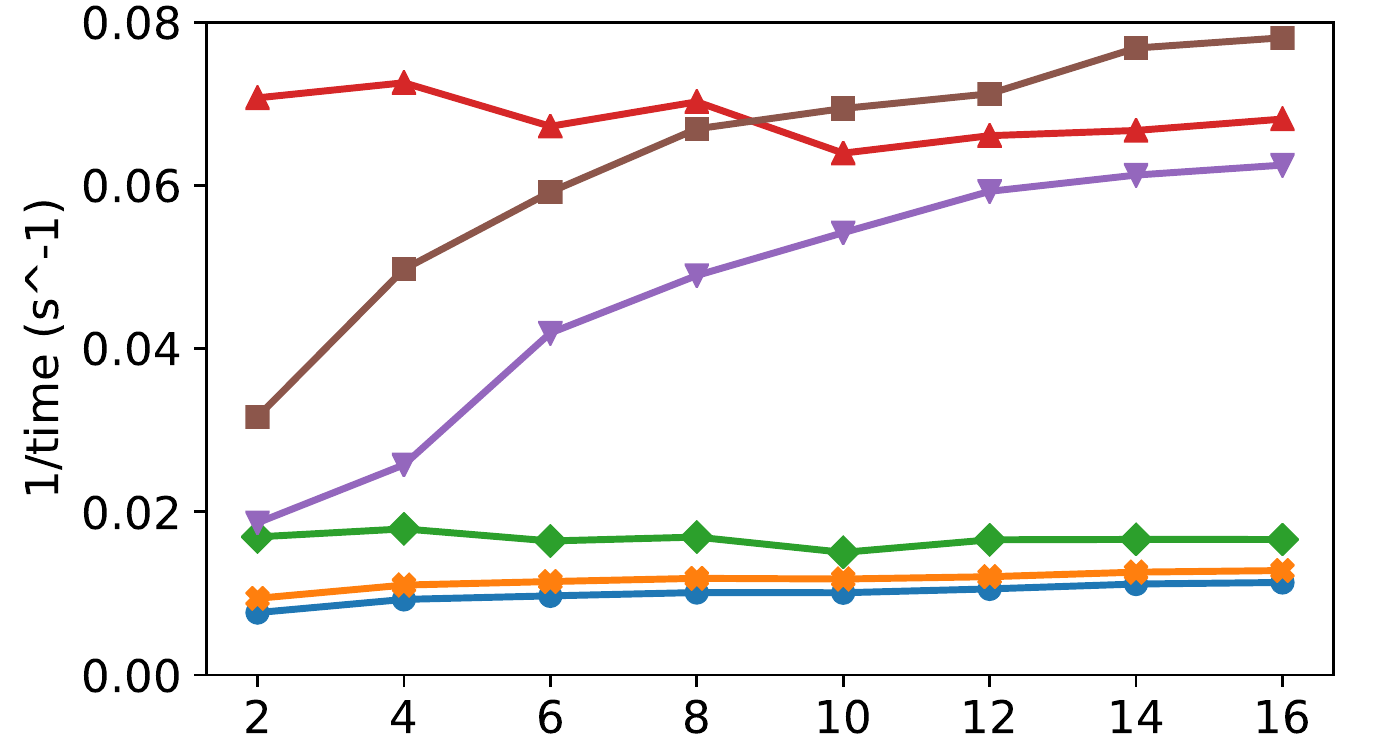}}
    \hspace{0.03in}
  \subfigure[GISETTE]{
    \label{fig:secure:scalenumber_imba}
    \includegraphics[width=0.22\linewidth]{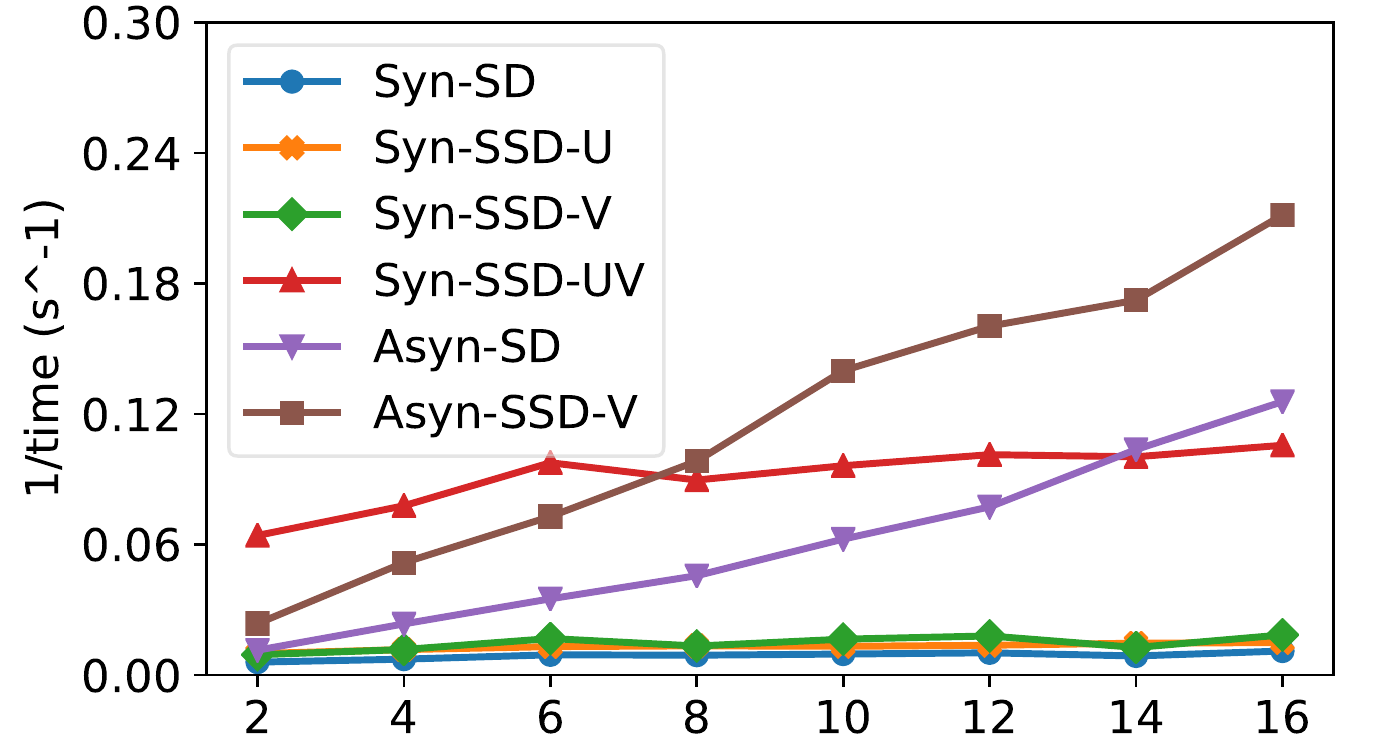}}
    \hspace{0.03in}
  \caption{Reciprocal of per-iteration time for imbalanced workload in secure distributed NMF}
  \label{fig:secure:scale_imba}
\end{figure*}

Since the time for each iteration is significantly reduced by our proposed DSANLS compared to MPI-FAUN,
in Fig.~\ref{fig:dsanl:error_time}, we show the relative error over time
for DSANLS and MPI-FAUN implementations of MU, HALS, and ANLS/BPP on the 6 real public datasets.
Observe that DSANLS/S performs best in all 6 datasets, although DSANLS/G has faster per-iteration convergence rate.
MU converges relatively slowly and usually has a bad convergence
result; on the other hand HALS may oscillate in the early rounds\footnote{HALS does not guarantee the objective function to decrease monotonically.}, but
converges quite fast and to a good solution. Surprisingly, although
ANLS/BPP is considered to be the state-of-art NMF algorithm, it does
not perform well in all 6 datasets. As we will see, this is due to its high per-iteration cost.

\subsubsection{Scalability Comparison}

We vary the number of nodes used in the cluster from 2 to 16 and record the average time for 100 iterations of each algorithm. Fig.~\ref{fig:dsanl:scale} shows the reciprocal of per-iteration time as a function of the number of nodes used.
All algorithms exhibit
good scalability for all datasets (nearly a straight line), except for FACE (i.e., Fig.~\ref{fig:dsanl:scaleface}).
FACE is the smallest dataset, whose number of columns is 300, while $k$ is set to 100 by default. When $n/N$ is smaller than $k$, the complexity is dominated by $k$, hence, increasing the number of nodes does not reduce the computational cost, but may increase the communication overhead.
In general, we can observe that DSANLS/Subsampling has the lowest per-iteration cost compared to all other algorithms, and DSANLS/Gaussian has similar cost to MU and HALS. ANLS/BPP has the highest per-iteration cost, explaining the bad performance of ANLS/BPP in Fig.~\ref{fig:dsanl:error_time}.

\subsubsection{Performance Varying the Value of $k$}

Although tuning the factorization rank $k$ is outside the scope of this paper, we compare the performance of DSANLS with MPI-FAUN varying the value of $k$ from 20 to 500
on RCV1.
Observe from Fig.~\ref{fig:dsanl:varyk} and Fig.~\ref{fig:dsanl:rcv1} ($k=100$) that DSANLS outperforms the state-of-art algorithms for all values of $k$. Naturally, the relative error of all algorithms decreases with the increase of $k$, but they also take longer to converge.

\subsubsection{Comparison with Projected Gradient Descent}
In Sec.~\ref{sec:solving subproblems}, we claimed that
our proximal coordinate descent approach (denoted as DSANLS-RCD) is faster than projected gradient descent (also presented in the same section, denoted as DSANLS-PGD).
Fig.~\ref{fig:dsanl:sub} confirms the difference in the convergence rate of the two approaches regardless of the random matrix generation approach.

\subsection{Evaluation on Secure Distributed NMF}

\subsubsection{Performance Comparison for Uniform Workload}

In Fig.~\ref{fig:secure:error_time}, we show
the relative error over time for secure distributed NMF algorithms on the 4 real public datasets, with a uniformly partition of columns.
\emph{Syn-SSD-UV} performs best in BOAT, FACE and GISETTE. As we will see in the next section, this is due to the fact that per-iteration cost of \emph{Syn-SSD-UV} is significantly reduced by sketching. On MNIST, \emph{Syn-SSD-U} and \emph{Syn-SSD-V} has a better convergence in terms of relative error. \emph{Syn-SD} and \emph{Asyn-SD} converge
relatively slowly and usually have a bad convergence result; on the
other hand \emph{Asyn-SSD-V} converges slowly but consistently generates better results than \emph{Syn-SD} and \emph{Asyn-SD}.

\subsubsection{Performance Comparison for Imbalanced Workload}

To evaluate the performance of different methods when the workload is
imbalanced, we conduct experiments on skewed partition of input matrix. Among
10 worker nodes, node $0$ is assigned with $50\%$ of the columns, while other
nodes have a uniform partition of the rest of columns. The measure for error
is the same as the case of uniform workload.

It can be observed that in imbalanced workload, asynchronous algorithms
generally outperform synchronous algorithms. \emph{Asyn-SSD-V} gives the best
result in terms of relative error over time, except dataset FACE. In FACE,
\emph{Asyn-SD} slowly converges to the best result. Unlike the case of uniform
workload in Fig.~\ref{fig:secure:error_time}, the sketching method
\emph{Syn-SSD-UV} does not perform well in imbalanced workload. \emph{Syn-SD}
are basically inapplicable in BOATS, MNIST and GISETTE datasets due to its
slow speed. In sparse datasets MNIST and GISETTE, \emph{Syn-SSD-V} and
\emph{Syn-SSD-U} can converge to a good result, but they do not generate
satisfactory results on dense dataset BOATS and FACE.

\subsubsection{Scalability Comparison}

We vary the number of nodes used in the cluster from 2 to 16 and record the
average time for 100 iterations of each algorithm.
Fig.~\ref{fig:secure:scale_uniform} shows the reciprocal of per-iteration time
as a function of the number of nodes for uniform workload. All
algorithms exhibit good scalability for all datasets (nearly a straight line),
except for FACE (i.e., Fig.~\ref{fig:secure:scaleface}). FACE is the smallest
dataset, whose number of columns is 361 and number of row is 2,429. When $n/N$
is smaller than $k=100$, the time consumed by subproblem solvers is dominated by the communication overhead. Hence, increasing the number of nodes is does not reduce per-iteration time. In general, we can observe that \emph{Syn-SSD-UV} has
the lowest per-iteration time compared to all other algorithms, and also has
the best scalability as we can see from the steepest slope. Synchronous
averaging has the highest per-iteration cost, explaining the bad performance
in uniform workload experiments in Fig.~\ref{fig:secure:error_time}.

In imbalanced workload settings, it is not surprising that asynchronous
algorithms outperform synchronous algorithms with respect to scalability, as
shown in Fig.~\ref{fig:secure:scale_imba}. Synchronization barriers before
All-Reduce operations severely affect the scalability of synchronous
algorithms, resulting in a nearly flat curve for per-iteration time. The
per-iteration time of \emph{Syn-SSD-UV} is satisfactory when cluster size is
small. However, it does not get significant improvements when more nodes are
deployed. On the other hand, asynchronous algorithms demonstrate decent
scalability as number of nodes grows. The short average iteration time of
\emph{Asyn-SD} and \emph{Asyn-SSD-V}, shown in
Fig.~\ref{fig:secure:scale_imba}, also explains their superior performance
over their synchronous counterparts in Fig.~\ref{fig:secure:error_time_imba}.

In conclusion, with an overall evaluation of convergence and scalability,
\emph{Syn-SSD-UV} should be adopted for secure distributed NMF under uniform
workload, while \emph{Asyn-SSD-V} is a more reasonable choice for secure
distributed NMF under imbalanced workload.

\section{Conclusion}\label{sec:discuss}

In this paper, we studied the acceleration and security problems for distributed
NMF. Firstly, we presented a novel distributed NMF algorithm DSANLS that can
be used for scalable analytics of high dimensional matrix data. Our approach
follows the general framework of ANLS, but utilizes matrix sketching to reduce
the problem size of each NLS subproblem. We discussed and compared two
different approaches for generating random matrices (i.e., Gaussian and
subsampling random matrices). We presented two subproblem solvers for our
general framework, and theoretically proved that our algorithm is convergent.
We analyzed the per-iteration computational and communication cost of our
approach and its convergence, showing its superiority compared to the 
state-of-the-art. Secondly, we designed four efficient distributed NMF methods
in both synchronous and asynchronous settings with a security guarantee. They are
the first distributed NMF methods over federated data, where data from all
parties are utilized together in NMF for better performances and the data of
each party remains confidential without leaking any individual information to
other parties during the process. Finally, we conducted extensive experiments
on several real datasets to show the superiority of our proposed methods. In
the future, we plan to study the applications of DSANLS in dense or sparse
tensors and consider more practical designs of asynchronous algorithm for
secure distributed NMF.

\section{Acknowledgment}
This work was supported by the National Natural Science Foundation of China (no. 61972328).

\appendices

\section{Proof of Lemma~\ref{lemma:exisit_optimal}}\label{sec:lemma1}

\begin{proof}[Proof of Lemma~\ref{lemma:exisit_optimal}] Suppose $(U^*,V^*)$ is the global optimal solution but fails to satisfy Eq.~\ref{eq:extra_con} in the paper. If there exist indices $i,j,l$ such that $U^*_{i:l}\cdot V^*_{j:l}> 2\|M\|_F$, then
\begin{equation}
\small
\begin{aligned}
\left\|M - U^* V^{*\top}\right\|_F^2 &\geq \left(U^*_{i:l}\cdot V^*_{j:l} - M_{i:j} \right)^2 > \left(2\|M\|_F - \|M\|_F\right)^2 \\
& \geq \|M\|_F^2.
\end{aligned}	
\end{equation}
	However, simply choosing $U=0$ and $V=0$ will yield a smaller error
	$\|M\|_F^2$, which contradicts the fact that $(U^*,V^*)$ is optimal. Therefore, if we define $\alpha_l=\max_i U^*_{i:l}$ and $\beta_l=\max_j V^*_{j:l}$, we must have $\alpha_l\cdot\beta_l\leq 2\|M\|_F$ for each $l$. Now we construct a new solution $(\overline{U}, \overline{V})$ by:
	\begin{equation}
	\overline{U}_{i:l} = U^*_{i:l} \cdot \sqrt{\beta_l/\alpha_l}\quad\text{and}\quad \overline{V}_{j:l} = V^*_{j:l} \cdot \sqrt{\alpha_l/\beta_l}.
	\end{equation}
	Note that
	\begin{equation}
	\begin{aligned}
	&\overline{U}_{i:l} \leq \alpha_l \cdot \sqrt{\beta_l/\alpha_l} = \sqrt{\alpha_l\cdot\beta_l} \leq \sqrt{2\|M\|_F},\\
	&\overline{V}_{j:l} \leq \beta_l \cdot \sqrt{\alpha_l/\beta_l} = \sqrt{\alpha_l\cdot\beta_l} \leq \sqrt{2\|M\|_F},
	\end{aligned}
	\end{equation}
	so $(\overline{U}, \overline{V})$ satisfy Eq.~\ref{eq:extra_con} in the paper. Besides,
	\begin{equation}
	\begin{aligned}
	&\left\|M-\overline{U}\,\overline{V}^\top\right\|_F^2 = \sum_{i,j} \Big(M_{i:j} - \sum_{l}\overline{U}_{i:l}\overline{V}_{j:l} \Big)^2 \\
	= & \sum_{i,j} \Big(M_{i:j} - \sum_{l}U^*_{i:l}\cdot\sqrt{\beta_l/\alpha_l}\cdot V^*_{j:l}\cdot\sqrt{\alpha_l/\beta_l} \Big)^2 \\
	=& \sum_{i,j} \Big(M_{i:j} - \sum_{l}U^*_{i:l}\cdot V^*_{j:l} \Big)^2
	= \|M - U^* V^{*\top}\|_F^2,
	\end{aligned}
	\end{equation}
	which means that $(\overline{U}, \overline{V})$ is also an optimal
	solution. In short, for any optimal solution of Eq.~\ref{eq:problem} 
	outside the domain shown in Eq.~\ref{eq:extra_con}, there exists a corresponding
	global optimal solution satisfying the domain shown in Eq.~\ref{eq:extra_con}, which further
	means that there exists at least one optimal solution in the domain shown in Eq.~\ref{eq:extra_con}.
\end{proof}

\section{Proof of Theorem~\ref{theorem:pgd}}\label{proof:Theorem}
For simplicity, we denote $f(U,V) = \|M - U V^\top\|_F^2$, $\tilde{f}_S=\|MS - U (V^\top
S)\|_F^2$, and $\tilde{f}'_{S'}=\|M^\top S' - V (U^\top
S')\|_F^2$. Let $G^t$ and $\tilde{G}^t$ denote the gradients of the
above quantities, i.e.,
\begin{equation}
\small
\begin{aligned}
&G^t\triangleq \nabla_U \left.f(U,V^t)\right|_{U=U^t},\quad \tilde{G}^t \triangleq \nabla_U \left.\tilde{f}_{S^t}(U,V^t)\right|_{U=U^t}, \\
&G'^t\triangleq \nabla_V \left.f(U^{t+1},V)\right|_{V=V^t},\quad \tilde{G}'^t \triangleq \nabla_V \left.\tilde{f}'_{S'^t}(U^{t+1},V)\right|_{V=V^t}.
\end{aligned}
\end{equation}
Besides, let
\begin{equation}
\Delta^t \triangleq \frac{1}{\eta_t}\left(U^t - U^{t+1}\right)\quad\text{and}\quad
\Delta'^t \triangleq \frac{1}{\eta_t}\left(V^t - V^{t+1}\right).
\end{equation}

\subsection{Preliminary Lemmas}\label{proof:Preliminary Lemmas}
To prove Theorem~\ref{theorem:pgd}, we need following lemmas (which
are proved in Sec.~\ref{LemmainTheorem}):

\begin{lemma}\label{lemma:expecation}
	Under Assumptions~\ref{assumption:1} and~\ref{assumption:2}, conditioned on $U^t$ and $V^t$, $\tilde{G}^t$ and $\tilde{G}'^t$ are unbiased estimators of $G^t$ and $G'^t$ respectively with uniformly bounded variance.
\end{lemma}

\begin{lemma}\label{lemma:min}
	Assume $X$ is a nonnegative random variable with mean $\mu$ and variance $\sigma^2$, and $c\geq 0$ is a constant. Then we have
	\begin{equation}\label{eq:to_prove}
	\E\left[\min\{X, c\}\right] \geq \min\left\{c,\frac{\mu}{2}\right\}\cdot\left(1  - \frac{4\sigma^2}{4\sigma^2 + \mu^2}\right).
	\end{equation}
\end{lemma}

\begin{lemma} \label{lemma:intern}
Define the function
\begin{equation}\label{eq:def_phi}
\phi(x,y,z)=\min\left\{|xy|,y^2/2\right\}\cdot\left(1 - \frac{4z^2}{4z^2+y^2}\right)\geq 0.
\end{equation}
Conditioned on $U^t$ and $V^t$, there exists an uniform constant $\sigma'^2 > 0$ such that
\begin{equation}\label{eq:to_prove_2}
\E[G^t_{i:l}\cdot\Delta^t_{i:l}] \geq \phi\left(U^t_{i:l}/\eta_t, G^t_{i:l},\sigma'^2\right)
\end{equation}
and
\begin{equation}
\E[G'^t_{j:l}\cdot\Delta'^t_{j:l}] \geq \phi\left(V^t_{j:l}/\eta_t, G'^t_{j:l},\sigma'^2\right)
\end{equation}
for any $i,j,l$.
\end{lemma}

\begin{lemma}[Supermartingale Convergence Theorem~\cite{neveu1975discrete}]\label{lemma:martingale}
Let $Y_t$, $Z_t$ and $W_t$, $t=0,1,\dots$, be three sequences of random variables and let $\mathcal{F}_t$, $t=0,1,\dots$, be sets of random variables such that $\mathcal{F}_t\subset \mathcal{F}_{t+1}$. Suppose that
\begin{enumerate}[leftmargin=13pt]
\item The random variables $Y_t$, $Z_t$ and $W_t$ are nonnegative, and are functions of the random variables in $\mathcal{F}_t$.
\item For each $t$, we have
    \begin{equation}
    \mathbb{E}[Y_{t+1}|\mathcal{F}_t] \leq Y_t - Z_t + W_t.
    \end{equation}
\item There holds, with probability 1, $\sum_{t=0}^\infty W_t< \infty$.
\end{enumerate}
Then we have $\sum_{t=0}^\infty Z_t< \infty$, and the sequence $Y_t$ converges to a nonnegative random variable $Y$, with probability 1.
\end{lemma}

\begin{lemma}[\cite{mairal2013stochastic}]\label{lemma:convergence_gradient}
For two nonnegative scalar sequences $\{a_t\}$ and $\{b_t\}$, if $\sum_{t=0}^\infty a_k=\infty$ and $\sum_{t=0}^\infty a_tb_t<\infty$, then
\begin{equation}
\liminf_{t\rightarrow \infty} b_t = 0.
\end{equation}
Furthermore, if $|b_{t+1} - b_t|\leq B\cdot a_t$ for some constant $B>0$, then
\begin{equation}
\lim_{t\rightarrow \infty} b_t = 0.
\end{equation}
\end{lemma}

\subsection{Proof of Theorem~\ref{theorem:pgd}}
\begin{proof}[Proof of Theorem~\ref{theorem:pgd}]
	Let us first focus on projected gradient descent. By conditioning on $U^t$ and $V^t$, we have
	\begin{small}
	\begin{align*}
	f(U^{t+1}, V^t)
	= &\left\| M - U^{t+1} V^{t\top} \right\|_F^2
	= \left\| M - \left(U^t - \eta_t \Delta^t \right) V^{t\top} \right\|_F^2\\
	= &\left\| \left(M - U^t V^{t\top}\right) - \eta_t \Delta^t V^{t\top} \right\|_F^2 \\
	=& \left\|M - U^t V^{t\top}\right\|_F^2 - 2\eta_t \left(M - U^t V^{t\top}\right)\cdot \left(\Delta^t V^{t\top}\right) \\
	& + \eta_t^2 \|\Delta^t V^{t\top}\|_F^2 \\
	=& f(U^t,V^t) - 2\eta_t \left(M - U^t V^{t\top}\right)\cdot \left(\Delta^t V^{t\top}\right)\\ 
	& + \eta_t^2 \|\Delta^t V^{t\top}\|_F^2.
	\numberthis \label{eq:sufficient_decrease}
	\end{align*}
	\end{small}For the second term of Eq.~\ref{eq:sufficient_decrease}, note that
	\begin{equation}
	\begin{aligned}
	2\left(M - U^t V^{t\top}\right)\cdot \left(\Delta^t V^{t\top}\right)
	&= 2\tr\left[ \left(M - U^t V^{t\top}\right) V^t \Delta^{t\top} \right] \\
	&= \tr\left[ G^t \Delta^{t\top} \right]
	= \sum_{i,l} G^t_{i:l}\cdot \Delta^t_{i:l}.
	\end{aligned}
	\end{equation}
	By taking expectation and using Lemma~\ref{lemma:intern}, we obtain:
	\begin{equation}
	\begin{aligned}
	\E\left[2\left(M - U^t V^{t\top}\right)\cdot \left(\Delta^t V^{t\top}\right)\right]
	&= \sum_{i,l} \E\left[G^t_{i:l}\cdot \Delta^t_{i:l}\right] \\
	&\geq \sum_{i,l}\phi\left(U^t_{i:l}/\eta_t, G^t_{i:l}, \sigma'^2\right).
	\end{aligned}
	\end{equation}
	For simplicity, we will use the notation
	\begin{equation}
	\Phi(U^t/\eta_t, G^t)\triangleq\sum_{i,l}\phi\left(U^t_{i:l}/\eta_t, G^t_{i:l}, \sigma'^2\right).
	\end{equation}For the third term of Eq.~\ref{eq:sufficient_decrease}, we can bound it in the following way:
	\begin{equation}
	\begin{aligned}
	\|\Delta^t V^{t\top}\|_F^2 \leq& \|\Delta^t\|_F^2 \cdot \|V^t\|_F^2
	\leq \|\tilde{G}^t\|_F^2 \cdot \|V^t\|_F^2 \\
	=& \left\| 2(U^t V^{t\top} - M) (S^t S^{t\top}) V^t \right\|_F^2 \cdot \|V^t\|_F^2 \\
	\leq& 4\|M - U^t V^{t\top}\|_F^2 \cdot \|S^t S^{t\top}\|^2_F \cdot \|V^t\|_F^4 \\
	\leq& 8\left(\|M\|_F^2 + \|U^t\|_F^2\cdot \|V^t\|_F^2\right) \cdot \|S^t S^{t\top}\|^2_F \cdot \|V^t\|_F^4 \\
	\leq& 8\left(\|M\|_F^2 + R^4\right)R^4 \cdot \|S^t S^{t\top}\|^2_F,
	\end{aligned}
	\end{equation}
	where in the last inequality we have applied Assumption~\ref{assumption:2}. If we take expectation, we have
	\begin{equation}
	\small
	\begin{aligned}
	\E\|\Delta^t V^{t\top}\|_F^2 \leq& 8\left(\|M\|_F^2 + R^4\right)R^4 \cdot \E\|S^t S^{t\top}\|^2_F \\
	\leq& 8\left(\|M\|_F^2 + R^4\right)R^4 \cdot \left(\left\|\E[ S^t S^{t\top} ]\right\|^2 + \mathbb{V}[S^t S^{t\top}] \right) \\
	\leq& 8\left(\|M\|_F^2 + R^4\right)R^4 \cdot \left(n + \sigma^2 \right),
	\end{aligned}
	\end{equation}where mean-variance decomposition have been applied in the second inequality, and Assumption~\ref{assumption:1} was used in the last line. For convenience, we will use
	\begin{equation}
	\Gamma\triangleq 8\left(\|M\|_F^2 + R^4\right)R^4 \cdot \left(n + \sigma^2 \right)\geq0
	\end{equation}
	to denote this constant later on.
	
	\vspace{5pt}
	\noindent By combining all results, we can rewrite Eq.~\ref{eq:sufficient_decrease} as
	\begin{equation}
	\E \left[f(U^{t+1},V^t)\right] \leq f(U^t, V^t) - \eta_t\Phi\left(U^t/\eta_t, G^t\right) + \eta_t^2 \Gamma.
	\end{equation}

	\vspace{5pt}
	\noindent Likewise, conditioned on $U^{t+1}$ and $V^t$, we can prove a similar inequality for $V$:
	\begin{equation}
	\E \left[f(U^{t+1},V^{t+1})\right] \leq f(U^{t+1}, V^t) - \eta_t\Phi\left(V^t/\eta_t, G'^t\right) + \eta_t^2 \Gamma',
	\end{equation}
	where $\Gamma'\geq 0$ is also some uniform constant. From definition,
	it is easy to see both $\Phi\left(U^t/\eta_t, G^t\right)$ and
	$\Phi\left(V^t/\eta_t, G'^t\right)$ are nonnegative. Along with
	condition the condition $\sum_{t=0}^\infty \eta_t^2<\infty$, we can
	apply the Supermartingale Convergence Theorem (Lemma \ref{lemma:martingale}) with
	\begin{equation}
	\begin{aligned}
	&Y_{2t}=f(U^t, V^t),\quad Y_{2t+1}=f(U^{t+1}, V^t),\\
	&Z_{2t}=\Phi\left(U^t/\eta_t,G^t\right),\quad Z_{2t+1}=\Phi\left(V^t/\eta_t,G'^t\right), \\
	&W_{2t}=\Gamma \eta_t^2,\quad W_{2t+1}=\Gamma' \eta_t^2,
	\end{aligned}
	\end{equation}
	and then conclude that both $\{f(U^{t+1},V^t)\}$ and $\{f(U^t,V^t)\}$ will converge to a same value, and besides:
	\begin{equation}
	\sum_{t=0}^\infty \eta_t \left[\Phi\left(U^t/\eta_t, G^t\right) + \Phi\left(V^t/\eta_t, G'^t\right)\right] < \infty,
	\end{equation}
	with probability 1. In addition, it is not hard to verify that $\left|\Phi\left(U^{t+1}/\eta_{t+1},G^{t+1}\right) - \Phi\left(U^t/\eta_t,G^t\right)\right|\leq C\cdot\eta_t$ for some constant $C$ because of the boundness of the gradients. Then, by Lemma \ref{lemma:convergence_gradient}, we obtain that
	\begin{equation}
	\lim_{t\rightarrow\infty} \Phi\left(U^t/\eta_t, G^t\right)=\lim_{t\rightarrow\infty} \sum_{i:l}\phi\left(U^t_{i:l}/\eta_t, G^t_{i:l},\sigma'^2\right) \rightarrow 0.
	\end{equation}
	Since each summand in the above is nonnegative, this equation further implies
	\begin{equation}
	\lim_{t\rightarrow\infty} \phi\left(U^t_{i:l}/\eta_t, G^t_{i:l},\sigma'^2\right) \rightarrow 0
	\end{equation}
	for all $i$ and $l$. By looking into the definition of $\phi$ in Eq.~\ref{eq:def_phi}, it is not hard to see that $\phi\left(U^t_{i:l}/\eta_t, G^t_{i:l},\sigma'^2\right)\rightarrow 0$ if and only if $\min\left\{U^t_{i:l}/\eta_t,\, \big|G^t_{i:l}\big|\right\}\rightarrow 0$. Considering $\eta_t>0$ and $\eta_t\rightarrow 0$, we can conclude that
	\begin{equation}
	\lim_{t\rightarrow \infty} \min\left\{U^t_{i:l},\, \big|G^t_{i:l}\big|\right\}\rightarrow 0
	\end{equation}
	for all $i,l$, which means either the gradient $G^t_{i:l}$ converges to $0$, or $U^t_{i:l}$ converges to the boundary $0$. In other words, the projected gradient at $(U^t, V^t)$ w.r.t $U$ converges to $0$ as $t\rightarrow \infty$. Likewise, we can prove
	\begin{equation}
	\lim_{t\rightarrow \infty} \min\left\{V^t_{j:l},\, \big|G'^t_{j:l}\big|\right\}\rightarrow 0,
	\end{equation}
	in a similar way, which completes the proof of projected gradient descent.
	
	\vspace{5pt}
	\noindent The proof of regularized coordinate descent is similar to that of projected
	gradient descent, and hence we only include a sketch proof here. The
	key here is to establish an inequality similar to Eq.~\ref{eq:sufficient_decrease}, but with the difference that just one
	column rather than whole $U$ or $V$ is changed every time. Take $U_{:1}$ as an example.
	An important observation is that when projection does not happen, we
	can rewrite (19) in the paper as
	$U^{t+1}_{:1}=U^t_{:1}-\tilde{G}_{:1}/(\tau_t+B^t_{j:}B^{t\top}_{j:})$,
	which means that the moving direction of regularized coordinate
	descent is the same as that of projected gradient descent, but with step size being $1/(\tau_t+B^t_{j:}B^{t\top}_{j:})$. Since both the expectation and variance of $B^t_{j:}B^{t\top}_{j:}$ are bounded, we will have $1/(\tau_t+B^t_{j:}B^{t\top}_{j:})\approx 1/\tau_t$ when $\tau_t$ is large. Given these two reasons, we can out down a similar inequality as Eq.~\ref{eq:sufficient_decrease}. The remaining proof just follows the one for projected gradient descent.
\end{proof}

\subsection{Proof of Preliminary Lemmas}\label{LemmainTheorem}


\begin{proof}[Proof of Lemma~\ref{lemma:expecation}]
	Since the proof related to $\tilde{G}'^t$ is similar to $\tilde{G}^t$, here we only focus on the latter one.
	
	\vspace{5pt}
	\noindent First, let us write down the definition of $G^t$ and $\tilde{G}^t$:
	\begin{equation}
	\begin{aligned}
	&G^t = 2(U^t V^{t\top} - M) V^t\\ 
	&\tilde{G}^t = 2(U^t V^{t\top} - M) (S^t S^{t\top}) V^t.
	\end{aligned}
	\end{equation}
	Therefore,
	\begin{equation}
	\begin{aligned}
	\E[\tilde{G}^t] = &\E\left[2(U^t V^{t\top} - M) (S^t S^{t\top}) V^t\right]\\
	= & 2(U^t V^{t\top} - M)\, \E[ S^t S^{t\top} ]\, V^t\\
	=&2(U^t V^{t\top} - M)\, I\, V^t
	= 2(U^t V^{t\top} - M) V^t
	= G^t,
	\end{aligned}
	\end{equation}
	which means $\tilde{G}^t$ is an unbiased estimator of $G^t$.
	Besides, its variance is uniformly bounded because
	\begin{equation}
	\begin{aligned}
	\mathbb{V}[\tilde{G}^t] \leq& \mathbb{V}\left[ 2(U^t V^{t\top} - M)_{i:} (S^t S^{t\top}) V^t_{:l}\right] \\
	\leq& 4\|M - U^t V^{t\top}\|_F^2 \cdot \mathbb{V}[S^t S^{t\top}] \cdot \|V^t\|_F^2 \\
	\leq& 8\left(\|M\|_F^2 + \|U^t\|_F^2 \|V^t\|_F^2\right)  \cdot \|V^t\|_F^2 \cdot \mathbb{V}[S^t S^{t\top}] \\
	\leq& 8\left(\|M\|_F^2 + R^4\right)R^2\cdot\sigma^2,
	\end{aligned}
	\end{equation}
	where both Assumptions~\ref{assumption:1} and~\ref{assumption:2} are applied in the last line.
\end{proof}

\begin{proof}[Proof of Lemma \ref{lemma:min}]
	In this proof, we will use Cantelli's inequality:
	\begin{equation}
	\P(X\geq \mu + \lambda) \geq 1 - \frac{\sigma^2}{\sigma^2+\lambda^2} \quad\forall \lambda < 0.
	\end{equation}
	
	\vspace{5pt}
	\noindent When $\mu=0$, it is easy to see that the right-hand-side of Eq.~\ref{eq:to_prove} is $0$. Considering that the left-hand-side is the expectation of a nonnegative random variable, Eq.~\ref{eq:to_prove} obviously holds in this case.
	
	\vspace{5pt}
	\noindent When $\mu> 0$ and $\mu\geq 2c$, by using the fact that $X$ is nonnegative, we have
	\begin{equation}
	\E\left[\min\{X, c\}\right] \geq c\cdot \P(X\geq c).
	\end{equation}
	Now we can apply Cantelli's inequality to bound $\P(X\geq c)$ with $\lambda=c - \mu < c - \mu/2\leq 0$, and obtain:
	\begin{equation}
	\label{eq:part_1}
	\begin{aligned}
	\E\left[\min\{X, c\}\right]
	& \geq c\cdot \left(1  - \frac{\sigma^2}{\sigma^2 + (\mu - c)^2}\right) \\
	& \geq c\cdot \left(1  - \frac{\sigma^2}{\sigma^2 + (\mu - \mu/2)^2}\right) \\
	&= c\cdot \left(1  - \frac{4\sigma^2}{4\sigma^2 + \mu^2}\right), 
	\end{aligned}
	\end{equation}
	where in the second inequality we used the fact $c\leq\mu/2$ again.
	
	\vspace{5pt}
	\noindent When $\mu> 0$ but $\mu<2c$, we have:
	\begin{equation}
	\E\left[\min\{X, c\}\right]\geq
	\E\left[\min\{X, \mu/2\}\right].
	\end{equation}
	Now we can apply inequality in Eq.~\ref{eq:part_1} from the previous part with $c=\mu/2$, and thus
	\begin{equation}
	\E\left[\min\{X, c\}\right]\geq \E\left[\min\{X, \mu/2\}\right] \geq \frac{\mu}{2}\cdot \left(1  - \frac{4\sigma^2}{4\sigma^2 + \mu^2}\right),
	\end{equation}
	which completes the proof.
\end{proof}

\begin{proof}[Proof of Lemma \ref{lemma:intern}]
	We only focus on $G^t$ and $\Delta^t$.
	We first show that
	\begin{equation}\label{eq:sign_preserving}
	G^t_{i:l} \cdot \tilde{G}^t_{i:l}\geq 0
	\end{equation}
	for any random matrix $S^t$. Note that
	\begin{equation}
	\begin{aligned}
	&G^t_{i:l}=2(U^t V^{t\top} - M)_{i:}V^t_{:l} \\
	&\tilde{G}^t_{i:l}=2(U^t V^{t\top} - M)_{i:} (S^t S^{t\top}) V^t_{:l}.
	\end{aligned}
	\end{equation}
	Hence it would be sufficient if we can show that there holds $a^\top (S^t S^{t\top}) b\cdot a^\top b\geq 0$ for any vectors $a$ and $b$:
	\begin{equation}
	\begin{aligned}
	a^\top (S^t S^{t\top}) b\cdot a^\top b &= \tr\left( a^\top (S^t S^{t\top}) b b^\top a \right) \\
	&= \tr\left( a a^\top (S^t S^{t\top}) b b^\top \right)
	\geq 0,
	\end{aligned}
	\end{equation}
	where the first equality is because $A\cdot B=\tr(AB^\top)$, the second equality is due to cyclic permutation invariant property of trace, and the last inequality is because all of $a a^\top$, $b b^\top$ and $S^t S^{t\top}$ are positive semi-definite matrices.
	
	\vspace{5pt}
	\noindent Now, let us consider the relationship between $\Delta^t$ and $\tilde{G}^t$:
	\begin{equation}
	\Delta^t = \frac{1}{\eta_t}\left(U^t - U^{t+1}\right) = \frac{1}{\eta_t}\left(U^t - \max\left\{ U^t - \eta_t\tilde{G}^t,\, 0 \right\}\right),
	\end{equation}
	from which it can be shown that
	\begin{equation}\label{eq:delta}
	\Delta^t_{i:l} = \min\left\{ U^t_{i:l}/\eta_t,\, \tilde{G}^t_{i:l} \right\}.
	\end{equation}
	
	\vspace{5pt}
	\noindent When $G^t_{i:l}=0$, it is easy to see that both sides of Eq.\ref{eq:to_prove_2} become 0, and hence Eq.\ref{eq:to_prove_2} holds.

	\vspace{5pt}
	\noindent When $G^t_{i:l}>0$, from Eq.\ref{eq:sign_preserving} we know that $\tilde{G}^t_{i:l}\geq 0$ regardless of the choice of $S^t$. From Lemma~\ref{lemma:expecation} we know that
	\begin{equation}
	\E[\tilde{G}^t_{i:l}]=G^t_{i:l}
	\end{equation}
	and there exists a constant $\sigma'^2 \geq 0$ such that
	\begin{equation}
	\mathbb{V}[\tilde{G}^t_{i:l}]\leq \sigma'^2.
	\end{equation}
	Since $U^t_{i:l}$ is a nonnegative constant here, we can apply Lemma~\ref{lemma:min} to Eq.\ref{eq:delta} and conclude
	\begin{equation}
	\begin{aligned}
	\E[\Delta^t_{i:l}] \geq& \min\left\{U^t_{i:l}/\eta_t, G^t_{i:l}/2\right\} \cdot \left(1 - \frac{4\mathbb{V}[\tilde{G}^t_{i:l}]}{4\mathbb{V}[\tilde{G}^t_{i:l}] + \big(G^t_{i:l}\big)^2}\right) \\
	\geq& \min\left\{U^t_{i:l}/\eta_t, G^t_{i:l}/2\right\} \cdot \left(1 - \frac{4\sigma'^2}{4\sigma'^2 + \big(G^t_{i:l}\big)^2}\right),
	\end{aligned}
	\end{equation}
	from which Eq.\ref{eq:to_prove_2} is obvious.

	\vspace{5pt}
	\noindent When $G^t_{i:l}<0$, also from Eq.~\ref{eq:sign_preserving} we know that $\tilde{G}^t_{i:l}\leq 0$. Since $U^t_{i:l}$ is a nonnegative constant here, we always have
	\begin{equation}
	\Delta^t_{i:l} = \min\left\{ U^t_{i:l}/\eta_t,\, \tilde{G}^t_{i:l} \right\} = \tilde{G}^t_{i:l}.
	\end{equation}
	Therefore, by taking expectation and using Lemma~\ref{lemma:expecation}, we obtain
	\begin{equation}
	\E[\Delta^t_{i:l}] = \E[\tilde{G}^t_{i:l}] = G^t_{i:l},
	\end{equation}
	and thus
	\begin{equation}
	\E\left[G^t_{i:l}\cdot \Delta^t_{i:l}\right] = \left(G^t_{i:l}\right)^2> \frac{\big(G^t_{i:l}\big)^2}{2} \cdot \left(1 - \frac{4\sigma'^2}{4\sigma'^2 + \big(G^t_{i:l}\big)^2}\right)
	\end{equation}
	for any constant $\sigma'$, which means that Eq.~\ref{eq:to_prove_2} holds.
\end{proof}

\bibliographystyle{IEEEtran}
\bibliography{main}

\begin{thebibliography}{53}
\providecommand{\natexlab}[1]{#1}
\providecommand{\url}[1]{#1}
\csname url@samestyle\endcsname
\providecommand{\newblock}{\relax}
\providecommand{\bibinfo}[2]{#2}
\providecommand{\BIBentrySTDinterwordspacing}{\spaceskip=0pt\relax}
\providecommand{\BIBentryALTinterwordstretchfactor}{4}
\providecommand{\BIBentryALTinterwordspacing}{\spaceskip=\fontdimen2\font plus
\BIBentryALTinterwordstretchfactor\fontdimen3\font minus
  \fontdimen4\font\relax}
\providecommand{\BIBforeignlanguage}[2]{{%
\expandafter\ifx\csname l@#1\endcsname\relax
\typeout{** WARNING: IEEEtranN.bst: No hyphenation pattern has been}%
\typeout{** loaded for the language `#1'. Using the pattern for}%
\typeout{** the default language instead.}%
\else
\language=\csname l@#1\endcsname
\fi
#2}}
\providecommand{\BIBdecl}{\relax}
\BIBdecl

\bibitem[Pauca et~al.(2004)Pauca, Shahnaz, Berry, and Plemmons]{pauca2004text}
V.~P. Pauca, F.~Shahnaz, M.~W. Berry, and R.~J. Plemmons, ``Text mining using
  non-negative matrix factorizations,'' in \emph{{SDM}}, 2004, pp. 452--456.

\bibitem[Kotsia et~al.(2007)Kotsia, Zafeiriou, and Pitas]{kotsia2007novel}
I.~Kotsia, S.~Zafeiriou, and I.~Pitas, ``A novel discriminant non-negative
  matrix factorization algorithm with applications to facial image
  characterization problems,'' \emph{{IEEE} Trans. Information Forensics and
  Security}, vol.~2, no. 3-2, pp. 588--595, 2007.

\bibitem[Gu et~al.(2010)Gu, Zhou, and Ding]{gu2010collaborative}
Q.~Gu, J.~Zhou, and C.~H.~Q. Ding, ``Collaborative filtering: Weighted
  nonnegative matrix factorization incorporating user and item graphs,'' in
  \emph{{SDM}}, 2010, pp. 199--210.

\bibitem[Zhang and Yeung(2012)]{ZhangY12}
Y.~Zhang and D.~Yeung, ``Overlapping community detection via bounded
  nonnegative matrix tri-factorization,'' in \emph{{KDD}}, 2012, pp. 606--614.

\bibitem[Kannan et~al.(2016)Kannan, Ballard, and Park]{kannan2016high}
R.~Kannan, G.~Ballard, and H.~Park, ``A high-performance parallel algorithm for
  nonnegative matrix factorization,'' in \emph{{PPOPP}}, 2016, pp. 9:1--9:11.

\bibitem[Kim et~al.(2017)Kim, Sun, Yu, and Jiang]{KimSYJ17}
Y.~Kim, J.~Sun, H.~Yu, and X.~Jiang, ``Federated tensor factorization for
  computational phenotyping,'' in \emph{{KDD}}, 2017, pp. 887--895.

\bibitem[Feng et~al.(2018)Feng, Yang, Zhu, and Choo]{feng2018privacy}
J.~Feng, L.~T. Yang, Q.~Zhu, and K.-K.~R. Choo, ``Privacy-preserving tensor
  decomposition over encrypted data in a federated cloud environment,''
  \emph{{IEEE} Trans. Dependable Sec. Comput.}, 2018.

\bibitem[Kannan et~al.(2018)Kannan, Ballard, and Park]{kannan2016mpi}
R.~Kannan, G.~Ballard, and H.~Park, ``{MPI-FAUN:} an mpi-based framework for
  alternating-updating nonnegative matrix factorization,'' \emph{{IEEE} Trans.
  Knowl. Data Eng.}, vol.~30, no.~3, pp. 544--558, 2018.

\bibitem[Lee and Seung(2000)]{lee2001algorithms}
D.~D. Lee and H.~S. Seung, ``Algorithms for non-negative matrix
  factorization,'' in \emph{{NIPS}}, 2000, pp. 556--562.

\bibitem[Gillis(2014)]{gillis2014and}
\BIBentryALTinterwordspacing
N.~Gillis, ``The why and how of nonnegative matrix factorization,'' \emph{arXiv
  Preprint}, 2014. [Online]. Available: \url{https://arxiv.org/abs/1401.5226}
\BIBentrySTDinterwordspacing

\bibitem[Daube-Witherspoon and Muehllehner(1986)]{daube1986iterative}
M.~E. Daube-Witherspoon and G.~Muehllehner, ``An iterative image space
  reconstruction algorthm suitable for volume ect,'' \emph{{IEEE} Trans. Med.
  Imaging}, vol.~5, no.~2, pp. 61--66, 1986.

\bibitem[Grippo and Sciandrone(2000)]{grippo2000convergence}
L.~Grippo and M.~Sciandrone, ``On the convergence of the block nonlinear
  gauss-seidel method under convex constraints,'' \emph{Oper. Res. Lett.},
  vol.~26, no.~3, pp. 127--136, 2000.

\bibitem[Kim and Park(2008)]{kim2008nonnegative}
H.~Kim and H.~Park, ``Nonnegative matrix factorization based on alternating
  nonnegativity constrained least squares and active set method,'' \emph{{SIAM}
  J. Matrix Analysis Applications}, vol.~30, no.~2, pp. 713--730, 2008.

\bibitem[Kim and Park(2011)]{kim2011fast}
J.~Kim and H.~Park, ``Fast nonnegative matrix factorization: An active-set-like
  method and comparisons,'' \emph{{SIAM} J. Scientific Computing}, vol.~33,
  no.~6, pp. 3261--3281, 2011.

\bibitem[Lin(2007)]{lin2007projected}
C.~Lin, ``Projected gradient methods for nonnegative matrix factorization,''
  \emph{Neural Computation}, vol.~19, no.~10, pp. 2756--2779, 2007.

\bibitem[Zdunek and Cichocki(2006)]{zdunek2006non}
R.~Zdunek and A.~Cichocki, ``Non-negative matrix factorization with
  quasi-newton optimization,'' in \emph{{ICAISC}}, vol. 4029, 2006, pp.
  870--879.

\bibitem[Guan et~al.(2012)Guan, Tao, Luo, and Yuan]{guan2012nenmf}
N.~Guan, D.~Tao, Z.~Luo, and B.~Yuan, ``Nenmf: An optimal gradient method for
  nonnegative matrix factorization,'' \emph{{IEEE} Trans. Signal Processing},
  vol.~60, no.~6, pp. 2882--2898, 2012.

\bibitem[Naor and Nissim(2001)]{NaorN01}
M.~Naor and K.~Nissim, ``Communication preserving protocols for secure function
  evaluation,'' in \emph{{STOC}}, 2001, pp. 590--599.

\bibitem[Kanjani(2007)]{kanjani2007parallel}
K.~Kanjani, ``Parallel non negative matrix factorization for document
  clustering,'' \emph{CPSC-659 (Parallel and Distributed Numerical Algorithms)
  course. Texas A\&M University, Tech. Rep}, 2007.

\bibitem[Robila and Maciak(2006)]{robila2006parallel}
S.~A. Robila and L.~G. Maciak, ``A parallel unmixing algorithm for
  hyperspectral images,'' in \emph{Intelligent Robots and Computer Vision XXIV:
  Algorithms, Techniques, and Active Vision}, vol. 6384, 2006, p. 63840F.

\bibitem[Liu et~al.(2010)Liu, Yang, Fan, He, and Wang]{liu2010distributed}
C.~Liu, H.~Yang, J.~Fan, L.~He, and Y.~Wang, ``Distributed nonnegative matrix
  factorization for web-scale dyadic data analysis on mapreduce,'' in
  \emph{{WWW}}, 2010, pp. 681--690.

\bibitem[Liao et~al.(2014)Liao, Zhang, Guan, and Zhou]{liao2014cloudnmf}
R.~Liao, Y.~Zhang, J.~Guan, and S.~Zhou, ``Cloudnmf: {A} mapreduce
  implementation of nonnegative matrix factorization for large-scale biological
  datasets,'' \emph{Genomics, Proteomics {\&} Bioinformatics}, vol.~12, no.~1,
  pp. 48--51, 2014.

\bibitem[Yin et~al.(2014)Yin, Gao, and Zhang]{yin2014scalable}
J.~Yin, L.~Gao, and Z.~M. Zhang, ``Scalable nonnegative matrix factorization
  with block-wise updates,'' in \emph{{ECML/PKDD} {(3)}}, ser. Lecture Notes in
  Computer Science, vol. 8726, 2014, pp. 337--352.

\bibitem[Meng et~al.(2016)Meng, Bradley, Yavuz, Sparks, Venkataraman, Liu,
  Freeman, Tsai, Amde, Owen, Xin, Xin, Franklin, Zadeh, Zaharia, and
  Talwalkar]{meng2016mllib}
X.~Meng, J.~K. Bradley, B.~Yavuz, E.~R. Sparks, S.~Venkataraman, D.~Liu,
  J.~Freeman, D.~B. Tsai, M.~Amde, S.~Owen, D.~Xin, R.~Xin, M.~J. Franklin,
  R.~Zadeh, M.~Zaharia, and A.~Talwalkar, ``Mllib: Machine learning in apache
  spark,'' \emph{J. Mach. Learn. Res.}, vol.~17, pp. 34:1--34:7, 2016.

\bibitem[Grove et~al.(2014)Grove, Milthorpe, and Tardieu]{grove2014supporting}
D.~Grove, J.~Milthorpe, and O.~Tardieu, ``Supporting array programming in
  {X10},'' in \emph{ARRAY@PLDI}, 2014, pp. 38--43.

\bibitem[Mej{\'{\i}}a{-}Roa et~al.(2015)Mej{\'{\i}}a{-}Roa, Tabas{-}Madrid,
  Setoain, Garc{\'{\i}}a, Tirado, and Pascual{-}Montano]{mejia2015nmf}
E.~Mej{\'{\i}}a{-}Roa, D.~Tabas{-}Madrid, J.~Setoain, C.~Garc{\'{\i}}a,
  F.~Tirado, and A.~D. Pascual{-}Montano, ``Nmf-mgpu: non-negative matrix
  factorization on multi-gpu systems,'' \emph{{BMC} Bioinformatics}, vol.~16,
  pp. 43:1--43:12, 2015.

\bibitem[Gower and Richt{\'{a}}rik(2015)]{gower2015randomized}
R.~M. Gower and P.~Richt{\'{a}}rik, ``Randomized iterative methods for linear
  systems,'' \emph{{SIAM} J. Matrix Analysis Applications}, vol.~36, no.~4, pp.
  1660--1690, 2015.

\bibitem[Pilanci and Wainwright(2016)]{PilanciW16}
M.~Pilanci and M.~J. Wainwright, ``Iterative hessian sketch: Fast and accurate
  solution approximation for constrained least-squares,'' \emph{J. Mach. Learn.
  Res.}, vol.~17, pp. 53:1--53:38, 2016.

\bibitem[Pilanci and Wainwright(2017)]{PilanciW17}
M.~Pilanci and M.~J. Wainwright, ``Newton sketch: {A} near linear-time
  optimization algorithm with linear-quadratic convergence,'' \emph{{SIAM}
  Journal on Optimization}, vol.~27, no.~1, pp. 205--245, 2017.

\bibitem[Wang and Li(2010)]{wang2010efficient}
F.~Wang and P.~Li, ``Efficient nonnegative matrix factorization with random
  projections,'' in \emph{{SDM}}, 2010, pp. 281--292.

\bibitem[Lindell and Pinkas(2000)]{lindell2000privacy}
Y.~Lindell and B.~Pinkas, ``Privacy preserving data mining,'' in
  \emph{{CRYPTO}}, vol. 1880, 2000, pp. 36--54.

\bibitem[Wan et~al.(2007)Wan, Ng, Han, and Lee]{WanNHL07}
L.~Wan, W.~K. Ng, S.~Han, and V.~C.~S. Lee, ``Privacy-preservation for gradient
  descent methods,'' in \emph{{KDD}}, 2007, pp. 775--783.

\bibitem[Han et~al.(2010)Han, Ng, Wan, and Lee]{HanNWL10}
S.~Han, W.~K. Ng, L.~Wan, and V.~C.~S. Lee, ``Privacy-preserving
  gradient-descent methods,'' \emph{{IEEE} Trans. Knowl. Data Eng.}, vol.~22,
  no.~6, pp. 884--899, 2010.

\bibitem[Pathak and Raj(2010)]{PathakR10a}
M.~A. Pathak and B.~Raj, ``Privacy preserving protocols for eigenvector
  computation,'' in \emph{{PSDML}}, vol. 6549, 2010, pp. 113--126.

\bibitem[Han et~al.(2009)Han, Ng, and Yu]{han2009privacy}
S.~Han, W.~K. Ng, and P.~S. Yu, ``Privacy-preserving singular value
  decomposition,'' in \emph{{ICDE}}, 2009, pp. 1267--1270.

\bibitem[Chen et~al.(2017)Chen, Lu, and Zhang]{ChenLZ17}
S.~Chen, R.~Lu, and J.~Zhang, ``A flexible privacy-preserving framework for
  singular value decomposition under internet of things environment,'' in
  \emph{{IFIPTM}}, vol. 505, 2017, pp. 21--37.

\bibitem[Sakuma and Kobayashi(2010)]{SakumaK10}
J.~Sakuma and S.~Kobayashi, ``Large-scale \emph{k}-means clustering with
  user-centric privacy-preservation,'' \emph{Knowl. Inf. Syst.}, vol.~25,
  no.~2, pp. 253--279, 2010.

\bibitem[Lin and Jaromczyk(2011)]{LinJ11}
Z.~Lin and J.~W. Jaromczyk, ``Privacy preserving spectral clustering over
  vertically partitioned data sets,'' in \emph{{FSKD}}, 2011, pp. 1206--1211.

\bibitem[Duan and Canny(2008)]{DuanC08}
Y.~Duan and J.~F. Canny, ``Practical private computation and zero-knowledge
  tools for privacy-preserving distributed data mining,'' in
  \emph{{SDM}}.\hskip 1em plus 0.5em minus 0.4em\relax {SIAM}, 2008, pp.
  265--276.

\bibitem[Beerliov{\'{a}}{-}Trub{\'{\i}}niov{\'{a}} and
  Hirt(2008)]{beerliova2008perfectly}
Z.~Beerliov{\'{a}}{-}Trub{\'{\i}}niov{\'{a}} and M.~Hirt, ``Perfectly-secure
  {MPC} with linear communication complexity,'' in \emph{{TCC}}, vol. 4948,
  2008, pp. 213--230.

\bibitem[Damg{\aa}rd and Nielsen(2007)]{damgaard2007scalable}
I.~Damg{\aa}rd and J.~B. Nielsen, ``Scalable and unconditionally secure
  multiparty computation,'' in \emph{{CRYPTO}}, vol. 4622, 2007, pp. 572--590.

\bibitem[Ailon and Chazelle(2006)]{ailon2006approximate}
N.~Ailon and B.~Chazelle, ``Approximate nearest neighbors and the fast
  johnson-lindenstrauss transform,'' in \emph{{STOC}}, 2006, pp. 557--563.

\bibitem[Lu et~al.(2013)Lu, Dhillon, Foster, and Ungar]{lu2013faster}
Y.~Lu, P.~S. Dhillon, D.~P. Foster, and L.~H. Ungar, ``Faster ridge regression
  via the subsampled randomized hadamard transform,'' in \emph{{NIPS}}, 2013,
  pp. 369--377.

\bibitem[Clarkson and Woodruff(2013)]{clarkson2013low}
K.~L. Clarkson and D.~P. Woodruff, ``Low rank approximation and regression in
  input sparsity time,'' in \emph{{STOC}}, 2013, pp. 81--90.

\bibitem[Pham and Pagh(2013)]{pham2013fast}
N.~Pham and R.~Pagh, ``Fast and scalable polynomial kernels via explicit
  feature maps,'' in \emph{{KDD}}, 2013, pp. 239--247.

\bibitem[Nemirovski et~al.(2009)Nemirovski, Juditsky, Lan, and
  Shapiro]{nemirovski2009robust}
A.~Nemirovski, A.~Juditsky, G.~Lan, and A.~Shapiro, ``Robust stochastic
  approximation approach to stochastic programming,'' \emph{{SIAM} J. on
  Optim.}, vol.~19, no.~4, pp. 1574--1609, 2009.

\bibitem[Rockafellar(1976)]{rockafellar1976monotone}
R.~T. Rockafellar, ``Monotone operators and the proximal point algorithm,''
  \emph{{SIAM} J. Control Optim.}, vol.~14, no.~5, pp. 877--898, 1976.

\bibitem[Fairbanks et~al.(2015)Fairbanks, Kannan, Park, and
  Bader]{fairbanks2015behavioral}
J.~P. Fairbanks, R.~Kannan, H.~Park, and D.~A. Bader, ``Behavioral clusters in
  dynamic graphs,'' \emph{Parallel Computing}, vol.~47, pp. 38--50, 2015.

\bibitem[Qian et~al.(2018)Qian, Tan, Mamoulis, and Cheung]{QianTMC18}
Y.~Qian, C.~Tan, N.~Mamoulis, and D.~W. Cheung, ``{DSANLS:} accelerating
  distributed nonnegative matrix factorization via sketching,'' in
  \emph{{WSDM}}, 2018, pp. 450--458.

\bibitem[Boyd et~al.(2003)Boyd, Xiao, and Mutapcic]{Boyd03}
\BIBentryALTinterwordspacing
S.~Boyd, L.~Xiao, and A.~Mutapcic, ``Subgradient methods,'' \emph{Stanford
  University}, 2003. [Online]. Available:
  \url{https://web.stanford.edu/class/ee392o/subgrad_method.pdf}
\BIBentrySTDinterwordspacing

\bibitem[Kim et~al.(2014)Kim, He, and Park]{kim2014algorithms}
J.~Kim, Y.~He, and H.~Park, ``Algorithms for nonnegative matrix and tensor
  factorizations: a unified view based on block coordinate descent framework,''
  \emph{J. Global Optimization}, vol.~58, no.~2, pp. 285--319, 2014.

\bibitem[Neveu(1975)]{neveu1975discrete}
J.~Neveu, \emph{Discrete-parameter martingales}.\hskip 1em plus 0.5em minus
  0.4em\relax Elsevier, 1975, vol.~10.

\bibitem[Mairal(2013)]{mairal2013stochastic}
J.~Mairal, ``Stochastic majorization-minimization algorithms for large-scale
  optimization,'' in \emph{{NIPS}}, 2013, pp. 2283--2291.

\end{thebibliography}

\vspace{-25pt}
\begin{IEEEbiography}[\vspace{-2em}{\includegraphics[width=1in,height=1.25in,clip,keepaspectratio]{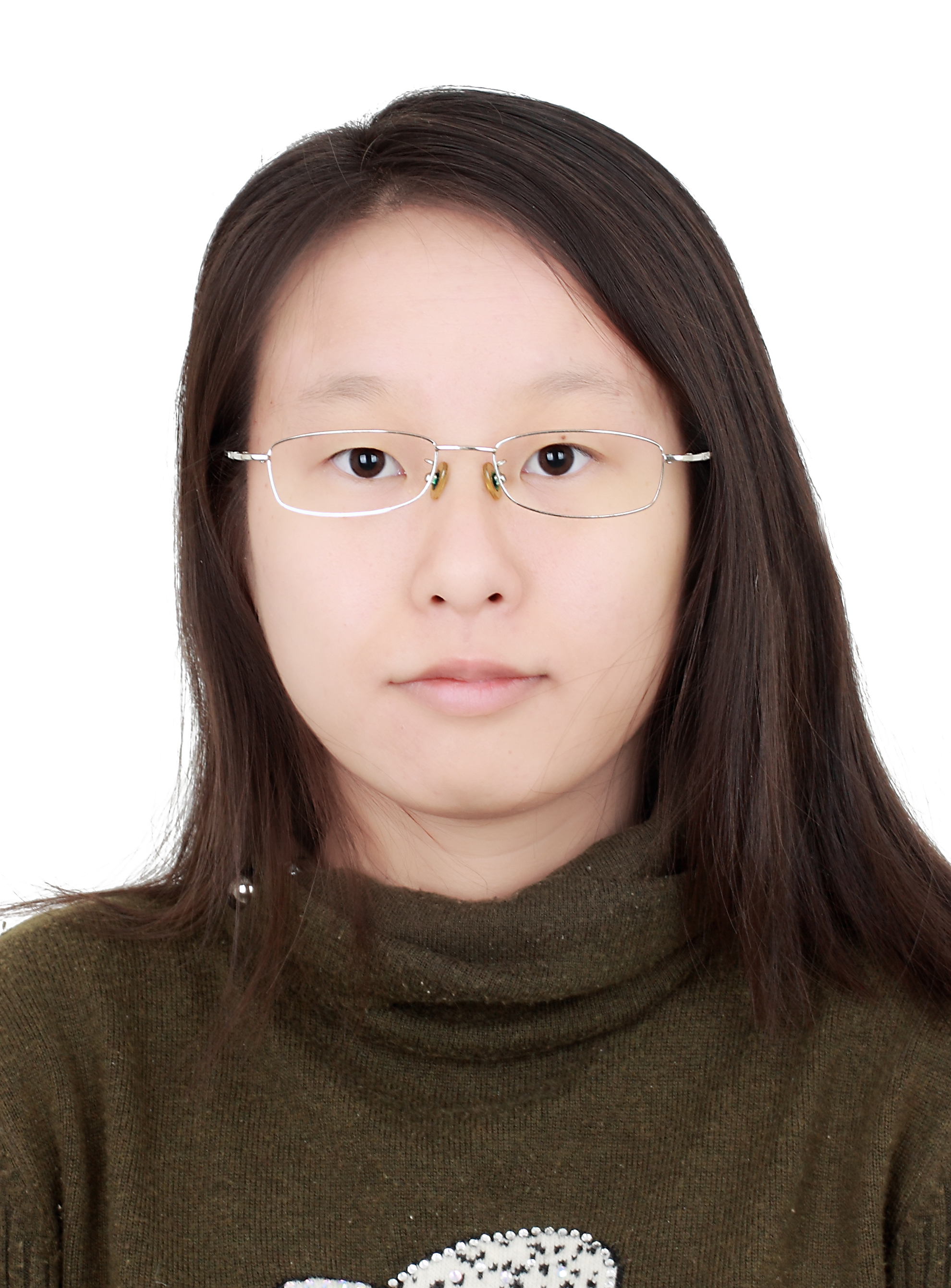}}]
{Yuqiu Qian} is currently an applied researcher in Tencent. Her research interests include data engineering and machine learning with applications in recommender systems. She received her B.Eng. degree in Computer Science from University of Science and Technology of China (2015), and her PhD degree in Computer Science from University of Hong Kong (2019).
\vspace{-40pt}
\end{IEEEbiography}

\begin{IEEEbiography}[\vspace{-2em}{\includegraphics[width=1in,height=1.25in,clip,keepaspectratio]{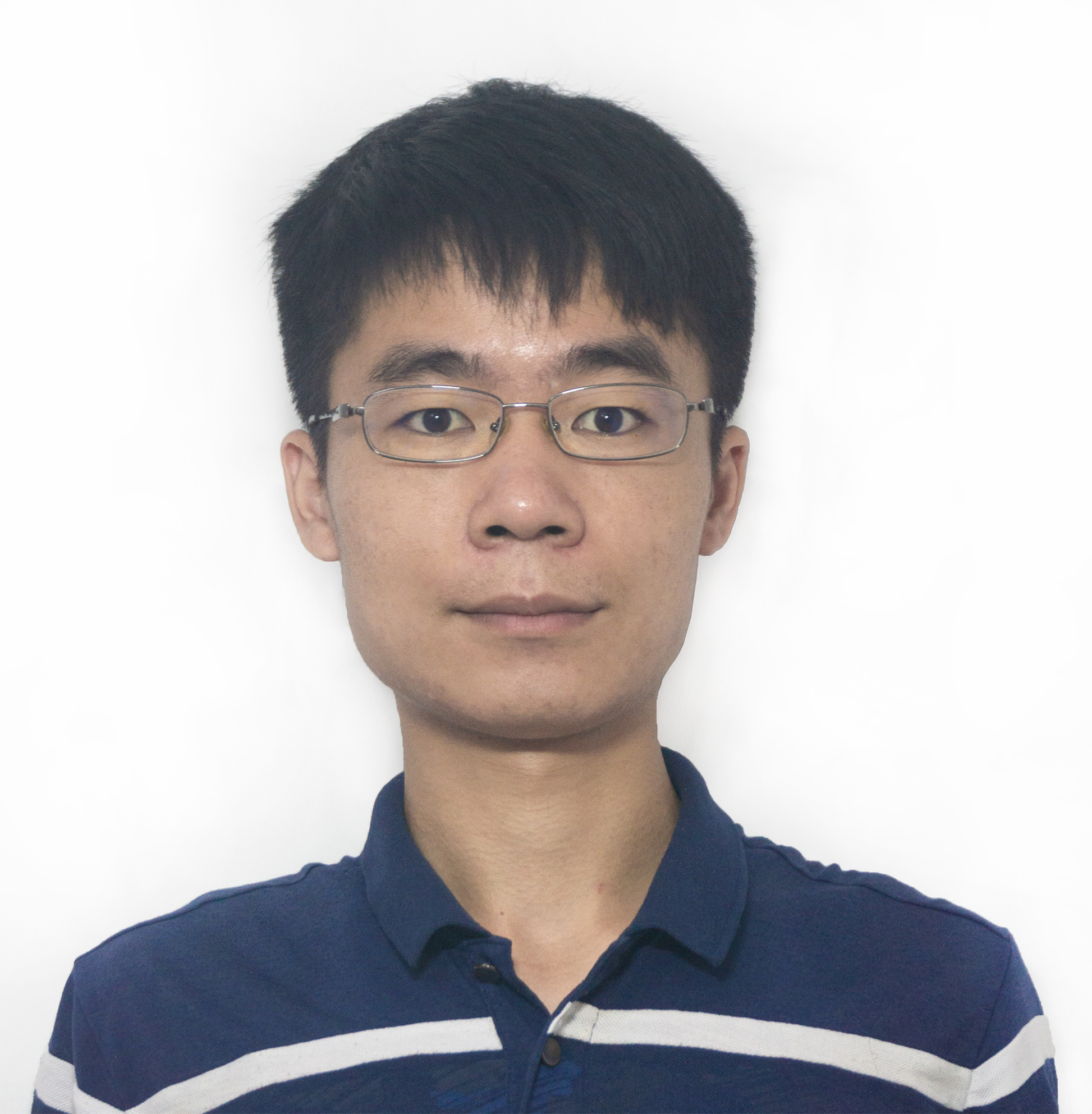}}]
{Conghui Tan} is currently a researcher in WeBank. His research interests include machine learning and optimization algorithms. He received his B.Eng. degree in Computer Science from University of Science and Technology of China (2015), and his PhD degree in System Engineering from Chinese University of Hong Kong (2019).
\vspace{-40pt}
\end{IEEEbiography}

\begin{IEEEbiography}[{\includegraphics[width=1in,height=1.25in,clip,keepaspectratio]{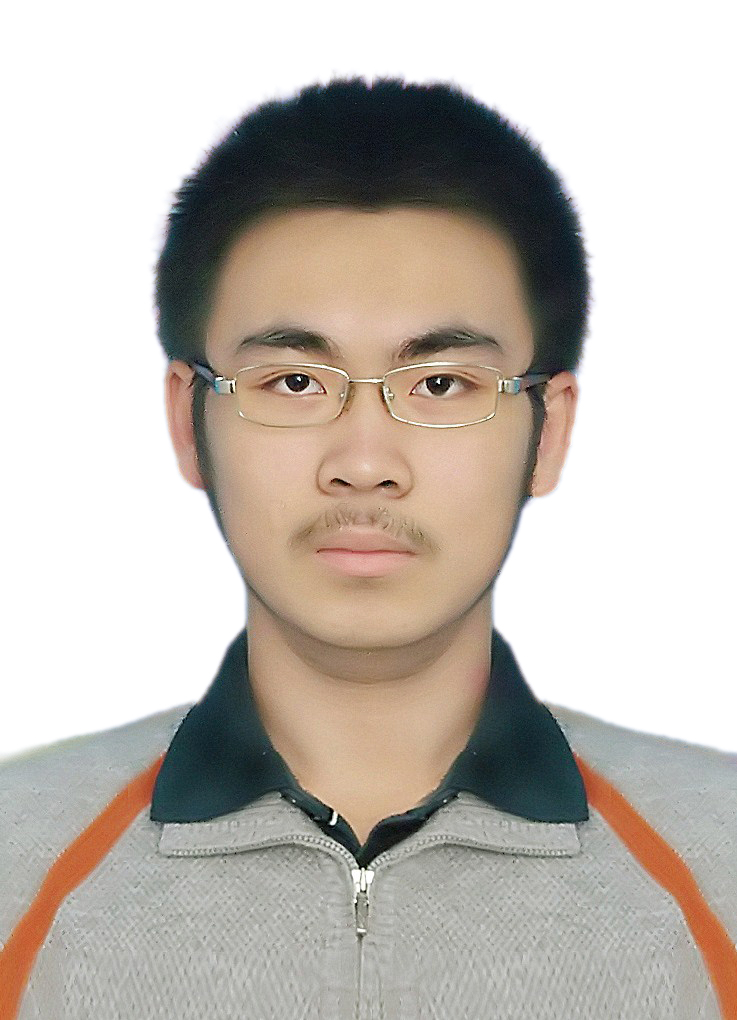}}]
{Danhao Ding} is currently a PhD candidate in Department of Computer Science, University of Hong Kong. His research interest include high performance computing and machine learning. He received his B.Eng. degree in Computing and Data Analytics from University of Hong Kong (2016).
\vspace{-40pt}
\end{IEEEbiography}

\begin{IEEEbiography}[{\includegraphics[width=1in,height=1.25in,clip,keepaspectratio]{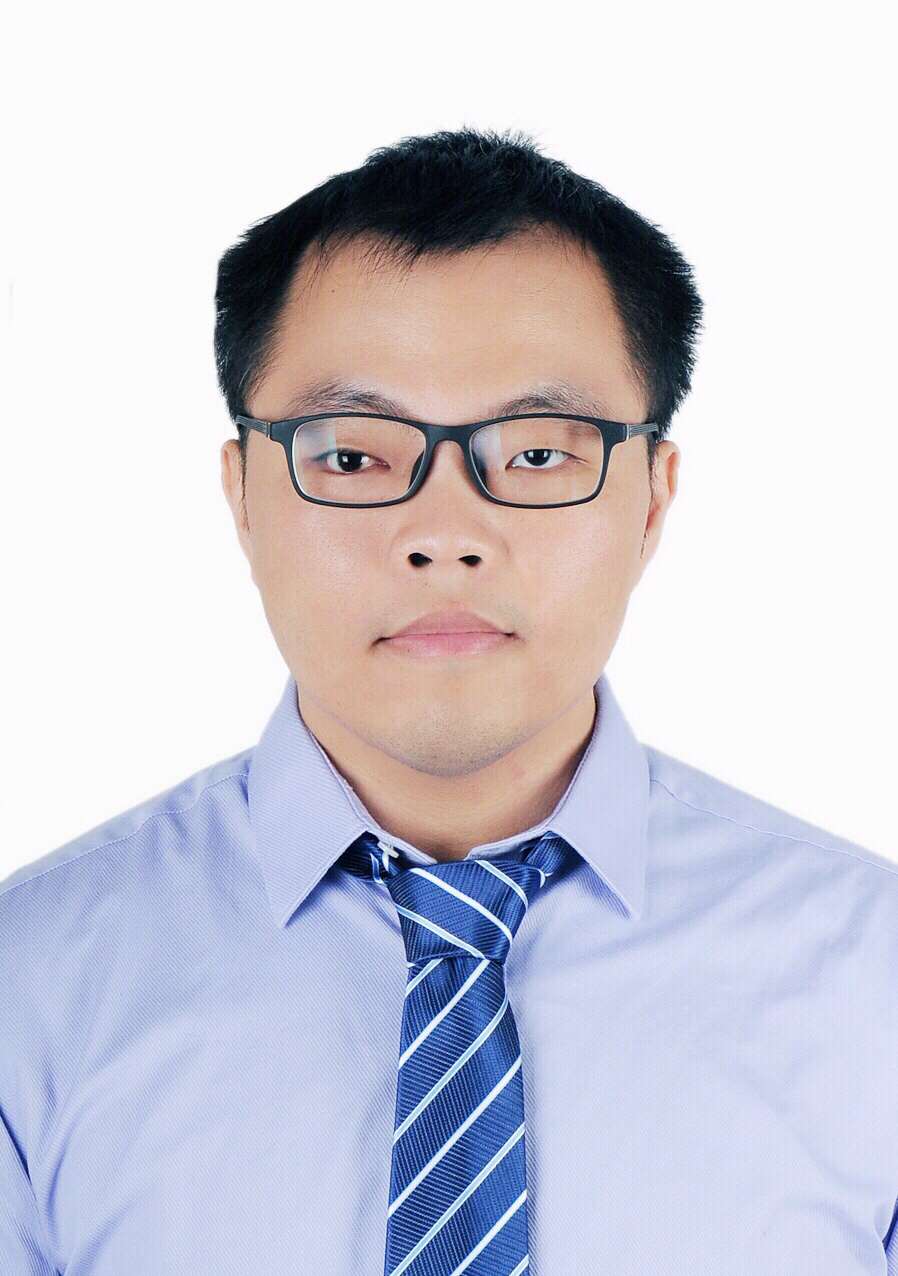}}]
{Hui Li} is currently an assistant professor in the School of Informatics, Xiamen University. His research interests include data mining and data management with applications in recommender systems and knowledge graph. He received his B.Eng. degree in Software Engineering from Central South University (2012), and his MPhil and PhD degrees in Computer Science from University of Hong Kong (2015, 2018).
\vspace{-40pt}
\end{IEEEbiography}

\begin{IEEEbiography}[{\includegraphics[width=1in,height=1.25in,clip,keepaspectratio]{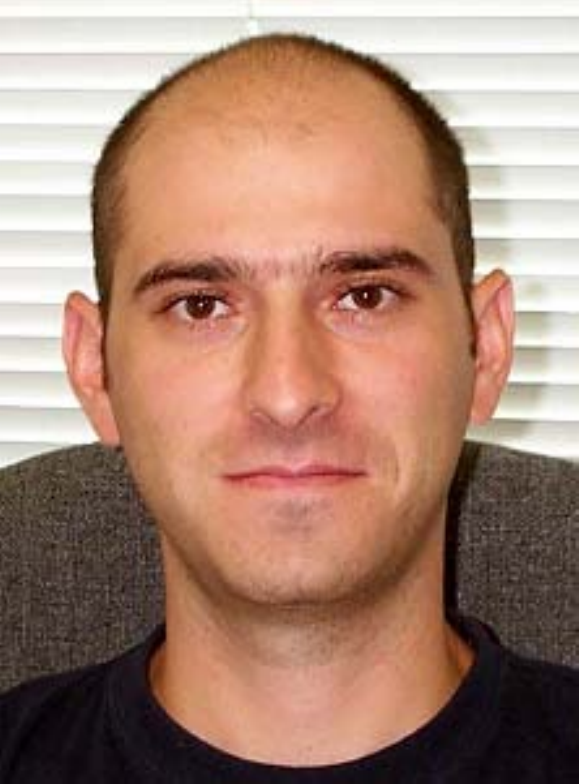}}]{Nikos Mamoulis} received his diploma in computer engineering and informatics in 1995 from the University of Patras, Greece, and his PhD in computer science in 2000 from HKUST. He is currently a faculty member at the University of Ioannina. Before, he was professor at the Department of Computer Science, University of Hong Kong. His research focuses on the management and mining of complex data types.
\end{IEEEbiography}

\end{document}